\documentclass[11pt]{article}
\usepackage[top=30truemm,bottom=30truemm,left=30truemm,right=30truemm]{geometry}   

\newcommand{\BlackBox}{\rule{1.5ex}{1.5ex}}  
\newenvironment{proof}{\par\noindent{\bf Proof\ }}{\hfill\BlackBox\\[2mm]}
\newtheorem{example}{Example}[section] 
\newtheorem{theorem}{Theorem}[section]
\newtheorem{lemma}[theorem]{Lemma} 
\newtheorem{proposition}[theorem]{Proposition} 
\newtheorem{remark}{Remark}[section]
\newtheorem{corollary}[theorem]{Corollary}
\newtheorem{definition}[theorem]{Definition}

\usepackage{natbib}
\usepackage{amsmath,amssymb,bm,graphicx,tikz,pgfplots,colonequals,mathtools} 
\usepackage{booktabs,multirow,multicol}
\usepackage{microtype}
\usepackage[utf8]{inputenc}
\usepackage{hyperref}
\usepackage{comment}
\newcommand{\dino}[1]{   {\bf \color{magenta}{DS: #1}}  }
\newcommand{\moto}[1]{   {\bf \color{blue}{MK: #1}}  }

\newcommand{\ph}[1]{   {\bf \color{green}{PH: #1}}  }

\newcommand{\g}{\mid}

\newcommand{\Cov}{\operatorname{cov}} 
\newcommand{\cov}{\operatorname{cov}} 
\newcommand{\Var}{\operatorname{var}} 
\newcommand{\var}{\operatorname{var}} 
\newcommand{\mean}{\operatorname{mean}} 
\renewcommand{\H}{\mathcal{H}}

\renewcommand{\Re}{\mathbb{R}}
\newcommand{\N}{\mathcal{N}} 

\newcommand{\cT}{\mathcal{T}}
\newcommand{\cA}{\mathcal{A}}

\newcommand{\Trans}{^{\intercal}} 

\newcommand{\eps}{\epsilon}

\newcommand{\argmin}{\operatorname*{arg\:min}}
\newcommand{\argmax}{\operatorname*{arg\:max}}

\newcommand{\ce}{\colonequals}

\newcommand{\q}{\quad}
\newcommand{\qq}{\qquad}

\renewcommand{\vec}{\boldsymbol}

\newcommand{\GP}{\mathcal{GP}}
\newcommand{\GPDC}{\mathrm{GPDC}}
\newcommand{\HSIC}{\mathrm{HSIC}}
\newcommand{\KSD}{\mathrm{KSD}}
\newcommand{\MMD}{\mathrm{MMD}}
\newcommand{\Id}{I}

\makeatletter
\newcommand{\superimpose}[2]{
 {\ooalign{$#1\@firstoftwo#2$\cr\hfil$#1\@secondoftwo#2$\hfil\cr}}}
\makeatother

\newcommand{\ssw}{\mathsf{w}}
\newcommand{\ssf}{\mathsf{f}}
\newcommand{\ssF}{\mathsf{F}}
\newcommand{\ssg}{\mathsf{g}}
\newcommand{\ff}{\mathfrak{f}}
\newcommand{\ssy}{\mathsf{y}}

\newcommand{\X}{\vec{X}}
\newcommand{\cH}{\mathcal{H}}
\newcommand{\cF}{\mathcal{F}}
\newcommand{\cG}{\mathcal{G}}
\newcommand{\cX}{\mathcal{X}}
\newcommand{\cY}{\mathcal{Y}}

\newcommand{\bE}{\mathbb{E}}
\newcommand{\tX}{\tilde{X}}

\newcommand{\fz}{\mathsf{z}}

\newcommand{\bmg}{{\bm g}}
\newcommand{\bmf}{{\bm f}}
\newcommand{\bs}{{\bm s}}
\newcommand{\bK}{{\bm K}}
\newcommand{\bL}{{\bm L}}
\newcommand{\iid}{\overset{i.i.d.}{\sim}}
\newcommand{\te}{\tilde{e}}

\newcommand{\hmu}{\hat{\mu}}
\newcommand{\cmu}{\check{\mu}}

\newcommand{\red}{\textcolor{red}}
\newcommand{\blue}{\textcolor{blue}}

\newcommand{\hrarrow}{\hookrightarrow}

\usetikzlibrary{arrows,shapes,plotmarks,decorations.pathmorphing}
\usetikzlibrary{backgrounds,calc,positioning,fadings}

\tikzset{>=stealth'} 
\tikzstyle{graphnode} = 
  [circle,draw=black,minimum size=22pt,text centered,text
    width=22pt,inner sep=0pt] 
\tikzstyle{var}   =[graphnode,fill=white]
\tikzstyle{obs}   =[graphnode,fill=black,text=white]
\tikzstyle{act}   =[rectangle,draw=black,text=white,minimum
size=22pt,text centered, text width=22pt,inner sep=0pt]
\tikzstyle{fac}   =[rectangle,draw=black,fill=black!25,minimum size=5pt]
\tikzstyle{facprior} =[rectangle,draw=black,fill=black,text=white,minimum size=5pt]
\tikzstyle{edge}  =[draw=white,double=black,thick,-]
\tikzstyle{prior} =[rectangle, draw=black, fill=black, minimum size=
5pt, inner sep=0pt]
\tikzstyle{dirprior} = [circle, draw=black, fill=black, minimum
size=5pt, inner sep=0pt]

\DeclareSymbolFont{stmry}{U}{stmry}{m}{n}
\DeclareMathSymbol\leftarrowtriangle\mathrel{stmry}{"5E}
\DeclareMathSymbol\rightarrowtriangle\mathrel{stmry}{"5F}
\DeclareMathSymbol\sslash\mathrel{stmry}{"0C}

\renewcommand{\to}{\operatorname*{\rightarrowtriangle}}

\usetikzlibrary{arrows,shapes,plotmarks,pgfplots.colormaps}
\pgfplotsset{compat=newest}
\pgfplotsset{
 every axis legend/.append style =
   {
     cells = { anchor = east },
     draw  = none
   },
}  
\makeatletter
\pgfplotsset{ 
   range frame/.style={
       tick align = outside,
       axis line style={opacity=0},
       after end axis/.code={
           \draw ({rel axis cs:0,0}-|{axis cs:\pgfplots@data@xmin,0}) -- ({rel axis cs:0,0}-|{axis cs:\pgfplots@data@xmax,0});
           \draw ({rel axis cs:0,0}|-{axis cs:0,\pgfplots@data@ymin}) -- ({rel axis cs:0,0}|-{axis cs:0,\pgfplots@data@ymax});
       }
   }
}
\makeatother

\pgfkeys{/pgfplots/mystyle/.style={
 xtick align = inside,
 ytick align = inside
 }}

\usepackage{algorithm}
\usepackage{algpseudocode}
\algrenewcommand{\algorithmiccomment}[1]{\hfill {\footnotesize $\sslash$ #1}}
\algrenewcommand\alglinenumber[1]{\scriptsize\texttt{#1}}

\makeatletter
\@addtoreset{algorithm}{chapter} 
\makeatother

\makeatletter
\def\therule{\makebox[\algorithmicindent][l]{\hspace*{.5em}\vrule height .75\baselineskip depth .25\baselineskip}}%

\newtoks\therules
\therules={}
\def\appendto#1#2{\expandafter#1\expandafter{\the#1#2}}
\def\gobblefirst#1{
  #1\expandafter\expandafter\expandafter{\expandafter\@gobble\the#1}}%
\def\LState{\State\unskip\the\therules}
\def\pushindent{\appendto\therules\therule}%
\def\popindent{\gobblefirst\therules}%
\def\printindent{\unskip\the\therules}%
\def\printandpush{\printindent\pushindent}%
\def\popandprint{\popindent\printindent}%

\algdef{SE}[WHILE]{While}{EndWhile}[1]
  {\printandpush\algorithmicwhile\ #1\ \algorithmicdo}
  {\popandprint\algorithmicend\ \algorithmicwhile}%
\algdef{SE}[FOR]{For}{EndFor}[1]
  {\printandpush\algorithmicfor\ #1\ \algorithmicdo}
  {\popandprint\algorithmicend\ \algorithmicfor}%
\algdef{S}[FOR]{ForAll}[1]
  {\printindent\algorithmicforall\ #1\ \algorithmicdo}%
\algdef{SE}[LOOP]{Loop}{EndLoop}
  {\printandpush\algorithmicloop}
  {\popandprint\algorithmicend\ \algorithmicloop}%
\algdef{SE}[REPEAT]{Repeat}{Until}
  {\printandpush\algorithmicrepeat}[1]
  {\popandprint\algorithmicuntil\ #1}%
\algdef{SE}[IF]{If}{EndIf}[1]
  {\printandpush\algorithmicif\ #1\ \algorithmicthen}
  {\popandprint\algorithmicend\ \algorithmicif}%
\algdef{C}[IF]{IF}{ElsIf}[1]
  {\popandprint\pushindent\algorithmicelse\ \algorithmicif\ #1\ \algorithmicthen}%
\algdef{Ce}[ELSE]{IF}{Else}{EndIf}
  {\popandprint\pushindent\algorithmicelse}%
\algdef{SE}[PROCEDURE]{Procedure}{EndProcedure}[2]
   {\printandpush\algorithmicprocedure\ \textproc{#1}\ifthenelse{\equal{#2}{}}{}{(#2)}}%
   {\popandprint\algorithmicend\ \algorithmicprocedure}%
\algdef{SE}[FUNCTION]{Function}{EndFunction}[2]
   {\printandpush\algorithmicfunction\ \textproc{#1}\ifthenelse{\equal{#2}{}}{}{(#2)}}%
   {\popandprint\algorithmicend\ \algorithmicfunction}%
\makeatother

\newlength{\figheight}
\newlength{\figwidth}

\begin{document}
\title{Gaussian Processes and Kernel Methods:\\ A Review on Connections and Equivalences}

\author{Motonobu Kanagawa$^1$, Philipp Hennig$^1$,\medskip{}\\ Dino Sejdinovic$^2$, and Bharath K Sriperumbudur$^3$ \medskip{} \\ \\
$^1$University of Tübingen and Max Planck Institute for Intelligent Systems,\\ Max-Planck-Ring 4, 72076 Tübingen, Germany \\
\url{{motonobu.kanagawa, ph}@tue.mpg.de}
\medskip{} \\
$^2$Department of Statistics, University of Oxford,\\
       24-29 St Giles', Oxford OX1 3LB, UK\\
       \url{dino.sejdinovic@stats.ox.ac.uk} \medskip{} \\
$^3$Department of Statistics, Pennsylvania State University,\\
       University Park, PA 16802, USA \\
       \url{bks18@psu.edu} \\
}

\maketitle

\begin{abstract}
  This paper is an attempt to bridge the conceptual gaps between researchers working on the two widely used  approaches based on positive definite kernels: Bayesian  learning or inference using Gaussian processes on the one side, and frequentist kernel methods based on reproducing kernel Hilbert spaces on the other. 
  It is widely known in machine learning that these two formalisms are closely related; for instance, the estimator of kernel ridge regression is identical to the posterior mean of Gaussian process regression.
  However, they have been studied and developed almost independently by two essentially separate communities, and this makes it difficult to seamlessly transfer results between them.
  Our aim is to overcome this potential difficulty. 
 To this end, we review several old and new results and concepts from either side, and juxtapose algorithmic quantities from each framework to highlight close similarities.
 We also provide discussions on subtle philosophical and theoretical differences between the two approaches.
\end{abstract}

\tableofcontents

\section{Introduction}
\label{sec:introduction}
In machine learning, two nonparametric approaches based on positive definite kernels have been widely used for the purpose of modeling nonlinear functional relationships.
On the one side, there is Bayesian machine learning with Gaussian processes (GP), which models a problem at hand probabilistically and produces a posterior distribution for an unknown function of interest.
On the other, frequentist kernel methods with reproducing kernel Hilbert spaces (RKHS) usually take a decision theoretic approach by defining a loss function and optimizing the empirical risk.
These two approaches have been shown to be practically powerful and theoretically sound, and have found a wide range of practical applications in dealing with nonlinear phenomena.

It is well known that the two approaches are intimately connected.
The most notable example is that, if both use the same kernel, the posterior mean of Gaussian process regression equals the estimator of kernel ridge regression \citep{kimeldorf1970correspondence}.
Another connection is between Bayesian quadrature \citep{Oha91} and kernel herding \citep{CheWelSmo10}, which are in fact equivalent approaches to numerical integration or deterministic sampling \citep{huszar2012optimally}. 
These equivalences arise from the more fundamental connection that the notion of positive definite kernels is leveraged both  in Gaussian processes as covariance functions, and in RKHSs as reproducing kernels. 

There are also less deeply studied connections between the Bayesian and the frequentist approaches.
The posterior variance plays a fundamental role in the Bayesian approach, where it quantifies uncertainty over latent quantities of interest.
As we show in Section \ref{sec:posterior_variance}, posterior variance can be interpreted as a worst case error in an RKHS. This frequentist interpretation is much less widely known, and less well understood. It is rarely mentioned in the literature on frequentist kernel methods, and some of its potential applications therein may have been missed. 

The two approaches also subtly differ in how they define hypothesis spaces, which is a core aspect of statistical methods.
Consider for instance the regression problem, which involves a hypothesis space for the unknown regression function.
The Bayesian approach defines a hypothesis space through a GP prior distribution, treating the true function as a random function.
The support of the GP then expresses the knowledge or belief over the truth, and the probability mass expresses the degree of belief.
On the other hand, the frequentist approach expresses one's prior knowledge or belief by assuming the truth belongs to an RKHS or can be approximated well by functions in the RKHS.
Even though the use of the same kernel leads to similar structural assumptions about the function of interest in both approaches, e.g., periodicity or smoothness, there is a fundamental modeling difference, because the support of a GP is {\em not} identical to the corresponding RKHS (e.g.,~\citet[Corollary 7.1]{LukBed01}; see also Section \ref{sec:driscol-sample-path}).
In fact, sample paths of the GP fall outside the RKHS of the covariance kernel with probability one.
This fact might give researchers an impression 
that differences outweigh the similarities and that the known connections between the Bayesian and frequentist approaches are rather superficial. However, as we show in Sections \ref{sec:theory} and \ref{sec:rates-and-posterior-contraction}, a closer look reveals further deep connections.

This text reviews known, and establishes new, equivalences between the Bayesian and frequentist approaches.
Our aim is to help researchers working in both fields gain mutual understanding, and be able to seamlessly transfer results in either side to the other. In fact there are some quantities that are almost exclusively studied and utilized on one side of the debate, and this may highlight interesting directions for the other community. 
Our second motivation is that, while the connections between the two approaches are found and mentioned individually in papers or books, we are not aware of thorough texts focusing specifically on this topic from a modern perspective.
We thus aim to collect a short yet systematic overview of the connections. Finally, we also hope that this text offers a short pedagogical introduction to researchers and students who are new to and interested in either of the two fields.

\subsection{Contents of the Paper}

The principal results mentioned in the later text can be summarized as follows.

\paragraph{Section \ref{sec:definition}: Gaussian Processes and RKHSs: Preliminaries} 
As a starting point, we review basic definitions and properties of GPs and RKHSs with illustrative examples.
We also provide a characterization of RKHSs based on Fourier transforms of kernels, which helps the reader to understand the structure of RKHSs in terms of smoothness of functions.

\paragraph{Section \ref{sec:regression}: Connections between Gaussian Process and Kernel Ridge Regression}
Regression is arguably one of the most basic and practically important problems in machine learning and statistics.
We consider Gaussian process regression and kernel ridge regression, and discuss equivalences between the two methods.
As mentioned above, it is well known that the posterior mean in GP-regression coincides with the estimator of kernel ridge regression.
We furthermore show that there is a frequentist interpretation for posterior variance in GP-regression, as a worst case error in kernel ridge regression.
In this sense, average-case and worst-case error are equivalent in the least-squares setting, which is key to understanding the connections between the Bayesian and frequentist approaches.

We also discuss the role of additive Gaussian noise in GP-regression and regularization in kernel ridge regression, showing that they are essentially equivalent as a mechanism to make regressors smoother.
We then discuss the noise-free setting, where regression becomes interpolation.
Here the equivalence between the two approaches can be useful: an upper-bound on posterior variance is derived, transferring a result from the frequentist literature on scattered data approximation to the Bayesian setting, as shown in Section \ref{sec:upper-bound-variance}.

\paragraph{Section \ref{sec:theory}: Hypothesis Spaces: Do Gaussian Process Draws Lie in an RKHS?}
We review the properties of GPs and RKHSs as hypothesis spaces, that is, as a way of expressing prior knowledge or belief.
We begin with characterizations of GPs and RKHSs by orthonormal expansions, known respectively as the Mercer representation and the Karhunen-Lo\'eve expansion. 
These characterizations allow us to phrase quantities of interest in terms of eigenvalues and eigenfunctions of an integral operator defined by the kernel.
We then discuss previous results of \citet{driscoll1973reproducing,LukBed01} providing a necessary and sufficient condition for a given GP to be a member of a given RKHS (which can be different from the RKHS associated with the covariance kernel of the GP).
This implies that, while GP sample paths are almost surely outside of the corresponding RKHS, they lie in a function space ``slightly larger'' than the RKHS, which is itself a certain RKHS \citep{SteSco12,Ste17}. In this sense, the Bayesian prior and the frequentist hypothesis space, while not identical, are arguably closer to each other than is often acknowledged.

\paragraph{Section \ref{sec:rates-and-posterior-contraction}: Convergence and Posterior Contraction Rates in Regression}

We compare convergence properties of GP-regression and kernel ridge regression.
Specifically, we show that convergence rates for GP-regression obtained in \citet{VarZan11} can be recovered from those for kernel ridge regression obtained in \cite{SteHusSco09}, considering the situation where a regression function is assumed to have a finite degree of smoothness.
Since the GP prior is supported on a slightly larger space than the RKHS, to recover convergence rates matching that of GP regression, we need to assume that the regression function belongs to this slightly larger space.
Even in this case, one can obtain a convergence rate for kernel ridge regression, thanks to the approximation capability of the RKHS.
That is, the regression function can be approximated well by functions in the RKHS, with the accuracy of approximation  determined by the choice of a regularization constant.
Interestingly, the asymptotically optimal schedule of regularization constants translates to the assumption in GP regression that noise variance remains constant with increasing sample size.
In this sense, a Bayesian assumption of additive noise is related to regularization in the frequentist approach.

\paragraph{Section \ref{sec:integral_transforms}: Integral Transforms}

This section deals with somewhat more exotic topics than regression, where connections between the Bayesian and frequentist literature have not been studied as deeply. 
We discuss integral transforms of (probability) measures with kernels, a framework known as {\em kernel mean embeddings of distributions} \citep{SmoGreSonSch07}.
This approach provides a nonparametric way of representing probability distributions, and of measuring a distance between them. The former are called {\em kernel means}, and the latter the {\em maximum mean discrepancy} (MMD).
These have been widely used in machine learning, and interested readers may have a look at the recent survey by \cite{MuaFukSriSch17}.

While the MMD is characterized as the \emph{worst-case} error of integrals with respect to functions in an RKHS, it can also be characterized as the {\em average-case error} of integrals with respect to draws from the corresponding GP \citep[Corollary 7 in p.40]{Rit00}.
This viewpoint provides an alternative way to understand kernel embeddings in the language of Bayesian quadrature for Bayesians who are familiar with GPs but not with RKHSs, and vice versa.
We also review a shrinkage estimator for kernel means proposed by \citet{JMLR:v17:14-195} and the corresponding GP-based Bayesian interpretation \citep{FlaSejCunFil2016}.
We then discuss the problem of sampling or numerical integration, for which GPs and kernel methods have also played fundamental roles in the form of integral transforms \citep{Oha91,Hic98,BriOatGirOsbSej15,DicKuoSlo13}.

Finally, we study the connections between GPs and kernel methods as applied to the problem of measuring dependence between random variables.
We consider the Hilbert-Schmidt independence criterion (HSIC), a kernel-based measure for dependency between two random variables \citep{GreBouSmoSch05}, which has a wide range of applications including independence testing, variable selection and causal discovery.
HSIC is defined in terms of RKHSs via the cross-covariance operator or via the joint kernel embedding of two random variables; this definition makes HSIC difficult to interpret without close familiarity with RKHSs. 
We give an alternative formulation of HSIC in terms of GPs, which is to the best of our knowledge novel, and recovers Brownian distance covariance proposed by \citet{SzeRiz09}.
We believe this result provides an intuitive explanation for people whose backgrounds are from GPs about why HSIC is a sensible dependency measure.

\paragraph{Related Literature}
We collect here a few key related literature on GPs and kernel methods that may be helpful for further reading.
Our aim is the closest in spirit to \citet{Berlinet2004}, who collected classic results on the use of RKHSs in statistics and probability; these include the results by \citet{kolmogorov1941interpolation,Par61,matheron1962traite,kimeldorf1970correspondence,larkin1972gaussian}, who made the earliest contributions to the field.
Other related books and monographs include \citet{wahba1990spline,Ad90,Janson1997,Ste99,Rit00,scholkopf2002learning,Wen05,SchWen06,RasmussenWilliams,AdlJon07,Steinwart2008,VarZan08,NovWoz08,NovWoz10,Stu10,SchSchSch13,HenOsbGirRSPA2015,MuaFukSriSch17}.

\subsection{Notation and 
Definitions} 
\label{sec:notation}
We collect the notation and definitions that will be used throughout the paper.
\paragraph{Basics}
For a matrix (or a vector) $M$, its transpose is denoted by $M\Trans$.
Let $\mathbb{N}$ be the set of natural numbers, $\mathbb{N}_0 := \mathbb{N} \cup \left\{ 0 \right\}$ and $\mathbb{N}_0^d$ be the $d$-dimensional Cartesian product of $\mathbb{N}_0$ with $d \in \mathbb{N}$.
For a multi-index $\alpha := (\alpha_1,\dots,\alpha)\Trans \in \mathbb{N}_0^d$, define $| \alpha | := \sum_{i=1}^d \alpha_i$.
$\Re$ denotes the real line,  $\Re^d$ the $d$-dimensional Euclidean space for $d \in \mathbb{N}$, and $\| \cdot \|$ the Euclidean norm.
For $\alpha \in \mathbb{N}_0^d$ and a function $f$ defined on $\Re^d$, let $\partial^\alpha f$ and $D^\alpha f$ be the $\alpha$-th partial derivative and the $\alpha$-th partial weak derivative, respectively.
For $i,j \in \mathbb{N}$, we define $\delta_{ij} \in \{0, 1\}$ as $\delta_{ij} = 1$ if $i = j$, and $\delta_{ij} = 0$ otherwise.

\paragraph{Probability}
For a random variable $x$ and a probability distribution $P$, writing $x \sim P$ means that $x$ has distribution $P$.
$\bE[\cdot]$ denotes the expectation of the argument in the bracket, with respect to a random variable concerned.
Depending on the context, we may write $\bE_{x}[\cdot]$ or $\bE_{x \sim P}[\cdot]$ to make the random variable and the distribution explicit.

\paragraph{Matrices}
Throughout, $\cX$ will denote a set of interest. 
Given two subsets $A\ce (a_1,\dots,a_M)$ and $B\ce(b_1,\dots,b_N)$ of $\cX$, $k_{AB}\in\Re^{M\times N}$ denotes the matrix with elements $[k_{AB}]_{ij} = k(a_i,b_j)$. For a real-valued function $f:\cX\to\Re$, $f_A\in\Re^{M}$ denotes the vector with elements $[f_A]_i= f(a_i)$. 

\paragraph{Function spaces}
For a topological space $\cX$, let $C(\cX)$ denote a set of continuous functions.
For a measurable space $\cX$, a measure $\nu$ on $\cX$ and a constant $p > 0$, let $L_p(\nu)$ be the Banach space of ($\nu$-a.e.~equivalent classes of) $p$-integrable functions with respect to $\nu$:
\begin{equation} \label{eq:lp-space}
L_p(\nu) := \left\{f:\cX \to \Re: \| f \|_{L_p(\nu)}^p := \int |f(x)|^p d\nu(x) < \infty \right\}.
\end{equation}
Denote by $L_\nu(\cX) := L_2(\nu)$ the one when $\cX \subset \Re^d$ and $\nu$ is the Lebesgue measure. 
For $f \in L_1(\Re^d)$, we define its Fourier transform by
$$
\cF[f](\omega) := \frac{1}{(2\pi)^{d/2}}\int f(x) e^{ - \sqrt{-1}\ x\Trans \omega } dx, \quad \omega \in \Re^d. 
$$

\section{Gaussian Processes and RKHSs: Preliminaries}\label{sec:definition}

We re-state standard definitions for Gaussian processes (GPs) and RKHSs, reviewing basic properties. Section \ref{sec:pd-kernel} defines positive definite kernels, Sections \ref{sec:GP-intro} and \ref{sec:RKHS-def} introduce GPs and RKHSs, respectively. Detailed characterizations of GPs and RKHSs can be found in Section \ref{sec:theory}.

\subsection{Positive Definite Kernels} \label{sec:pd-kernel}
{\em Positive definite kernels} play a key role in both Gaussian processes and RKHSs. 

\begin{definition}[Positive definite kernels] \label{def:pd-kernel}
Let $\cX$ be a nonempty set.
A symmetric function $k:\cX\times\cX \to \Re$ is called a positive definite kernel, if for any $n \in \mathbb{N}$, $(c_1,\ldots,c_n)\subset\Re$ and $(x_1,\dots,x_n) \subset \cX$,
$$
\sum_{i=1}^n \sum_{j=1}^n  c_i c_j k(x_i,x_j) \geq 0.
$$
\end{definition}

\begin{remark}\rm
Definition \ref{def:pd-kernel} can be equivalently stated thus:
A symmetric function $k$ is positive definite if the matrix $k_{XX} \in\Re^{n\times n}$ with elements $[k_{XX}]_{ij}=k(x_i,x_j)$ is positive semidefinite for any finite set $X \ce (x_1,\dots,x_n) \in \cX^n$ of any size $n \in \mathbb{N}$.
\end{remark}

In the remainder, for simplicity, \emph{kernel} always means \emph{positive definite kernel}.
For $X = (x_1,\dots,x) \in \cX^n$, the matrix $k_{XX}$ is the \emph{kernel matrix} or \emph{Gram matrix}.

\begin{example}[Gaussian RBF/Square-Exponential Kernels] \label{ex:square-exp-kernel}
Let $\cX \subset \Re^d$.
For $\gamma > 0$, a Gaussian RBF kernel or a square exponential kernel $k_\gamma : \cX \times \cX \to \Re$ is defined by
\begin{equation} \label{eq:square-exp-kernel}
k_\gamma(x,x') = \exp\left( - \frac{ \| x - x' \|^2 } {\gamma^2 } \right), \quad x,x' \in \cX.
\end{equation}
\end{example}

In the Gaussian processes literature, to avoid confusion about the term ``Gaussian'', the kernel (\ref{eq:square-exp-kernel}) is often referred to as the \emph{square-exponential} kernel,\footnote{Sometimes it is also called ``square{\em d} exponential'' or ``exponentiated quadratic.''} while in the kernel literature it is called Gaussian, or Gaussian \emph{radial basis function} (RBF) kernel. 
The parameter $\gamma$ determines the length-scale of the associated hypothesis space of functions: 
As $\gamma$ increases, the kernel \eqref{eq:square-exp-kernel} and induced functions change less rapidly, and thus get ``smoother'' (while they are always infinite differentiable). See Section \ref{sec:theory} for details.

Another popular kernel is the Mat\'ern class of functions \citep{Mat60}, which is a standard in the spatial statistics literature \citep[Sections 2.7, 2.10]{Ste99}: In fact, \citet[Sec. 1.7]{Ste99} said ``Use the Mat\'ern model'' as a summary of practical suggestions for modeling spatial data.

\begin{example}[Mat\'ern kernels] \label{ex:matern-kernel}
Let $\cX \subset \Re^d$. 
For constants $\alpha > 0$ and $h > 0$, the Mat\'ern kernel $k_{\alpha,h}:\cX \times \cX \to \Re$ is defined by
\begin{equation} \label{eq:matern-kernel}
k_{\alpha, h}(x,x') = \frac{1}{2^{\alpha-1} \Gamma(\alpha)} \left( \frac{\sqrt{2\alpha} \| x - x' \|}{h} \right)^\alpha K_\alpha \left( \frac{ \sqrt{2\alpha} \| x - x' \| }{h} \right), \quad x,x' \in \cX,
\end{equation}
where $\Gamma$ is the gamma function, and $K_\alpha$ is the modified Bessel function of the second kind of order $\alpha$,
\end{example}

\begin{remark}\rm
If $\alpha$ can be written as $\alpha = m + 1/2$ for a non-negative integer $m$, then expression (\ref{eq:matern-kernel}) reduces to a product of the exponential function and a polynomial of degree $m$, which can be computed easily \citep[Section 4.2.1 and Eq.~4.16]{RasmussenWilliams}:
$$
k_{\alpha, h} (x,x') = \exp\left( - \frac{ \sqrt{2\alpha} \| x - x' \| }{h} \right) \frac{\Gamma(m + 1)}{\Gamma(2m + 1)} \sum_{i=1}^m \frac{ (m+1)! }{ i! (m-1)! } \left( \frac{ \sqrt{8\alpha} \| x - x' \| }{ h } \right)^{m-i} .
$$
For instance, if $\alpha = 1/2$, $\alpha = 3/2$ or $\alpha = 5/2$, the corresponding Mat\'ern kernels are 
\begin{eqnarray}
k_{1/2,h}(x,x') &=& \exp\left( - \frac{ \| x - x' \| }{h} \right), \label{eq:Laplace-kernel} \\
k_{3/2,h}(x,x') &=& \left( 1 + \frac{\sqrt{3} \| x - x' \| }{ h } \right) \exp\left( - \frac{ \sqrt{3} \| x - x' \| }{h} \right), \nonumber \\
k_{5/2,h}(x,x') &=& \left( 1 + \frac{\sqrt{5} \| x - x' \| }{ h } + \frac{ 5 \| x - x' \|^2 }{ 3h^2 } \right) \exp\left( - \frac{ \sqrt{5} \| x - x' \| }{h} \right). \nonumber
\end{eqnarray}
In particular, \eqref{eq:Laplace-kernel} is known as the Laplace or exponential kernel.
\end{remark}

In the expression \eqref{eq:matern-kernel}, the parameter $h$ determines the scale, and $\alpha$ the {\em smoothness} of functions in the associated hypothesis class: as $\alpha$ increases, the induced functions get smoother.
Mat\'ern kernels are appropriate when dealing with ``reasonably smooth'' (but not very smooth) functions, as the functions in the induced hypothesis class have a finite degree of smoothness \citep[Section 6.5]{Ste99}; this is in contrast to a square-exponential kernel, which induces functions with infinite smoothness (i.e., infinite differentiable functions).


\begin{remark}\rm \label{rem:matern-gauss}
Square-exponential kernels in Example \ref{ex:square-exp-kernel} can be obtained as limits of Mat\'ern kernels for $\alpha \to \infty$ \citep[p.~50]{Ste99}.
That is, for a Mat\'ern kernel $k_{\alpha,h}$ with $h > 0$ being fixed, we have  
$$
\lim_{\alpha \to \infty}k_{\alpha,h}(x,x') =  \exp \left(- \frac{\| x - x'\|^2}{2 h^2}  \right), \quad x, x' \in \mathbb{R}^d.
$$
\end{remark}

Our last example here is polynomial kernels \citep[Lemma 4.7]{Steinwart2008}, which induce hypothesis spaces consisting of polynomials.
This class of kernels have been popular in machine learning.
\begin{example}[Polynomial kernels]
Let $\cX \subset \Re^d$.
For $c > 0$ and $m \in \mathbb{N}$, the polynomial kernel $k_{m,c}: \cX \times \cX \to \Re$ is defined by
$$
k_{m,c}(x,x') = (x\Trans x' + c)^m, \quad x,x' \in \cX.
$$
\end{example}

While we have reviewed here only kernels defined on $\cX \subset \Re^d$, there are also various kernels defined on non-Euclidian spaces, such as sequences, graphs and distributions; see e.g.~\cite{scholkopf2002learning,SchTsuVer04,HofSchSmo08} and \citet[Section 4.2]{RasmussenWilliams}.
In fact, as Definition \ref{def:pd-kernel} indicates, positive definite kernels can be defined on any nonempty set.


\subsection{Gaussian Processes} \label{sec:GP-intro}
For Gaussian processes, positive definite kernels serve as {\em covariance functions} of random function values, so they are also called {\em covariance kernels}.
The following definition is taken from \citet[p.~443]{Dud02}.
\begin{definition}[Gaussian processes] \label{def:GP}
Let $\cX$ be a nonempty set, $k\colon\cX\times\cX\to\Re$ be a positive definite kernel and $m:\cX\to\Re$ be any real-valued function. 
Then a random function $\ssf:\cX \to \Re$ is said to be a Gaussian process (GP) with mean function $m$ and covariance kernel $k$, denoted by $\GP(m,k)$, if the following holds: For any finite set $X = (x_{1},\ldots,x_{n}) \subset \cX$ of any size $n \in \mathbb{N}$, the random vector $$\ssf_X = (\ssf(x_1),\dots,\ssf(x_n))\Trans \in \Re^n$$ follows the multivariate normal distribution $\N(
m_X,k_{XX})$ with covariance matrix $k_{XX} = (k(x_i,x_j))_{i,j = 1}^n \in \Re^{n \times n}$ and mean vector $m_X = (m(x_1),\dots,m(x_n))\Trans$.
\end{definition}

\begin{remark}\rm \label{rem:kernel-defines-GP}
Definition \ref{def:GP} implies that if $\ssf$ is a Gaussian process then there exists a mean function $m:\cX \to \Re$ and a covariance kernel $k:\cX \times \cX \to \Re$.
On the other hand, it is also true that for {\em any} positive definite kernel $k$ and mean function $m$, there exists a corresponding Gaussian process $\ssf \sim \GP(m,k)$ \citep[Theorem 12.1.3]{Dud02}.
There is thus a one-to-one correspondence between Gaussian processes $\ssf \sim \GP(m,k)$ and pairs $(m, k)$ of mean function $m$ and positive definite kernel $k$.
\end{remark}

\begin{remark}\rm
Since $k$ is the covariance function of a Gaussian process, by definition it can be written as 
\begin{equation} \label{eq:kerenl-gp-expect}
k(x,y) = \bE_{\ssf \sim \GP(m,k)} \left[ ( \ssf(x) - m(x) ) ( \ssf(y) - m(y) ) \right], \quad x,y \in \cX,
\end{equation}
where the expectation is with respect to the random function $\ssf \sim \GP(m,k)$.
This important property will be used extensively throughout this text.
\end{remark}

\begin{remark}\rm
In Definition \ref{def:pd-kernel}, the kernel matrix $k_{XX}$ may be singular: For instance when the kernel $k$ is a polynomial kernel or when some of the points $x_1,\dots,x_n$ are identical. 
Even in this case, the normal distribution $\mathcal{N}(m_X,k_{XX})$ is well-defined (and thus Definition \ref{def:pd-kernel} makes sense), while it does not have a density function with respect to the Lebesgue measure; see \citet[Theorem 9.5.7]{Dud02}.
\end{remark}

As mentioned in Remark \ref{rem:kernel-defines-GP}, the use of a specific kernel $k$ and a mean function $m$ implicitly leads to the use of the corresponding $\GP(k,m)$.
Therefore it is practically important to understand the properties of $\GP(k,m)$ (such as smoothness) that are induced by the specification of $k$ and $m$.
For example, if $k$ is a square-exponential kernel on an open set $\cX \subset \Re^d$, then a sample path $\ssf \sim \GP(0,k)$ is infinitely continuously differentiable, i.e., $\ssf$ is very smooth.
In Section \ref{sec:theory}, we will provide other examples as well as various characterizations for Gaussian processes.

For most practitioners, Gaussian process models manifest themselves in practice much as in their definition, through their finite-dimensional restriction to a concrete set of evaluation nodes; for example a plotting grid. In this case, Gaussian process models are actually very concrete models that are easy to handle on a computer. This fact is sometimes lost in theoretical texts, so we stress it in the following example.
\begin{example}[GP restricted to finite discrete domain]
Let $m:\cX \to\Re$ and $k:\cX \times \cX\to\Re$ be a mean and covariance function(kernel), respectively---the arguably most prominent choices are $m(x)\equiv 0$ and $k(x,x') = \exp(-(x-x')^2 / 2)$. Given a \emph{finite} set of representer points $X=(x_1,\dots,x_n)\subset\cX$, the Algorithm \ref{alg:GP_draws} produces a valid draw from the function $f\sim\GP(k,m)$, evaluated at the locations $[f(x_1),\dots,f(x_N)]$. 
For example, this is how the green draws in Figures \ref{fig:GP_intro} and \ref{fig:RKHS_v_GP} were produced.
\end{example}

\begin{algorithm}
\begin{algorithmic}[1]
\Procedure{GPsample}{$k,m,X$}
\LState \raisebox{0pt}[-1pt][-1pt]{$\phantom{[R,y]}\mathllap{m_X} = [m(x_i)]_{i=1,\dots,n}\in\Re^n$} \Comment{build mean vector}
\LState \raisebox{0pt}[-1pt][-1pt]{$\phantom{[R,y]}\mathllap{k_{XX}} = [k(x_i,x_j)]_{i,j=1,\dots,n}\in\Re^{n\times n}$} \Comment{build covariance matrix}
\LState \raisebox{0pt}[-1pt][-1pt]{$\phantom{[R,y]}\mathllap{R} = \Call{Cholesky}{k_{XX}}$} \Comment{compute Cholesky decomposition $k_{XX}=R\Trans R$}
\LState \raisebox{0pt}[-1pt][-1pt]{$\phantom{[R,y]}\mathllap{u} = \Call{randn}{n,1}$} \Comment{draw $u\sim\N(0,\Id_n)$.}
\LState \raisebox{0pt}[-1pt][-1pt]{$\phantom{[R,y]}\mathllap{f_X} = R\Trans u + m_X$} \Comment{affine transformation of $u$ gives sample from GP}
\EndProcedure
\end{algorithmic}
\caption{Concrete algorithm producing independent samples from $\GP(k,m)$ on the finite domain $X\in\cX^n$. }
\label{alg:GP_draws}
\end{algorithm}

The following more abstract example, taken from  \citet{LinRueLin11}, explains that Gaussian processes of Mat\'ern kernels are given as solutions of certain stochastic partial differential equations (SPDE).
This was first shown by \citet[Section 9]{Whi54} for the special case of $d = 2$; see \citet{LinRueLin11} and references therein for further details.
\begin{example}[GPs of Mat\'ern kernels] \label{ex:matern-spde}
Let $k_{\alpha,h}$ be a Mat\'ern kernel in Example \ref{ex:matern-kernel} with parameters $\alpha > 0$ and $h > 0$ defined on $\mathcal{X} = \Re$.
Then the corresponding Gaussian process $\ssf \sim \GP(0,k_{\alpha, h})$ is the only stationary solution to the following SPDE
$$
\left( \frac{2\alpha}{h^2} - \Delta \right)^{ (\alpha + d/2) /2} \ssf (x) = \ssw(x), \quad x \in \mathbb{R}^d,
$$
where $\Delta := \sum_{i=1}^d \frac{\partial^2}{\partial x_i^2}$ is the Laplacian and $\eta$ is the Gaussian white noise process with unit variance.  
\end{example}

\subsection{Reproducing Kernel Hilbert Spaces} \label{sec:RKHS-def}

Reproducing kernel Hilbert spaces are defined as follows, where positive definite kernels serve as reproducing kernels.

\begin{definition}[RKHS] \label{def:RKHS}
Let $\cX$ be a nonempty set and $k$ be a positive definite kernel on $\cX$.
A Hilbert space $\cH_k$ of functions on $\cX$ equipped with an inner-product $\left< \cdot,\cdot\right>_{\cH_k}$ is called a reproducing kernel Hilbert space (RKHS) with reproducing kernel $k$, if the following are satisfied:
\begin{enumerate}
 \item For all $x\in\cX$, we have $k(\cdot,x) \in \cH_k$;
 \item For all $x \in \cX$ and for all $f\in\cH_k$,
 $$f(x) = \langle f,k(\cdot,x)\rangle_{\cH_k} \quad {\rm (Reproducing\ property)}. $$
\end{enumerate}
\end{definition}

\begin{remark}\rm
In Definition \ref{def:RKHS}, $k(\cdot,x)$ with $x$ being fixed is a real-valued function such that $y \mapsto k(y,x)$ for $y \in \cX$.
$k(\cdot,x)$ is called the canonical feature map of $x$, since $k$ can be written as an inner-product in the RKHS as
\begin{equation*} \label{eq:kernel-inner-prod}
k(x,y) = \left<k(\cdot,x), k(\cdot,y)\right>_{\cH_k}, \quad x,y \in \cX,
\end{equation*}
which follows from the reproducing property.
Therefore $k(\cdot,x)$ is a (possibly infinite dimensional) feature representation of $x$.
\end{remark}

\begin{remark}\rm
By the Moore-Aronszajn theorem \citep{Aronszajn1950}, for every positive definite kernel $k$, there exists a unique RKHS $\cH_k$ for which $k$ is the reproducing kernel.
In this sense, RKHSs and positive definite kernels are one-to-one: for each kernel $k$ there exists a uniquely associated RKHS $\cH_k$, and vice versa.
\end{remark}

Given a positive definite kernel $k$, its RKHS $\cH_k$ can be constructed as follows. 
Let $\cH_0$ be the linear span of feature vectors, that is, each function in $\cH_0$ can be expressed as a finite linear combination of feature vectors:
$$
\cH_0 := {\rm span}\left\{ k(\cdot,x) : x \in \cX \right\} = \left\{ f = \sum_{i=1}^n c_i k(\cdot,x_i) : n \in \mathbb{N},\ c_1,\dots,c_n \in \Re,\ x_1,\dots,x_n \in \cX  \right\}.
 $$
One can make $\cH_0$ a pre-Hilbert space, by defining an inner-product as follows: For any 
$f := \sum_{i=1}^n a_i k(\cdot,x_i) \in \cH_0$ and $g := \sum_{j=1}^m b_j k(\cdot,y_j) \in \cH_0$
with $n,m \in \mathbb{N}$, $a_1,\dots,a_n, b_1,\dots,b_m \in \Re$ and $x_1\dots,x_n, y_1,\dots,y_m \in \cX$, the inner-product is defined by
$$
\left< f, g \right>_{\cH_0} := \sum_{i=1}^n \sum_{j=1}^m a_i b_j k(x_i, y_j).
$$
The norm of $\cH_0$ is induced by the inner-product, i.e., $\| f \|_{\cH_0}^2 = \left< f, f \right>_{\cH_0}$.
The RKHS $\cH_k$ associated with $k$ is then defined as the closure of $\cH_0$ with respect to the norm $\| \cdot \|_{\cH_0}$, i.e, $\cH_k := \overline{\cH_0}$.
That is, 
\begin{eqnarray}
\cH_k = \Biggl\{ f = \sum_{i=1}^\infty c_i k(\cdot,x_i) &:& (c_1, c_2 \dots) \subset \Re,\ (x_1, x_2, \dots)  \subset \cX,\ {\rm such\ that} \label{eq:RKHS-linear-span} \\
&& \left\| f \right\|_{\cH_k}^2 := \lim_{n \to \infty} \left\| \sum_{i=1}^n c_i k(\cdot,x_i) \right\|_{\cH_0}^2 = \sum_{i,j = 1}^\infty c_i c_j k(x_i,x_j) < \infty  \Biggr\}. \nonumber
\end{eqnarray}

\begin{remark}\rm
From (\ref{eq:RKHS-linear-span}), it is easy to see that functions $f$ in the RKHS $\cH_k$ inherit the properties of the kernel $k$.
For instance, if the kernel $k$ is $s$-times differentiable for $s \in \mathbb{N}$, then so are the functions in $\cH_k$ \citep[Corollary 4.36]{Steinwart2008}.
\end{remark}

An important property of the RKHS norm $\| f \|_{\cH_k}$ is that it captures not only the magnitude of a function $f \in \cH_k$, but also its {\em smoothness}: $f$ gets smoother as  $\| f \|_{\cH_k}$ decreases, and vice versa.
This is particularly important in understanding why regularization is required for kernel ridge regression, to avoid overfitting.
This smoothness property of the RKHS norm can be seen in the following example on the RKHS of a Mat\'ern kernel, which follows from \citet[Eq.~4.15]{RasmussenWilliams} and \citet[Corollary 10.48]{Wen05}.
A complete characterization of RKHSs of Mat\'ern kernels involve Fourier transforms the kernels; see Section \ref{sec:fourier_RKHS} for details.

\begin{example}[RKHSs of Mat\'ern kernels: Sobolev spaces] \label{ex:matern-rkhs}
Let $k_{\alpha,h}$ be the Mat\'ern kernel on $\cX \subset \Re^d$ with Lipschitz boundary\footnote{For the definition of Lipschitz boundary, see e.g., \citet[p.189]{Ste70}, \citet[Definition 4.3]{Tri06} and \citet[Definition 3]{KanSriFuk17}.} in Example \ref{ex:matern-kernel} with parameters $\alpha > 0$ and $h > 0$ such that $s := \alpha + d/2$ is an integer.
Then the RKHS $\cH_{k_{\alpha, h}}$ of $k_{\alpha, h}$ is norm-equivalent\footnote{Normed vector spaces $\cH_1$ and $\cH_2$ are called norm-equivalent, if $\cH_1 = \cH_2$ as a set, and if there are constants $c_1, c_2 > 0$ such that $c_1 \| f \|_{\cH_2} \leq \| f \|_{\cH_1} \leq c_2 \| f \|_{\cH_2}$ holds for all $f \in \cH_1 = \cH_2$, where $\| \cdot \|_{\cH_1}$ and $\| \cdot \|_{\cH_2}$ denote the norms equipped with $\cH_1$ and $\cH_2$, respectively.} to the Sobolev space $W_2^s (\cX)$ of order $s$ defined by 
\begin{equation} \label{eq:sobolev_space}
W_2^s (\cX) := \left\{ f \in L_2(\cX) :\ \| f \|_{W_2^s (\cX)}^2 := \sum_{\beta \in \mathbb{N}_0^d: | \beta | \leq s} \left\| D^\beta f \right\|_{L_2(\cX)}^2 < \infty \right\}.
\end{equation}
That is, we have $\cH_{k_{\alpha,h}} = W_2^s (\cX)$ as a set of functions, and there exist constants $c_1, c_2 > 0$ such that 
\begin{equation} \label{eq:matern_rkhs-norm}
c_1 \| f \|_{W_2^s(\cX)} \leq  \| f \|_{\cH_{k_{\alpha,h}}} \leq c_2 \| f \|_{W_2^s(\cX)}, \quad \forall f \in \cH_{k_{\alpha,h}},
\end{equation}
\end{example}
\begin{remark}\rm
The inequality \eqref{eq:matern_rkhs-norm} shows the equivalence of the RKHS norm  $\| f \|_{\cH_{k_{\alpha,h}}}$ and the Sobolev norm $\| f \|_{W_2^s (\cX)}^2$ defined in \eqref{eq:sobolev_space}.
Thus the RKHS norm $\| f \|_{\cH_{k_{\alpha,h}}}$ captures the smoothness of the function $f$ with parameter $\alpha$ specifying the order of differentiability.
That is, $\| f \|_{\cH_{k_{\alpha,h}}}$ takes into account weak derivatives up to order $s = \alpha + d/2$ of the function $f$.
For details of Sobolev spaces, see e.g.~\citet{AdaFou03} .
\end{remark}
\begin{remark}\rm
The RKHS $\cH_{k_{\alpha,h}}$ consists of functions that are weak differentiable up to order $s = \alpha + d/2$.
Here one should not confuse the weak differentiability with the classic notion of differentiability.
In the classical sense, functions in $\cH_{k_{\alpha,h}}$ are only guaranteed to be differentiable up to order $\alpha$, not $s = \alpha + d/2$; this is a consequence of the Sobolev embedding theorem \citep[Theorem 4.12]{AdaFou03}. 
For definition of weak derivatives, see e.g.~\citet[Section 1.62]{AdaFou03}.
For instance, consider the case $\alpha = 1/2$, where the kernel is given by \eqref{eq:Laplace-kernel} and is not differentiable at origin.
By Definition \ref{def:RKHS}, we have $k_{1/2,h}(\cdot,x) \in \cH_{k_{1/2,h}}$ for any $x \in \mathbb{R}^d$; this implies that $\cH_{k_{1/2,h}}$ contain functions that are not differentiable in the classical sense.
\end{remark}


\subsection{A Spectral Characterization for RKHSs Associated with Shift-Invariant Kernels}
\label{sec:fourier_RKHS}

We provide a characterization of RKHSs associated with shift-invariant kernels on $\cX = \Re^d$.
Recall that a kernel $k$ is called shift-invariant, if it can be written as $k(x,y) = \Phi(x-y)$ for all $x,y \in \Re^d$ with a positive definite function $\Phi:\Re^d \to \Re$.
In the following, the key role is played by the Fourier transform $\cF[\Phi]$ of this positive definite function.

Theorem \ref{theo:RKHS-shift-invariant} below provides a characterization of the RKHS of a shift-invariant kernel in terms of the Fourier transform $\cF[\Phi]$ of $\Phi$.
This result is available from, e.g., \citet[Lemma 3.1]{kimeldorf1970correspondence} and \citet[Theorem 10.12]{Wen05}.
\begin{theorem} \label{theo:RKHS-shift-invariant}
Let $k$ be a shift-invariant kernel on $\cX = \Re^d$ such that $k(x,y) := \Phi(x-y)$ for $\Phi \in C(\mathbb{R}^d) \cap L_1(\mathbb{R}^d)$. 
Then the RKHS $\cH_k$ of $k$ is given by
\begin{equation} \label{eq:RKHS-shift-invariant}
\cH_k = \left\{ f \in L_2( \Re^d ) \cap C(\Re^d):\ \| f \|_{\cH_k}^2 = \frac{1}{ (2\pi)^{d/2} }  \int \frac{ |\cF[f](\omega)|^2 }{\cF[\Phi](\omega)} d\omega < \infty \right\}, 
\end{equation}
with the inner-product being
$$
\left<f, g \right>_{\cH_k} =  \frac{1}{ (2\pi)^{d/2} } \int \frac{\cF [f] (\omega) \overline{\cF[g](\omega)}}{\cF [\Phi](\omega) }  d\omega,\quad f,g \in \cH_k,
$$
where $\overline{\cF[g](\omega)}$ denotes the complex conjugate of $\cF[g](\omega)$.
\end{theorem}

\begin{remark}\rm
Theorem \ref{theo:RKHS-shift-invariant} shows that the Fourier transform $\cF[\Phi]$ determines the members of the RKHS. 
More specifically, the requirement in \eqref{eq:RKHS-shift-invariant} shows that, if $\cF[\Phi](\omega)$ decays quickly as $|\omega| \to \infty$, the Fourier transform $\cF[f](\omega)$ of each $f \in \cH_k$ should also decay quickly as $|\omega| \to \infty$. 
Since the tail behaviors of $\cF[\Phi]$ and $\cF[f]$ determines the smoothness of $\Phi$ and $f$ respectively, this implies that if $\Phi$ is smooth, $f$ should also be smooth; see examples below.
\end{remark}
\begin{remark}\rm
The Fourier transform $\cF[\Phi](\omega)$ is known as the {\em power spectral density} in the stochastic process literature; see e.g.~\citet[Section 3.3]{Bre14}.
It can be written in terms of a certain Fourier transform  of the Gaussian process $\ssf \sim \GP(0,k)$ \citep[p.161]{Bre14}.
We do not explain it in detail, since it requires an explanation of a certain stochastic integral \citep[Theorem 3.4.1]{Bre14}, which is out of the scope of this paper.
\end{remark}
The following examples illustrate Theorem \ref{theo:RKHS-shift-invariant}, providing spectral characterizations for RKHSs of square-exponential and Mat\'ern kernels.
\begin{example}[RKHSs of square-exponential kernels]
Let $k_\gamma(x,y) := \Phi_\gamma(x-y) := \exp( - \| x - y \|^2 / \gamma^2 )$ be the square-exponential kernel with bandwidth $\gamma > 0$ in Example \ref{ex:square-exp-kernel}, and let $\cH_{k_\gamma}$ be the associated RKHS.
The Fourier transform of $\Phi_\gamma$ is given by
\begin{equation*} \label{eq:Gaussian-fourier}
\cF[\Phi_\gamma](\omega) = C_{d,\gamma} \exp(-\gamma^2 \| \omega \|^2 / 4), \quad \omega \in \Re^d,
\end{equation*}
where $C_{d,\gamma}$ is a constant depending only on $d$ and $\gamma$; see e.g.~\citet[Theorem 5.20]{Wen05}.
Therefore the RKHS $\cH_{k_\gamma}$ can be written as
$$
\cH_{k_\gamma} = \left\{ f \in L_2( \Re^d ) \cap C(\Re^d):\ \| f \|_{\cH_{k_\gamma}}^2 = \frac{1}{ (2\pi)^{d/2} C_{d,\gamma} }  \int  |\cF[f](\omega)|^2  \exp(\gamma^2 \| \omega \|^2 / 4)  d\omega < \infty \right\},
$$
which shows that, for any $f \in \cH_{k_\gamma}$, the magnitude of its Fourier transform $|\cF[f](\omega)|$ decays exponentially fast as $|\omega| \to \infty$, and the speed of decay gets quicker as $\gamma$ increases.
\end{example}
\begin{example}[RKHSs of Mat\'ern kernels]
Let $k_{\alpha, h}$ the Mat\'ern kernel on $\Re^d$ with parameters $\alpha > 0$ and $h > 0$ in Example \ref{ex:matern-kernel}, and let $\cH_{k_{\alpha,h}}$ of $k_{\alpha,h}$ be the associated RKHS.
Then $k_{\alpha,h}(x,y) = \Phi_{\alpha,h} (x-y)$ with $\Phi_{\alpha,h}(x) := \frac{2^{1-\alpha}}{\Gamma(\alpha)} ( \sqrt{2\alpha} \| x \| / h ) K_\alpha( \sqrt{2\alpha} \|x\| / h )$, and the Fourier transform of $\Phi_{\alpha, h}$ is given by  
\begin{equation} \label{eq:Fourier-matern}
\cF[\Phi_{\alpha, h}](\omega) = C_{\alpha,h,d}\ \left( \frac{2\alpha}{h^2} + 4 \pi^2 \| \omega \|^2 \right)^{- \alpha - d/2}, \quad \omega \in \Re^d,
\end{equation}
where $C_{\alpha,h,d}$ is a constant depending only on $\alpha$, $h$ and $d$; see, e.g., \citet[Eq.~4.15]{RasmussenWilliams}.
Therefore the RKHS $\cH_{k_{\alpha,h}}$ can be written as
\begin{eqnarray*}
 \cH_{k_{\alpha,h}} & = & \Biggl\{  f \in L_2( \Re^d ) \cap C(\Re^d):  \\ 
&& \| f \|_{\cH_{k_{\alpha,h}}}^2 =  \frac{1}{ (2\pi)^{d/2} C_{\alpha,h,d} }  \int  |\cF[f](\omega)|^2  \left( \frac{2\alpha}{h^2} + 4 \pi^2 \| \omega \|^2 \right)^{\alpha + d/2}  d\omega < \infty \Biggr\},
\end{eqnarray*}
which shows that, for any $f \in \cH_{k_{\alpha,h}}$, the magnitude of its Fourier transform $|\cF[f](\omega)|$ decays polynomially fast as $|\omega| \to \infty$, and the speed of decay gets quicker as $\alpha$ increases.
Moreover, from \eqref{eq:Fourier-matern} and \citet[Corollary 10.48]{Wen05}, it follows that $\cH_{k_{\alpha,h}}$ is norm-equivalent to the Sobolev space of order $\alpha+d/2$. 
\end{example}

\section{Connections between Gaussian Process and Kernel Ridge Regression} 
\label{sec:regression}


\begin{figure}[t]
  \centering \scriptsize
	\includegraphics[width=\linewidth]{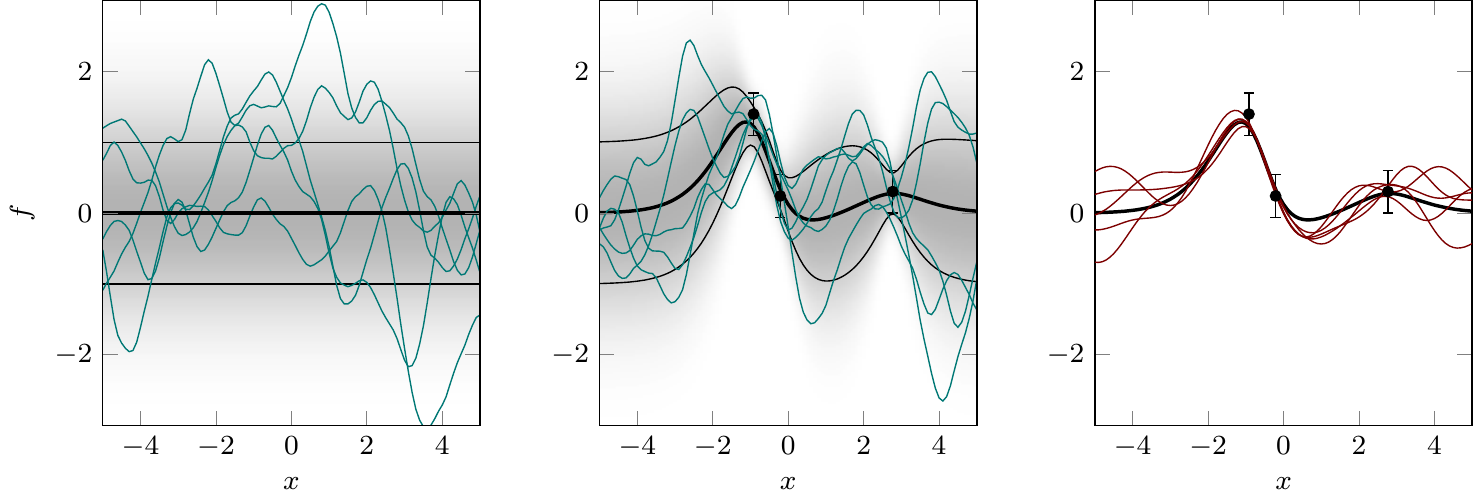}
  \caption{Conceptual sketches of Gaussian process regression (left, center) and kernel ridge regression (right).
  \textbf{Left:} Prior measure $\ff \sim \GP(0,k)$ with vanishing prior mean and the Matérn-class kernel $k(x,x') = (1 + \sqrt{5}r + 5/3 r^2) \exp(-\sqrt{5}r)$ with $r\ce |x-x'|$. Prior mean function in thick black. 
  Two marginal standard deviations in thin black. 
  Marginal densities as gray shading. 
  5 samples from prior as green lines. 
  \textbf{Center:} Given a dataset $(X,Y)$ of $n=3$ data points with i.i.d.~zero-mean normal noise of standard deviation $\sigma=0.1$, the posterior measure is also a Gaussian process, with updated mean and covariance functions (all quantities as on the left). 
  \textbf{Right:} Kernel ridge regression yields a point estimate (thick black) that is exactly equal to the Gaussian process posterior mean. 
  In contrast to Gaussian process regression, an error estimate is usually not provided. 
  This absence can be deliberate, as one may not be willing to impose the assumptions necessary to define such an estimate (e.g., additive Gaussian noise assumption).
  For comparison with the GP samples, the plot also shows some functions with the property that $f_X\Trans k_{XX}^{-1} f_X = \|f_X\|_{k_{XX}^{-1}} ^2 = 1$ (but KRR does not assume the true function is of this form).} 
  \label{fig:GP_intro}
\end{figure}

\emph{Regression} is a fundamental task in statistics and machine learning.
The {\em interpolation} problem is regression with noise-free observations and has been studied mainly in the literature on numerical analysis, and more recently in the context of Bayesian optimization.
We compare two approaches to these problems based on Gaussian processes and kernel methods, namely {\em Gaussian process regression} and {\em kernel ridge regression} (see also Figure \ref{fig:GP_intro} for illustration). 

We first describe the problem of regression, and set notation.
Let $\cX$ be a nonempty set and $\ff:\cX \to \Re$ be a function.
Assume that one is given a set of pairs $(x_i,y_i)_{i=1}^n \subset \cX \times \Re$ for $n \in \mathbb{N}$, which is referred to as {\em training data}, such that 
\begin{equation} \label{eq:regres_noise_model}
 y_i = \ff(x_i) + \xi_i,\qq i=1,\dots,n,
\end{equation}
where $\xi_i$ is a zero-mean random variable that represents ``noise'' in the output, or the variability in the responses which is not explained by the input vectors.
The task of regression is to estimate the unknown function $\ff$ based on the training data $(x_i,y_i)_{i=1}^n$.
The function $\ff$ is called the {\em regression function}, and is the conditional expectation of the output given an input:
$$
\ff(x) = \bE[ y | x ],
$$
where $(x,y)$ is a random variable with the conditional distribution of $y$ given $x$ following the model \eqref{eq:regres_noise_model}.

If there is no output noise, i.e., $\xi_i = 0$, the problem is called {\em interpolation}; in this case one can obtain exact function values $y_i = \ff(x_i)$ for training.
We will frequently use the notation $X := (x_1,\dots,x_n) \in \cX^n$ for the set of input data points, and $Y := (y_1,\dots,y_n)\Trans \in \Re^n$ for the set of outputs (or $\ff_X := (\ff(x_1),\dots,\ff(x_n))\Trans \in \Re^n$ in the noise free case).



This section first reviews Gaussian process regression and interpolation in Section \ref{sec:GP-regression}, and kernel ridge regression and kernel interpolation in Section \ref{sec:kernel-ridge-regression}.
We summarize and discuss equivalences between the two approaches in Section \ref{sec:GP-KRR-equivalence}.
In GP-regression, the {\em posterior variance} function plays a fundamental role, but its kernel interpretation has not been well understood.
In Section \ref{sec:posterior_variance}, we show that there exists an interpretation of the posterior variance function as a certain {\em worst case error} in the RKHS.
Coming back to regression itself, in Section \ref{sec:weight-vec-KRR} we provide a weight-vector viewpoint for the regression problem, and discuss the equivalence between regularization and an additive noise assumption.

\subsection{Gaussian Process Regression and Interpolation} \label{sec:GP-regression}

Gaussian process regression, also known as \emph{Kriging} or \emph{Wiener-Kolmogorov prediction}, is a Bayesian nonparametric method for regression.
Being a Bayesian approach, GP-regression produces a {\em posterior distribution} of the unknown regression function $\ff$, provided the training data $(X,Y)$, a prior distribution $\Pi_0$ on $\ff$, and a likelihood function denoted by $\ell_{X,Y}(\ff)$.
More specifically, the prior $\Pi_0$ is defined as a Gaussian process $\GP(m,k)$ with  mean function $m : \cX \to \Re$ and  covariance kernel $k: \cX \times \cX \to \Re$, i.e., 
\begin{equation} \label{eq:GP-prior}
\ff \sim \GP(m,k).
\end{equation}
Since this GP serves as a prior, the mean function $m$ and the kernel $k$ should be chosen so that they reflect one's prior knowledge or belief about the regression function $\ff$; this will be discussed later.

On the other hand, a likelihood function is defined by a probabilistic model $p(y_i|\ff(x_i))$ for the noise variables $\xi_1,\dots,\xi_n$, since this determines the distribution of the observations $Y = (y_1,\cdots,y_n)\Trans$ with the additive noise model (\ref{eq:regres_noise_model}).
It is typical to assume that $\xi_1,\dots,\xi_n$ are i.i.d.~centered Gaussian random variables with variance $\sigma^2 > 0$:
\begin{equation} \label{eq:Gaussian-noise}
\xi_i \iid \N(0,\sigma^2), \quad i=1,\dots,n. 
\end{equation}
Thus the likelihood function is defined as
\begin{equation} \label{eq:GP-likelihood}
\ell_{X,Y}(\ff) = \prod_{i=1}^n \mathcal{N}(y_i|\ff(x_i),\sigma^2),
\end{equation}
where $\mathcal{N}(\cdot|\mu,\sigma^2)$ denotes the density function of the normal distribution $\N(\mu,\sigma^2)$ of mean $\mu$ and variance $\sigma^2$.
In general however, GP-regression allows the noise variables to be correlated Gaussian with varying magnitudes of variances.

By Bayes' rule, the posterior distribution $\Pi_n(\ff|Y,X)$ is then given as
\begin{equation}
 \label{eq:GP-posterior-bayes}
 d\Pi_n(\ff | X, Y) \propto \ell_{X,Y}(\ff) d\Pi_0(\ff) = \prod_{i=1}^n \mathcal{N}(y_i|\ff(x_i),\sigma^2) d\Pi_0(\ff).
\end{equation}
As shown in the following theorem, which is well known in the literature, the posterior $\Pi_n(\ff|X,Y)$ is again a Gaussian process, whose mean function and covariance function are obtained by simple linear algebra.

\begin{theorem} \label{theo:GP-posterior}
Assume \eqref{eq:regres_noise_model}, \eqref{eq:GP-prior} and \eqref{eq:Gaussian-noise}, and let $X = (x_1,\dots,x_n) \in \cX^n$ and $Y=(y_1,\dots,y_n)\Trans \in \Re^n$.
Then we have
\begin{equation*} 
\ff | Y \sim \GP(\bar{m},\bar{k}), 
\end{equation*}
where $\bar{m}:\cX \to \Re$ and $\bar{k}:\cX \times \cX \to \Re$ are given by
\begin{eqnarray} 
 \bar{m}(x) &=& m(x) + k_{xX}(k_{XX} + \sigma^2 I_n)^{-1}(Y - m_X),\quad x \in \cX,  \label{eq:posteior_mean} \\
  \bar{k}(x,x') &=& k(x,x') - k_{xX}(k_{XX} + \sigma^2 I_n)^{-1} k_{Xx'}, \quad x,x' \in \cX,
  \label{eq:posterior-variance}
\end{eqnarray}
where $k_{Xx}=k^\top_{xX}=(k(x_1,x),\ldots,k(x_n,x))^\top.$
\end{theorem}

As $\GP(\bar{m},\bar{k})$ is a posterior Gaussian process, $\bar{m}$ is referred to as the {\em posterior mean function} and $\bar{k}$ the {\em posterior covariance function}.
It is instructive to see how the Gaussian noise assumptions \eqref{eq:Gaussian-noise} and the GP prior \eqref{eq:GP-prior} lead to the closed form expressions \eqref{eq:posteior_mean} and \eqref{eq:posterior-variance}, because this can be done {\em without relying on Bayes' rule}.
This is important for the following two reasons: (i) Since the prior and posterior are defined on an {\em infinite} dimensional space of functions, Bayes' rule is more involved and thus does not produce the expressions \eqref{eq:posteior_mean} and \eqref{eq:posterior-variance} directly (see e.g.~\citealt[Theorem 6.31]{Stu10}); (ii) When dealing with the {\em noise-free} setting where $\sigma^2 = 0$, Bayes' rule cannot be used because the likelihood function is degenerate \citep{CocOatSulGir17}.

To prove Theorem \ref{theo:GP-posterior}, first recall a basic formula for conditional distributions of Gaussian random vectors (see e.g.~\citealt[Appendix A.2]{RasmussenWilliams}).
\begin{proposition} \label{eq:Gauus-formula}
Let $a \in \Re^n$ and $b \in \Re^m$ be Gaussian random vectors such that
\begin{equation} \label{eq:Gaussian-joint}
\begin{bmatrix} a \\ b \end{bmatrix} \sim \N\left( \begin{bmatrix} \mu_a \\ \mu_b \end{bmatrix}, \begin{bmatrix} A & C \\ C\Trans & B  \end{bmatrix} \right),
\end{equation}
where $\mu_a \in \Re^n$, $\mu_b \in \Re^m$ are the mean vectors,  $A \in \Re^{n \times n}$, $B \in \Re^{m \times m}$ are the covariance matrices (where $A$ is strictly positive definite), and $C \in \Re^{n \times m}$.
Then the conditional distribution of $b$ given $a$ is
\begin{equation} \label{eq:Gaussian-cond}
b | a \sim \N \left( \mu_b + C\Trans A^{-1} (a-\mu_a), B - C\Trans A^{-1} C \right).
\end{equation}
\end{proposition}
\begin{proof}[Theorem \ref{theo:GP-posterior}]
Let $m \in \mathbb{N}$, and let $Z = (z_1,\dots,z_m) \in \cX^m$ be any finite set of points.
Then the observations $Y \in \Re^n$ and GP-function values $\ff_Z = (\ff(z_1),\dots,\ff(z_m))\Trans \in \Re^m$ are jointly Gaussian random vectors such that
$$
\begin{bmatrix}
Y \\
\ff_Z
\end{bmatrix} 
\sim 
\N\left( \begin{bmatrix} m_X \\ m_Z \end{bmatrix},\  \begin{pmatrix} k_{XX} + \sigma^2 I_n & k_{XZ} \\ k_{ZX} & k_{ZZ} \end{pmatrix} \right).
$$
In the notation of \eqref{eq:Gaussian-joint}, this corresponds to $a = Y|X$, $b = \ff_Z$, $\mu_a = m_X$, $\mu_b = m_Z$, $A = k_{XX} + \sigma^2 I_n$,  $B = k_{ZZ}$ and $C =  k_{XZ}$.
Applying the formula \eqref{eq:Gaussian-cond} in Proposition \ref{eq:Gauus-formula}, the conditional distribution of $\ff_Z$ given $Y$ is then given as
$$
\ff_Z | Y \sim \N(\bar{\mu}, \bar{\Sigma}),
$$
where 
\begin{eqnarray*}
\bar{\mu} &:=& m_Z + k_{ZX} (k_{XX} + \sigma^2 I_n)^{-1} (Y - m_X) \in \Re^m, \label{eq:mean-vec-pos} \\
\bar{\Sigma} &:=& k_{ZZ} - k_{ZX} (k_{XX} + \sigma^2 I_n)^{-1} k_{XZ} \in \Re^{m \times m}. \label{eq:cov-mat-pos}
\end{eqnarray*}
This mean vector and the covariance matrix can be written as  $\bar{\mu} := \bar{m}_Z$, $\bar{\Sigma} = \bar{k}_{ZZ}$, where  $\bar{m}:\cX \to \Re$ and $\bar{k}:\cX \times \cX \to \Re$ are  defined as \eqref{eq:posteior_mean} and \eqref{eq:posterior-variance}.
In other words,
\begin{equation} \label{eq:finite-post-dist}
\ff_Z|Y \sim \N(\bar{m}_Z, \bar{k}_{ZZ}).
\end{equation}
Note that \eqref{eq:finite-post-dist} holds for any set of points $Z = (z_1,\dots,z_m) \in \cX^m$ of any size $m \in \mathbb{N}$.
Therefore, by the Kolmogorov extension theorem \citep[Theorems 12.1.2]{Dud02} and the definition of GPs (Definition \ref{def:GP}), this implies that the process $\ssf \sim \GP(m,k)$ {\em conditioned on} the training data $X,Y$ is a draw from $\GP(\bar{m},\bar{k})$.
\end{proof}

\begin{remark}\rm
Given a test input $x$, prediction of the output value $\ssf(x)$ is carried out by evaluating the posterior mean function \eqref{eq:posteior_mean}, as we have $\bE[\ff(x) | X, Y ] = \bar{m}(x)$ by definition.
On the other hand, the posterior covariance $\bar{k}$ can be used to quantify uncertainties over output values; this will be discussed in Section \ref{sec:posterior_variance}.
\end{remark}

\begin{remark}\rm
As it can be seen from the expressions \eqref{eq:posteior_mean} and \eqref{eq:posterior-variance}, $\bar{m}$ and $\bar{k}$ depend on the choice of the prior mean function $m$, the kernel $k$ and the noise variance $\sigma^2$.
These are hyper-parameters of GP-regression, and the determination of them can be carried out, for example, by the empirical Bayes method, i.e., maximization of the marginal likelihood of the data given hyperparameters (for regression this is available in closed form); see \citet{RasmussenWilliams} for details.
\end{remark}


\paragraph{Noise-free case: Gaussian process interpolation.}
Consider the noise-free case where exact function values $y_i = \ff(x_i)$, $i=1,\dots,n$ are provided for training.
In this case, the likelihood function \eqref{eq:GP-likelihood} is degenerate and thus not well-defined, since the distribution of $y_i$ given $\ff(x_i)$ is the Dirac distribution at $\ff(x_i)$, which has no density function.
Thus, it is not possible to apply Bayes' rule to derive the posterior distribution of $\ff$ as in \eqref{eq:GP-posterior-bayes}; see also \citet[Section 2.5]{CocOatSulGir17}.
However, as the proof for Theorem \ref{theo:GP-posterior} indicates, the conditional distribution of $\ff$ given training data $(x_i,\ff(x_i))_{i=1}^n$ can be derived based on Gaussian calculus, without relying on Bayes' rule.
The resulting posterior mean function and covariance function are respectively given as \eqref{eq:posteior_mean} and \eqref{eq:posterior-variance} with $\sigma^2 = 0$, as shown in the following theorem.
\begin{theorem} \label{theo:GP-interpolation}
Assume \eqref{eq:GP-prior}, and let $X = (x_1,\dots,x_n) \in \cX^n$ and $\ff_X = (\ff(x_1),\dots,\ff(x_n))\Trans \in \Re^n$.
Moreover, assume that the kernel matrix $k_{XX} = (k(x_i,x_j))_{i,j=1}^n \in \Re^{n \times n}$ is invertible.
Then the conditional distribution of  $\ff$ given $(X,\ff_X)$ is a Gaussian process
$$
\ff~|~\ff_X \sim \GP(\bar{m},\bar{k}), 
$$
where $\bar{m}:\cX \to \Re$ and $\bar{k}:\cX \times \cX \to \Re$ are given by
\begin{eqnarray} 
 \bar{m}(x) &=& m(x) + k_{xX} k_{XX}^{-1}(\ff_X - m_X),\quad x \in \cX, \label{eq:posteior_mean_interp} \\
  \bar{k}(x,x') &=& k(x,x') - k_{xX} k_{XX}^{-1} k_{Xx'}, \quad x,x' \in \cX. \label{eq:posterior-variance_interp}
\end{eqnarray}
\end{theorem}

\begin{proof}
Since $k_{XX}$ is assumed to be invertible, the assertion can be proven by modifying the proof of Theorem \ref{theo:GP-posterior}.
Specifically, this can be done by replacing $k_{XX}+\sigma^2 I_n$ in the proof of Theorem \ref{theo:GP-posterior} by $k_{XX}$, and $Y$ by $\ssf_X$.
\end{proof}

\begin{remark}\rm
In Theorem \ref{theo:GP-interpolation}, the kernel matrix $k_{XX}$ is required to be invertible. 
If this condition is not satisfied, then the expressions \eqref{eq:posteior_mean_interp} are \eqref{eq:posterior-variance_interp} not well-defined.
For instance, $k_{XX}$ is not invertible, if some of the points in $X = (x_1,\dots,x_n)$ are identical, or if the kernel $k$ is a polynomial kernel of order $m$ such that $n > m$.
\end{remark}

This way of using Gaussian processes in modeling {\em deterministic} functions is becoming popular in machine learning, in particular in the context of Bayesian optimization (e.g.,  \citeauthor{Bul11}, \citeyear{Bul11}) as well as in the emerging field of probabilistic numerics \citep{HenOsbGirRSPA2015}: For instance, Bayesian quadrature, a probabilistic numerics approach to numerical integration, involves integration of a fixed deterministic function, which is modeled as a Gaussian process with noise-free outputs; see Section \ref{sec:kernel_and_bayesian_quadrature} for details.

The noise-free situation appears for instance when a measurement equipment for the output values is very accurate, or when the function values are obtained as a result of computer experiments.
In the latter case, GP-interpolation is often called {\em emulation} in the literature.
In these situations, typically the function of interest is very expensive to evaluate, so inference should be done based on a small number of function evaluations.
Gaussian processes are useful for this purpose, since one can gain statistical efficiency by incorporating available prior knowledge about the function via the choice of a covariance kernel.


\begin{remark}\rm
For the noise-free case, a posterior distribution may be well-defined by assuming the existence of very small noise in outputs, which corresponds to applying regularization with a very small regularization constant; this is called ``jitter'' in the kriging literature.
This is practically reasonable, since if the kernel matrix $k_{XX}$ is singular (or close to singular,  leading to numerical issues), then the posterior mean \eqref{eq:posteior_mean_interp} as well as the posterior variance \eqref{eq:posterior-variance_interp} are not well-defined without regularization. 
\end{remark}

\subsection{Kernel Ridge Regression and Kernel Interpolation}
\label{sec:kernel-ridge-regression}
Kernel ridge regression (KRR), which is also known as \emph{regularized least-squares} \citep{CapDev07} or \emph{spline smoothing} \citep{wahba1990spline}, arises as a regularized empirical risk minimization problem where the hypothesis space is chosen to be an RKHS $\cH_k$.
That is, we are interested in solving the problem
\begin{equation*}
 \label{eq:erm}
 \hat{f} = \argmin_{f\in\cH_k} \frac{1}{n} \sum_{i=1}^n L\left(x_i,y_i,f(x_i)\right) + \lambda \|f\|_{\cH_k} ^2. 
\end{equation*}
where $L\colon \cX \times \Re \times \Re \to \Re^+$ is a loss function, and $\lambda > 0$ is a regularization constant. 
The loss function penalizes the deviations between predicted outputs $f(x_i)$ and true outputs $y_i$. 
The regularization constant $\lambda$ controls the  smoothness of the estimator, to avoid overfitting: the larger the $\lambda$ is, the smoother the resulting estimator $\hat{f}$ becomes.
Regularization is necessary, as nonparametric estimation of a function from a finite sample is an ill-posed inverse problem, given also that output values are contaminated by noise.


The KRR estimator then arises when using the square loss $L(x,y,y')=(y-y')^2$:
\begin{equation}
 \label{eq:square-loss}
 \hat{f} = \argmin_{f\in\cH_k} \frac{1}{n} \sum_{i=1}^n ( f(x_i) - y_i )^2 + \lambda \|f\|_{\cH_k} ^2. 
\end{equation}
While this least-square problem is over the function space $\cH_k$, which may be infinite dimensional, its solution can be obtained by simple linear algebra, as the following theorem shows.
As it is simple and instructive, we show its proof based on the representer theorem \citep{SchHerSmo01}.

\begin{theorem} \label{theo:KRR-estimator}
If $\lambda > 0$, the solution to \eqref{eq:square-loss} is unique as a function, and is given by
\begin{equation} \label{eq:KRR_estimator}
   \hat{f} (x) = k_{xX}(k_{XX} + n \lambda \Id_n)^{-1} Y = \sum_{i=1}^n \alpha_i k(x,x_i), \quad x \in \cX,
\end{equation}
where 
\begin{equation} \label{eq:KRR-coefficients}
(\alpha_1,\dots,\alpha_n)\Trans := (k_{XX} + n \lambda \Id_n)^{-1} Y \in \Re^n.
\end{equation}
If we further assume that $k_{XX}$ is invertible, then the coefficients $(\alpha_1,\dots,\alpha_n)$ in \eqref{eq:KRR_estimator} are uniquely given by \eqref{eq:KRR-coefficients}.
\end{theorem}

\begin{proof}
Because of the regularization term in \eqref{eq:square-loss}, one can apply the representer theorem \citep[Theorem 1]{SchHerSmo01}.
This implies that the solution to \eqref{eq:square-loss} 
can be written as a weighted sum of feature vectors $k(\cdot,x_1),\dots,k(\cdot,x_n)$, i.e.,
\begin{equation} \label{eq:KRR-estimator-weighted}
\hat{f} = \sum_{i=1}^n \alpha_i k(\cdot,x_i),
\end{equation}
for some coefficients $\alpha_1,\dots,\alpha_n \in \Re$.
Let ${\bm \alpha} := (\alpha_1,\dots,\alpha_n)\Trans \in \Re^n$.
By substituting the expression \eqref{eq:KRR-estimator-weighted} in \eqref{eq:square-loss}, the optimization problem now becomes
\begin{equation} \label{eq:KRR-finite-dim-objective}
\min_{{\bm \alpha} \in \Re^n} \frac{1}{n} \left[ {\bm \alpha}\Trans k_{XX}^2 {\bm \alpha} - 2 {\bm \alpha}\Trans k_{XX} Y + \| Y \|^2 \right] + \lambda {\bm \alpha}\Trans k_{XX} {\bm \alpha},
\end{equation}
where we used ${\bm \alpha}\Trans k_{XX} {\bm \alpha} = \| \hat{f} \|_{\cH_k}^2$, which follows from the reproducing property.
Differentiating this objective function with respect to ${\bm \alpha}$, setting it equal to $0$ and arranging the resulting equation yields
\begin{equation} \label{eq:KRR-finite-dim-opt}
k_{XX} (k_{XX} + n \lambda I_n) {\bm \alpha} = k_{XX} Y.
\end{equation}
Obviously ${\bm \alpha} = (k_{XX} + n \lambda \Id_n)^{-1} Y$ is one of the solutions to \eqref{eq:KRR-finite-dim-opt}. 
Since the objective function in \eqref{eq:KRR-finite-dim-objective} is a convex function of ${\bm \alpha}$ (while it may not be strictly convex unless $k_{XX}$ is strictly positive definite or invertible), ${\bm \alpha}$ attains the minimum of the objective function.
Since the objective function in \eqref{eq:KRR-finite-dim-objective} is equal to that of \eqref{eq:square-loss}, the function \eqref{eq:KRR-estimator-weighted} with ${\bm \alpha} = (k_{XX} + n \lambda \Id_n)^{-1} Y$ attains the minimum of \eqref{eq:square-loss}.

Note that since the square loss is convex\footnote{A loss function $L: \cX \times \Re \times \Re \to \Re$ is called convex, if $L(x,y,\cdot): \Re \to \Re$ is convex for all fixed $x \in \cX$ and $y \in \Re$ \citep[Definition 2.12]{Steinwart2008}}, the solution to \eqref{eq:square-loss} is unique as a function \citep[Theorem 5.5]{Steinwart2008}.
Hence \eqref{eq:KRR-estimator-weighted} with ${\bm \alpha} = (k_{XX} + n \lambda \Id_n)^{-1} Y$ gives the unique solution to \eqref{eq:square-loss} as a function, and this proves the first claim.
If $k_{XX}$ is further invertible, \eqref{eq:KRR-finite-dim-opt} reduces to $ (k_{XX} + n \lambda I_n) {\bm \alpha} =  Y$, from which the second claim follows.

\end{proof}

\begin{remark}\rm

While Theorem \ref{theo:KRR-estimator} shows that \eqref{eq:KRR_estimator} is the unique solution of \eqref{eq:square-loss} as a {\em function}, this does not mean that the {\em coefficients} $\alpha_1, \dots, \alpha_n$ in \eqref{eq:KRR_estimator} are uniquely determined, unless the kernel matrix $k_{XX}$ is invertible.
This is because there may be multiple solutions to the linear system \eqref{eq:KRR-finite-dim-opt}, if $k_{XX}$ is not invertible. 
(More precisely, if $(k_{XX} + n \lambda I_n){\bm \alpha}' - Y$ is in the null space of $k_{XX}$, such an ${\bm \alpha}'$ is a solution to \eqref{eq:KRR-finite-dim-opt}, even when $(k_{XX} + n \lambda I_n){\bm \alpha}' - Y \not= 0$.)
However, even when multiple solutions to \eqref{eq:KRR-finite-dim-opt} exist, they result in the same estimator \eqref{eq:KRR_estimator} as a function, which can be shown as follows.
Therefore one can always use the coefficients given in \eqref{eq:KRR-coefficients}.

Let ${\bm \alpha}' := (\alpha_1',\dots,\alpha_n')\Trans \in \Re^n$ be another solution to \eqref{eq:KRR-finite-dim-opt}.
As mentioned in the proof, since the objective function \eqref{eq:KRR-finite-dim-objective} is a convex function of ${\bm \alpha}$, this solution ${\bm \alpha}'$ also attains the minimum of the objective function in \eqref{eq:KRR-finite-dim-objective}, and thus the resulting function $\hat{f}' := \sum_{i=1}^n \alpha_i' k(\cdot,x_i)$ attains the minimum of the objective function in \eqref{eq:square-loss}.
However, since the solution to \eqref{eq:square-loss} is unique \citep[Theorem 5.5]{Steinwart2008}, we have $\hat{f} = \hat{f}'$, where $\hat{f}$ is the KRR estimator \eqref{eq:KRR_estimator}.

\end{remark}



\paragraph{Noise-free case: Kernel interpolation}

In the noise-free case where $y_i = \ff(x_i)$, $i=1,\dots,n$, the estimator of $\ff$ is given by \eqref{eq:KRR_estimator} with $\lambda = 0$; that is, 
\begin{equation} \label{eq:KRR_interpolation}
   \hat{f} (x) = k_{xX}k_{XX}^{-1}  \ff_X, = \sum_{i=1}^n \alpha_i k(x,x_i), \quad x \in \cX,
\end{equation}
where $\ff_X := (\ff(x_1),\dots,\ff(x_n))\Trans \in \Re^n$ and $(\alpha_1,\dots,\alpha_n)\Trans := k_{XX}^{-1} \ff_X$.
Thus, in this case the kernel matrix $k_{XX}$ is required to be invertible.
The estimator \eqref{eq:KRR_interpolation} is obtained as a solution of the following optimization problem in the RKHS.
We provide a proof based on that of \citet[Theorem 58 in p.~112]{Berlinet2004}.

\begin{theorem} \label{theo:kernel-interpolation}
Let $k$ be a kernel on a nonempty set $\cX$, and $X = (x_1,\dots,x_n) \in \cX^n$ be such that the kernel matrix $k_{XX}$ is invertible.
Then \eqref{eq:KRR_interpolation} is the unique solution of the following optimization problem:
\begin{equation} \label{eq:kernel-interp-opt}
\hat{f} := \argmin_{f \in \cH_k} \| f \|_{\cH_k}\quad {\rm subject\ to}\ \quad f(x_i) = \ff(x_i), \quad i = 1,\dots,n.
\end{equation}
\end{theorem}

\begin{proof}
Let $\mathcal{S}_0$ be the linear span of the feature vectors $k(\cdot,x_1),\dots,k(\cdot,x_n)$, that is,
$$
\mathcal{S}_0 := \left\{f = \sum_{i=1}^n \alpha_i k(\cdot, x_i): \quad \alpha_1,\dots,\alpha_n \in \Re  \right\}.
$$
Let $\cH_0$ be the set of all functions in $\cH_k$ that interpolate the data $(x_i,\ff(x_i))_{i=1}^n$:
$$
\cH_0 := \left\{ f \in \cH_k:\quad f(x_i) = \ff(x_i), \quad i = 1,\dots,n \right\}.
$$
It is easy to see that \eqref{eq:KRR_interpolation} satisfies $\hat{f}(x_\ell) = \ff(x_\ell)$ for all $\ell = 1,\dots,n$, and thus $\hat{f} \in \mathcal{S}_0 \cap \cH_0$. 

We first show that $\mathcal{S}_0 \cap \cH_0$ consists only of $\hat{f}$. 
To this end, assume that there exist another $g \in  \mathcal{S}_0 \cap \cH_0$, and let $g = \sum_{i=1}^n \beta_i k(\cdot,x_i)$ for $\beta_1,\dots,\beta_n \in \Re$.
Then
$$
\hat{f} - g = \sum_{i=1}^n (\alpha_i-\beta_i) k(\cdot, x_i) \in \mathcal{S}_0 .
$$
On the other hand, since $\hat{f}(x_\ell) = g(x_\ell) = \ff(x_\ell)$ for all $\ell=1,\dots,n$, we have 
$$
\hat{f}(x_\ell) - g(x_\ell) =  \left<\hat{f}-g, k(\cdot,x_\ell) \right>_{\cH_k} = 0, \quad \forall \ell = 1,\dots,n.
$$
This implies that $\hat{f}-g \in \mathcal{S}_0^\perp$, where $\mathcal{S}_0^\perp \subset \cH_k$ is the orthogonal complement of $\mathcal{S}_0$.
Therefore $\hat{f}-g \in \mathcal{S}_0 \cap \mathcal{S}_0^\perp = \{ 0 \}$, which implies that $\hat{f} = g$.
Thus, $\mathcal{S}_0 \cap \cH_0 = \{ \hat{f} \}$.

Finally, we show that $\bar{f}$ is the solution of \eqref{eq:kernel-interp-opt}.
It is easy to show that $\cH_0$ is convex and closed.
Thus there exists an element $f^* \in \cH_0$ such that
$$
f^* = \argmin_{f \in \cH_0} \| f \|_{\cH_k}.
$$
For any $v \in \mathcal{S}_0^\perp$, we have 
$\left< f^* + v,  k(\cdot,x_\ell) \right>_{\cH_k} = \left< f^*,  k(\cdot,x_\ell) \right>_{\cH_k} = f^*(x_\ell) = \ff(x_\ell)$ for all $\ell = 1,\dots,n$, and thus $f^* + v \in \cH_0$. 
By definition, $\| f^* \|_{\cH_k} \leq \| f^* + v \|_{\cH_k}$, and this holds for all $v \in \mathcal{S}_0^\perp$.
This implies that $f^*$ belongs to the orthogonal complement of $\mathcal{S}_0^\perp$, which is $\mathcal{S}_0$ since $\mathcal{S}_0$ is closed.
That said, $f^* \in \mathcal{S}_0$ and thus $f^* \in \cH_0 \cap \mathcal{S}_0 = \{ \hat{f} \}$, which implies $f^* = \hat{f}$.
\end{proof}

\begin{remark}\rm
As the RKHS norm $\| f \|_{\cH_k}$ quantifies the smoothness of the function $f \in \cH_k$ (see Example \ref{ex:matern-rkhs} and Section \ref{sec:theory} for this property of the RKHS norm), the solution to the optimization problem \eqref{eq:kernel-interp-opt} is interpreted as the smoothest function in the RKHS that passes all the training data $(x_1, \ff(x_1)),\dots,(x_n, \ff(x_n))$.
\end{remark}

\begin{remark}\rm
In practice, regularization for matrix inversion in \eqref{eq:KRR_interpolation} may still be needed even if there is no output noise, for the sake of numerical stability when the kernel matrix is nearly singular.
For this purpose, nevertheless, the regularization constant may be chosen to be very small, since it is not relevant to the variance of output noise. 
See \citet[Section 3.4]{WenRie05} and \citet[Section 7.8]{SchWen06} for theoretical supports.
\end{remark}

\subsection{Equivalences in Regression and Interpolation}
\label{sec:GP-KRR-equivalence}
From the expressions \eqref{eq:posteior_mean} and \eqref{eq:KRR_estimator}, it is immediate that the following equivalence holds for GP-regression and kernel ridge regression.
While this result has been well known in the literature, we summarize it in the following proposition.
\begin{proposition} \label{prop:equivalnce}
Let $k$ be a positive definite kernel on a nonempty set $\cX$ and $(x_i,y_i)_{i=1}^n \subset \cX \times \Re$ be training data, and define $X:= (x_1,\dots,x_n) \in \cX^n$ and $Y:=(y_1,\dots,y_n)\Trans \in \Re^n$.
Then we have $\bar{m} = \hat{f}$ if $\sigma^2 = n\lambda$, where
\begin{itemize}
\item $\bar{m}$ is the posterior mean function \eqref{eq:posteior_mean} of GP-regression based on $(X,Y)$, the GP prior $\ff \sim \GP(0,k)$ and the modeling assumption \eqref{eq:regres_noise_model}, where $\xi_1,\dots,\xi_n \iid \mathcal{N}(0,\sigma^2)$ with variance $\sigma^2 > 0$;
\item $\hat{f}$ is the solution \eqref{eq:KRR_estimator} to kernel ridge regression \eqref{eq:square-loss} based on $(X,Y)$, the RKHS $\cH_k$, and  regularization constant $\lambda > 0$.
\end{itemize}
\end{proposition}

\begin{remark}\rm
One immediate consequence of Proposition \ref{prop:equivalnce} is that the posterior mean function $\bar{m}$ belongs to the RKHS $\cH_k$, under the assumptions in Proposition \ref{prop:equivalnce}.
On the other hand, it is well known that a sample $\ssf \sim \GP(\bar{m},\bar{k})$ from the posterior GP does not belong to $\cH_k$ almost surely; we will discuss why this is the case in Section \ref{sec:theory}, and see nevertheless that the GP sample belongs to a certain RKHS induced by $\cH_k$, which is larger than $\cH_k$.
\end{remark}

\begin{remark}\rm
Proposition \ref{prop:equivalnce} implies that the additive Gaussian noise assumption \eqref{eq:regres_noise_model} in GP regression plays the role of regularization in KRR, as the two estimators are identical if $\lambda = \sigma^2/n$.
Recall that $\lambda$ controls the smoothness of $\hat{f}$: as $\lambda$ increases, $\hat{f}$ gets smoother.
Therefore, the assumption that the noise variance $\sigma^2$ is large amounts to the assumption that the latent function $\ff$ is smoother than the observed process.
This interpretation may be explained in the following way.
Let $\eta$ be the zero-mean Gaussian process with a covariance kernel $\delta: \cX \times \cX \to \Re$ defined by 
\begin{equation} \label{eq:Kronecker-delta-kernel}
\delta(x,x') = 
\begin{cases}
1 \quad (x = x') \\
0 \quad (x \not= x').
\end{cases}
\end{equation}
Note that this is a valid kernel, since it is positive definite.
Then the noise variable $\xi_i$ can be written as $\xi_i = \sigma \eta(x_i)$, since this results in $\xi_1,\dots,\xi_n \iid \mathcal{N}(0,1)$.
Define a Gaussian process $\ssy$ by 
\begin{equation} \label{eq:noise-contamination}
\ssy = \ff + \sigma \eta.
\end{equation}
where $\ff: \cX \to \Re$ is the latent function.
The training observations $y_i$ can then be given as evaluations of the process \eqref{eq:noise-contamination}, that is $y_i = \ssy(x_i)$, $i=1,\dots,n$.
Thus, the problem of regression is to infer the latent function $\ff$ based on evaluations of the process \eqref{eq:noise-contamination}, i.e., $(x_i,\ssy(x_i))_{i=1}^n$.
Knowing the model \eqref{eq:noise-contamination}, which states that $\ssy$ is a noisy version of $\ff$, one knows that the observed process $\ssy$ must be rougher than the latent function $\ff$, or that $\ff$ must be smoother than $\ssy$.
In other words, assuming the noise model  \eqref{eq:noise-contamination} amounts to assuming the latent function $\ff$ being smoother than the observed process $\ssy$; this is how the noise assumption plays the role of regularization.
\end{remark}

\paragraph{Noise-free case: interpolation.}

For the noise-free case, there also exists an equivalence between GP-interpolation and kernel interpolation, which we summarize in the following result.
\begin{proposition} \label{prop:equivalnce-interp}
Let $k$ be a positive definite kernel on a nonempty set $\cX$ and $(x_i,\ff(x_i))_{i=1}^n \subset \cX \times \Re$ be training data, and define $X:= (x_1,\dots,x_n) \in \cX^n$ and $\ff_X:=(\ff(x_1),\dots,\ff(x_n))\Trans \in \Re^n$.
Assume that the kernel matrix $k_{XX}$ is invertible.
Then we have $\bar{m} = \hat{f}$, where
\begin{itemize}
\item $\bar{m}$ is the posterior mean function \eqref{eq:posteior_mean_interp} of GP-interpolation based on $(X,\ff_X)$ and the GP prior $\ff \sim \GP(0,k)$;
\item $\hat{f}$ is the solution \eqref{eq:KRR_interpolation} to kernel interpolation \eqref{eq:kernel-interp-opt} based on $(X,\ff_X)$ and the RKHS $\cH_k$.
\end{itemize}
\end{proposition}

\subsection{Error Estimates: Posterior Variance and Worst-Case Error} 
\label{sec:posterior_variance}

\begin{figure}[t]
  \centering \scriptsize
	\includegraphics[width=\linewidth]{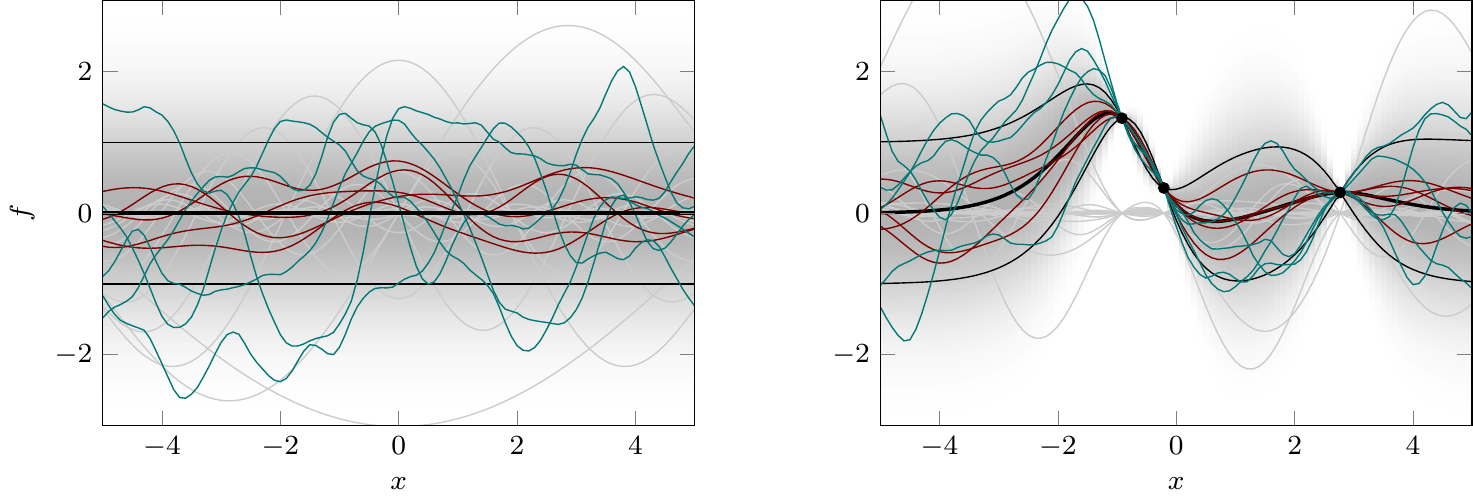}
  \caption{In-model error estimation. Plots similar to Fig.~\ref{fig:GP_intro}. \textbf{Left:} Hypothesis class/prior: The plot shows five sample paths from the GP prior in green and, for comparison, five functions with $f_X\Trans k_{XX} ^{-1} f_X=1$ in red. In light gray in the background: Eigenfunction spectrum (regular grid over the continuous space of such functions), scaled by their eigenvalues. (See Sec.~\ref{sec:Mercer} for eigen expansions of GP and RKHSs) \textbf{Right:} When constrained on noise-less observations, both Gaussian process regression and kernel ridge regression afford the same in-model error estimate, plotted as two thin black lines (Proposition \ref{prop:wce_pvar}). In the GP context, this is the error bar of one marginal standard deviation. In the kernel context, it is the worst case error if the true function has unit RKHS norm. The red functions (which approximate such unit-norm RKHS elements) lie entirely inside this region, while GP samples (green) lie inside it for $\sim 68\%$ of the path (the expected value, the Gaussian probability mass within one standard-deviation).
 Another visible feature is that the GP samples are rougher than the unit-norm representers.}
  \label{fig:RKHS_v_GP} 
\end{figure}

Gaussian process regression is usually employed in settings that also call for a notion of {\em uncertainty}, or {\em error estimate}. The object given this interpretation is the posterior covariance function $\bar{k}$ given by (\ref{eq:posterior-variance}) or its scalar value $\bar{k}(x,x)$ at a particular location $x \in \cX$; this is the \emph{(marginal) posterior variance}, the square root of which is interpreted as an ``error bar''. 
Such uncertainty estimates have numerous applications, one example being active learning, where one explores input locations where uncertainties over output values are high.

The posterior variance is, by definition, the posterior expected square difference between the posterior mean $\bar{m}(x)$ given by (\ref{eq:posteior_mean}) and the output $\ssf(x)$ of posterior GP sample $\ssf \sim \GP(\bar{m},\bar{k})$, that is,
\begin{equation} \label{eq:posterior_var_exp}
 \bar{k}(x,x) = \bE_{\ssf \sim \GP(\bar{m},\bar{k})}[(\ssf(x) - \bar{m}(x))^2].
\end{equation}
In other words, $\bar{k}(x,x)$ is interpreted as the \emph{average case error} at a location $x$ from the Bayesian viewpoint.
The purpose of this subsection is show that there exists a kernel/frequentist interpretation of $\bar{k}(x,x)$ as a certain \emph{worst case error}.
To the best of our knowledge, this fact has not been known in the literature.


For simplicity, we focus here on regression with a zero-mean GP prior $\ssf \sim \GP(0,k)$. 
We use the following notation.
Let $w^\sigma: \cX \to \Re^n$ be a vector-valued function defined by  
\begin{equation} \label{eq:weight-vec}
w^\sigma(x) := (k_{XX} + \sigma^2 I_n)^{-1}k_{Xx} \in \Re^n, \quad x \in \cX,
\end{equation}
so that the posterior mean function \eqref{eq:posteior_mean} can be written as a projection of $Y = (y_1,\dots,y_n)\Trans$ onto $w^\sigma(x)$:
\begin{equation} \label{eq:pmean-projection}
\bar{m}(x) = \sum_{i=1}^n w_i^\sigma(x) y_i =  Y\Trans w^\sigma(x).
\end{equation}
Moreover, define a new kernel $k^\sigma : \cX \times \cX \to \Re$ by
\begin{equation} \label{eq:reg-kernel}
k^\sigma(x,y) := k(x,y) + \sigma^2 \delta(x,y), \quad x,y \in \cX,
\end{equation}
where $\delta:\cX \times \cX \to \Re$ is the kernel defined in \eqref{eq:Kronecker-delta-kernel}.
Then, by definition, the kernel matrix of $k^\sigma$ for $X=(x_1,\dots,x_n)$ is given by $k_{XX}^\sigma = k_{XX} + \sigma^2 I_n$.
Note that \eqref{eq:reg-kernel} is understood as the covariance kernel of the contaminated observation process \eqref{eq:noise-contamination}, where $\ff := \ssf \sim \GP(0,k)$.
Let $\cH_{k^\sigma}$ be the RKHS of $k^\sigma$.

Given these preliminaries, we now present our result on the worst case error interpretation of the posterior variance \eqref{eq:posterior_var_exp}.
\begin{proposition} \label{prop:post-var-worst-noisy}
Let $\bar{k}$ be the posterior covariance function \eqref{eq:posterior-variance} with noise variance $\sigma^2$.
Then, for any $x \in \cX$ with $x \neq x_i$, $i=1,\dots,n$, we have
\begin{equation}
\sqrt{ \bar{k}(x,x) + \sigma^2 }
= \sup_{  g \in \cH_{k^\sigma}: \| g \|_{\cH_{k^\sigma}} \leq 1 } \left( g(x) - \sum_{i=1}^n w_i^\sigma(x) g(x_i) \right)  \label{eq:cov-equi-noise}.
\end{equation}
\end{proposition}

To prove the above proposition, we need the following lemma, which is useful in general.
The proof can be found in Appendix \ref{sec:proof-norm-worst}
\begin{lemma} \label{lemma:RKHS-norm-worst}
Let $k$ be a kernel on $\cX$ and $\cH_k$ be its RKHS.
Then for any $m \in \mathbb{N}$, $x_1,\dots,x_m \in \cX$ and $c_1,\dots,c_m \in \Re$, we have
\begin{equation} \label{eq:RKHS-norm-worst}
\left\| \sum_{i=1}^m c_i k(\cdot,x_i) \right\|_{\cH_k} = \sup_{f \in \cH_k: \| f \|_{\cH_k} \leq 1} \sum_{i=1}^m c_i f(x_i).
\end{equation}
\end{lemma}

Lemma \ref{lemma:RKHS-norm-worst} shows that the RKHS norm of a linear combination of feature vectors can be written as a supremum over functions in the unit ball of the RKHS.
Based on this result, Proposition \ref{prop:post-var-worst-noisy} can be proven as follows.
\begin{proof}
By Lemma \ref{lemma:RKHS-norm-worst}, we have
\begin{equation}
\left\| k^\sigma(\cdot,x) - \sum_{i=1}^n w_i^\sigma(x) k^\sigma(\cdot,x_i) \right\|_{\cH_{k^\sigma}} = \sup_{\stackrel{g \in \cH_{k^\sigma} :} {\| g \|_{\cH_{k^\sigma}} \leq 1 }} \left( g(x) - \sum_{i=1}^n w_i^\sigma(x) g(x_i) \right)  \label{eq:cov-equi-noise2}.
\end{equation}
The left side of this equality can be expanded as
\begin{eqnarray*}
&& \left\| k^\sigma(\cdot,x) - \sum_{i=1}^n w_i^\sigma(x) k^\sigma(\cdot,x_i) \right\|_{\cH_{k^\sigma}}^2 \\
&=& k^\sigma(x,x) - 2 \sum_{i=1}^n w_i^\sigma(x) k^\sigma(x,x_i) + \sum_{i,j=1}^n w_i^\sigma(x) w_j^\sigma(x) k^\sigma(x_i,x_j) \\
&=& k(x,x) + \sigma^2 - 2 \sum_{i=1}^n w_i^\sigma(x) k(x,x_i) + w^\sigma(x) \Trans k^\sigma_{XX} w^\sigma(x) \\
&=& k(x,x) + \sigma^2 - 2 w^\sigma(x)\Trans k_{Xx} +  k_{xX} (k_{XX} + \sigma^2 I_n)^{-1} (k_{XX} + \sigma^2 I_n) (k_{XX} + \sigma^2 I_n)^{-1} k_{Xx} \\
&=&  k(x,x) + \sigma^2 - 2 k_{xX} (k_{XX} + \sigma^2 I_n)^{-1} k_{Xx} + k_{xX} (k_{XX} + \sigma^2 I_n)^{-1} k_{Xx} \\
&=& k(x,x) + \sigma^2 -  k_{xX} (k_{XX} + \sigma^2 I_n)^{-1} k_{Xx}\\
&=& \bar{k}(x,x) + \sigma^2,
\end{eqnarray*}
where we used the assumption $x \not= x_i$ for $i=1,\dots,n$ in the second equality.
The assertion follows from this last expression and \eqref{eq:cov-equi-noise2}.
\end{proof}

\begin{remark}\rm
To understand Proposition \ref{prop:post-var-worst-noisy}, we need to understand the structure of the RKHS $\cH_{k^\sigma}$.
First, the RKHS of $\sigma^2 \delta$, denoted by $\cH_{\sigma^2 \delta}$, is characterized as
\begin{equation} \label{eq:RKHS-delta-kernel}
\cH_{\sigma^2 \delta} = \left\{ h = \sum_{x \in \cX} c_x \sigma^2 \delta(\cdot,x) :\quad \| h \|_{\cH_{\sigma^2 \delta}}^2 = \sigma^4 \sum_{x \in \cX}  c_x^2 < \infty, \quad c_x \in \Re,\ x \in \cX \right\},
\end{equation}
where the summation $\sum_{x \in \cX}$ can be uncountable.
It is known that $\cH_{\sigma^2 \delta}$ is not separable; see e.g.~\citet[Example 3.9]{SteSco12}.
Note that the kernel $\sigma^2 \delta(x,y)$ is given as a multiplication of $\sigma^2$ to the kernel $\delta(x,y)$, so $\cH_{\sigma^2 \delta}$ is norm-equivalent to the RKHS $\cH_{\delta}$ of $\delta$.
Moreover, from \eqref{eq:RKHS-delta-kernel}, it is easy to see that for any $h \in \cH_{\delta}$, we have $\| h \|_{\cH_{\sigma^2 \delta}} = \| h \|_{\cH_{\delta}}$ for all $\sigma > 0$.

Since $k^\sigma$ is defined as the sum of two kernels $k$ and $\sigma^2 \delta$, the RKHS $\cH_{k^\sigma}$ consists of functions that can be written as a sum of functions from $\cH_k$ and $\cH_{\sigma^2 \delta}$ \citep[Section 6]{Aronszajn1950}:
\begin{equation} \label{eq:RKHS-new-kernel}
\cH_{k^\sigma} = \left\{g = f + h:\ \ f \in \cH_k,\ h \in \cH_{\sigma^2 \delta} \right\} = \left\{g = f + h:\ \ f \in \cH_k,\ h \in \cH_{\delta} \right\},
\end{equation}
where the corresponding RKHS norm is given by
$$
\| g \|_{\cH_{k^\sigma}} = \inf_{\substack{f \in \cH_k,\ h \in \cH_{\sigma^2 \delta} \\ g = f + h}} \| f \|_{\cH_k} + \| h \|_{\cH_{\sigma^2 \delta}} = \inf_{\substack{f \in \cH_k,\ h \in \cH_{\delta} \\ g = f + h}} \| f \|_{\cH_k} + \| h \|_{\cH_{\delta}}.
$$
\end{remark}

\begin{remark}\rm
In \eqref{eq:cov-equi-noise}, the quantity $\sum_{i=1}^n w_i^\sigma(x) g(x_i)$ for a fixed $g \in \cH_{k^\sigma}$ with $\| g \|_{\cH_{k^\sigma}} \leq 1$ is the posterior mean function given training data $(x_i, g(x_i))_{i=1}^n$.
Note that by \eqref{eq:RKHS-new-kernel}, $g$ can be written as $g = f + h$ for some $f \in \cH_k$ and $h \in \cH_{\sigma^2 \delta}$ such that $\| f \|_{\cH_k} + \| h \|_{\cH_{\sigma^2 \delta}} \leq 1$.
Therefore each training output $g(x_i)$ can be written as $g(x_i) = f(x_i) + h(x_i)$, where $h(x_i)$ is understood as ``independent noise.''
The supremum in \eqref{eq:cov-equi-noise} is thus the worst case error between the ``true'' value $g(x) = f(x) + h(x)$ and the estimate $\sum_{i=1}^n w_i^\sigma(x) g(x_i) = \sum_{i=1}^n w_i^\sigma(x) (f(x_i) + h(x_i)) $.
Note that since $h(x)$ is ``independent noise,'' it is not possible to estimate this value from the training data, and thus a certain amount of error is unavoidable.
This may intuitively explain why the additional $\sigma^2$ term appears in the left side of \eqref{eq:cov-equi-noise}, which shows that the worst case error is more pessimistic than the average case error, in the presence of noise.
\end{remark}

\begin{remark}\rm
From the proof of Lemma \ref{lemma:RKHS-norm-worst}, it is easy to see that the function $g \in \cH_{k^\sigma}$ that attains the supremum in \eqref{eq:cov-equi-noise} is 
\begin{eqnarray*}
g &=& C \left( k^\sigma(\cdot,x) - \sum_{i=1}^n w_i^\sigma(x) k^\sigma(\cdot,x_i) \right) \\
&=& C \left( k(\cdot,x) - \sum_{i=1}^n w_i^\sigma(x) k(\cdot,x_i) \right) + C \sigma^2 \left( \delta(\cdot,x) - \sum_{i=1}^n w_i^\sigma(x) \delta(\cdot,x_i)  \right).
\end{eqnarray*}
where $C := \left\| k^\sigma(\cdot,x) - \sum_{i=1}^n w_i^\sigma(x) k^\sigma(\cdot,x_i) \right\|_{\cH_{k^\sigma}}^{-1}$ is a normalizing constant that ensures that $\| f \|_{\cH_{k^\sigma}} = 1$.
Thus, the worst adversarial function can be written as $g = f + h$, where 
$$
f := C \left( k(\cdot,x) - \sum_{i=1}^n w_i^\sigma(x) k(\cdot,x_i) \right) \in \cH_k\,\,\,\text{and}\,\,\, h:= C \sigma^2 \left( \delta(\cdot,x) - \sum_{i=1}^n w_i^\sigma(x) \delta(\cdot,x_i)  \right) \in \cH_{\sigma^2 \delta}.
$$
This implies that, as the noise variance $\sigma^2$ increases, the relative contribution of the noise term $h$ to the adversarial function $g$ increases; this makes it more difficult to fit the adversarial function, and thus the worst case error becomes more pessimistic than the average case error, as shown in \eqref{eq:cov-equi-noise}.
\end{remark}

\paragraph{Noise-free case.}
We consider the important special case of noise-free observations, that is the case where $\sigma^2 = 0$, to further illustrate Proposition \ref{prop:post-var-worst-noisy}.
In this case, the posterior standard deviation $\sqrt{k(x,x)}$, or (the square-root of) the average case error, is identical to the worst case error; see Fig.~\ref{fig:RKHS_v_GP} for illustration.
The following result, which does not require $x \not= x_i$ for $i = 1,\dots,n$ as opposed to Proposition \ref{prop:post-var-worst-noisy}, can be proven in a similar way to that of Proposition \ref{prop:post-var-worst-noisy}. 
\begin{proposition} \label{prop:wce_pvar}
Assume that $\sigma^2 = 0$, and that the kernel matrix $k_{XX}$ is invertible.
Then we have
\begin{equation} \label{eq:noiseless-identity}
  \sqrt{\bar{k}(x,x)} = \sup_{\stackrel{f\in\cH_k}{\|f\|_{\cH_k} \leq 1}} \left(\sum_{i=1}^n w_i(x) f(x_i) - f(x)\right), \quad x \in \cX,
\end{equation}
where $(w_1(x),\dots,w_n(x))\Trans = k_{XX}^{-1}k_{Xx} \in \Re^n$ and $\bar{k}$ is given by \eqref{eq:posterior-variance} with $\sigma^2 = 0$.
\end{proposition}

By applying the Cauchy-Schwartz inequality to \eqref{eq:noiseless-identity}, we have the following corollary.
It shows that the posterior variance $\bar{k}(x,x)$ provides an upper-bound on the error of kernel-based interpolation for a fixed target function.
\begin{corollary} \label{coro:upper-bound-mean-func}
Assume that  $\sigma^2 = 0$, and that the kernel matrix $k_{XX}$ is invertible. 
Then for all $f \in \cH_k$, we have
\begin{equation*}
 ( \bar{m}(x) - f(x) )^2 \leq \| f \|_{\cH_k}^2 \bar{k}(x,x), \quad x \in \cX.
\end{equation*}
where $\bar{m}(x) = \sum_{i=1}^n w_i(x)  f(x_i)$.
\end{corollary}

\subsection{A Weight Vector Viewpoint of Regularization and the Additive Noise Assumption}
\label{sec:weight-vec-KRR}

Based on the weight vector \eqref{eq:weight-vec} and the worst case error in the right side of \eqref{eq:noiseless-identity}, we provide another interpretation of the equivalence between regularization and the additive noise assumption in regression. This is given by the following result.
\begin{proposition} \label{prop:weight-interpret}
Let $X = (x_1,\dots,x_n) \in \cX^n$ be fixed, and let $Y = (y_1,\dots,y_n)\Trans \in \Re^n$ with $y_i = \ff(x_i) + \xi_i$, $i=1,\dots,n$, where $\ff:\cX \to \Re$ is a fixed function and $\xi_1,\dots,\xi_n$ are random variables such that $\bE[\xi_i] = 0$ and $\bE[\xi_i \xi_j] = \sigma^2 \delta_{ij}$ for $\sigma > 0$. Let $w^\sigma: \cX \to \Re^n$ be the vector-valued function as defined in \eqref{eq:weight-vec} with  $\sigma$, $X$ and kernel $k:\cX \times \cX \to \Re$.
Then we have 
\begin{eqnarray} 
w^\sigma(x) &=& \argmin_{w \in \Re^n} \left[ \sup_{\| f \|_{\cH_k} \leq 1} \left( f(x) - f_X\Trans w \right) \right]^2 + \sigma^2 \| w \|^2 \label{eq:weight-worst-norm} \\
&=& \argmin_{w \in \Re^n} \left[ \sup_{\| f \|_{\cH_k} \leq 1} \left( f(x) - f_X\Trans w \right) \right]^2 + \var_{\xi_1,\dots,\xi_n} [ Y\Trans w ], \quad x \in \cX, \label{eq:weight-worst-var}
\end{eqnarray}
where $f_X := (f(x_1),\dots,f(x_n))\Trans \in \Re^n$.
\end{proposition}

\begin{proof}
First, by the reproducing property it is easy to show that 
\begin{equation} \label{eq:weight-vec-RKHS}
w^\sigma(x) = \argmin_{w \in \Re^n} \left\| k(\cdot,x) - \sum_{i=1}^n w_i k(\cdot,x_i) \right\|_{\cH_k}^2  + \sigma^2 \| w \|^2, \quad x \in \cX.
\end{equation}
For the first term in the right side, we have by Lemma \ref{lemma:RKHS-norm-worst}, 
$$
\left\| k(\cdot,x) - \sum_{i=1}^n w_i k(\cdot,x_i) \right\|_{\cH_k} = \sup_{\| f \|_{\cH_k} \leq 1} \left( f(x) - f_X\Trans w \right).
$$
On the other hand, for the second term in the right side of \eqref{eq:weight-vec-RKHS}, we have
\begin{eqnarray} \label{eq:square-weight-variance}
\var_{\xi_1,\dots,\xi_n} [ Y\Trans w ] = \bE_{\xi_1,\dots,\xi_n} [ ( \xi\Trans w  )^2 ] = \bE_{\xi_1,\dots,\xi_n} [ w\Trans \xi \xi\Trans w ] = \sigma^2 \| w \|^2,
\end{eqnarray}
where $\xi := (\xi_1,\dots,\xi_n)\Trans \in \Re^n$.
The assertion follows by inserting these identities in \eqref{eq:weight-vec-RKHS}. 
\end{proof}


\begin{remark}\rm
Note that in \eqref{eq:weight-worst-var} and \eqref{eq:square-weight-variance}, the noise variables $\xi_1,\dots,\xi_n$ are independent of the function values $\ff(x_1),\dots,\ff(x_n)$, because of our assumption in Proposition \ref{prop:weight-interpret}; note also that the function $\ff$ is fixed, and is not assumed to be a Gaussian process.
\end{remark}

\begin{remark}\rm
To discuss Proposition \ref{prop:weight-interpret}, let us fix a weight vector $w \in \Re^n$.
Then the first term in the right side of \eqref{eq:weight-worst-norm} is the worst case error in the noise-free setting, since $f_X\Trans w = \sum_{i=1}^n w_i f(x_i)$ can be considered as an estimator of $f(x)$ based on noise-free observations $f(x_1),\dots,f(x_n)$, where $f$ is taken from the unit ball in the RKHS; recall also \eqref{eq:pmean-projection}.
On the other hand, the second term in \eqref{eq:weight-worst-norm} is a regularizer that makes the squared Euclidian norm of the weight vector $w$ not too large.
Importantly, it shows that the noise variance $\sigma^2$ serves as a regularization constant.

Regarding the second term in \eqref{eq:weight-worst-var}, $Y\Trans w$ is an estimator of $\ff(x)$, where $\ff$ is the (fixed) latent regression function; see again \eqref{eq:pmean-projection}.
Thus $\var_{\xi_1,\dots,\xi_n} [ Y\Trans w ]$ is the variance of the regression estimator, which is equal to the regularization term $\sigma^2 \| w \|^2$ in  \eqref{eq:weight-worst-norm}.
Therefore, \eqref{eq:weight-worst-var} shows that the weight vector \eqref{eq:weight-vec} is obtained so as to minimize the sum of the noise-free worst case error and the variance of the regression estimator based on noisy observations.


\end{remark}

\section{Hypothesis Spaces: Do Gaussian Process Draws Lie in an RKHS?}
\label{sec:theory}

In discussions about the similarity between GPs and kernel methods, it is often pointed out that the hypothesis space of Gaussian processes is not equal to that of kernel ridge regression (i.e., the corresponding RKHS). 
For instance, \citet[Section 7]{Nea98} discussed this topic, arguing why GP models had not been widely used at the time of his writing:
\begin{quotation}\it 
I speculate that a more fundamental reason for the neglect of Gaussian process models is a widespread preference for simple models, resulting from a confusion between prior beliefs regarding the true function being modeled and expectations regarding the properties of the best predictor for this function (the posterior mean, under squared error loss). These need not be at all similar. For example, our beliefs about the true function might sometimes be captured by an Ornstein-Uhlenbeck process, a Gaussian process with covariance function $\exp( - | x^{(i)} - x^{(j)} |)$. 
Realizations from this process are nowhere differentiable, but the predictive mean function will consist of pieces that are sums of exponentials, as can be seen from equation (3). 
\end{quotation}
As explained in Section \ref{sec:GP-KRR-equivalence}, the posterior mean function of GP-regression (which is the ``predictive mean function'' in the above quotation) lies in the RKHS of the GP covariance kernel.
On the other hand, if one considers a sample path of the GP prior (the Ornstein-Uhlenbeck process in the quotation, which is the GP of the Mat\'ern kernel with $\alpha = 1/2$ and $d = 1$; see also Example \ref{ex:matern-spde}), it is almost surely less smooth than functions in the RKHS (which is norm-equivalent to the first-order Sobolev space; see Example \ref{ex:matern-rkhs}.) and hence does belong to that RKHS almost surely.
This is the difference \citet{Nea98} mentioned.

The purpose of this section is to explain why the above mentioned difference exists, by reviewing sample path properties of GPs and how they are related to RKHSs.
To this end, in Section \ref{sec:hypothesis-mercer} we first look at characterizations of GPs and RKHSs by orthonormal expansions, namely the {\em Karhunen-Loève expansion} for GPs and {\em Mercer representation} for RKHSs.
We then review Driscoll's theorem \citep{driscoll1973reproducing,LukBed01}, which provides a necessary and sufficient condition for a GP sample path to lie in a {\em given} RKHS (which can be different from the RKHS associated with the GP covariance kernel) in Section \ref{sec:driscol-sample-path}.
Using this result, we show in Section \ref{sec:power-RKHS-sample-path} that GP sample spaces can be constructed as {\em powers of RKHSs} defined from the Mercer representation \citep{SteSco12}; this recovers a special case of recent generic results by \citet{Ste17} on sample path properties.
We conclude this section by using these results to derive GP sample path properties for square-exponential kernels and Mat\'ern kernels in Section \ref{sec:examples-sample-path}, the latter providing a theoretical explanation of the above difference mentioned by \citet{Nea98}.

The main message of this section may be summarized as follows: While GP sample paths fall outside of the RKHS of the GP covariance kernel almost surely, they actually lie on certain RKHSs defined as powers of that RKHS; therefore GPs and RKHSs are still deeply connected in terms of the induced hypothesis spaces, and the difference such as the one mentioned by \citet{Nea98} should not warrant strong conceptual separation between the two frameworks.
We will also use the sample path properties in this section to discuss the equivalence between convergence properties of GP and kernel ridge regression in Section \ref{sec:rates-and-posterior-contraction}.

\subsection{Characterizations via Orthonormal Expansions} \label{sec:hypothesis-mercer}
To gain intuition and understanding on the structure of RKHS and GP, we review here their expressions via orthonormal functions, that is, {\em Karhunen-Loève expansion} for GPs and {\em Mercer representation} for RKHSs.
These expressions are given in terms of the eigenvalues and eigenfunctions of a kernel integral operator defined below.
For simplicity, we assume here that $\cX$ is a compact metric space (e.g., a bounded and closed subset of $\Re^d$), and $k$ is a continuous kernel on $\cX$.

\subsubsection{Mercer's Theorem}
\label{sec:Mercer}

Let $\nu$ be a finite Borel measure on $\cX$ with $\cX$ being its support (e.g., the Lebesgue measure on $\cX \subset \Re^d$). 
Let $L_2(\nu)$ be the Hilbert space of square-integrable functions\footnote{Strictly, here each $f \in L_2(\nu)$ represents the class of functions that are equivalent $\nu$-almost everywhere.} with respect to $\nu$, as defined in (\ref{eq:lp-space}) with $p = 2$.
Define an operator $T_k: L_2(\nu) \to L_2(\nu)$ as the integral operator with the kernel $k$ and the measure $\nu$:
\begin{equation} \label{eq:integral_operator}
T_k f := \int k(\cdot,x) f(x) d \nu(x), \quad f \in L_2(\nu).
\end{equation}
If the kernel $k$ is defined on $\cX \subset \Re^d$ and shift-invariant (that is, it can be written in the form $k(x,y) = \phi(x-y)$ for some positive definite function $\phi$), then this operator is a convolution of $k$ and a function $f$.\footnote{Or more precisely, a convolution of the positive definite function $\phi$ and a measure $d\eta := fd\nu$ that has $f$ as a Radon-Nikodym derivative w.r.t.~$\nu$}
Therefore the output function $T_k f$ can be seen as a smoothed version of $f$, if $k$ is smooth.

Since $T_k$ is compact, positive and self-adjoint, the spectral theorem (see e.g.~\citealt[Theorem A.5.13]{Steinwart2008}) guarantees an eigen-decomposition of $T_k$ in the form
\begin{equation} \label{eq:integral-eigen-decomp}
T_k f = \sum_{i \in I} \lambda_i \left< \phi_i, f \right>_{L_2(\nu)} \phi_i,
\end{equation}
where  the convergence is in $L_2(\nu)$.
Here $I \subset \mathbb{N}$ is a set of indices (e.g., $I = \mathbb{N}$ when the RKHS is infinite dimensional, and $I=\{1,...,K\}$ with $K \in \mathbb{N}$ when the RKHS if $K$-dimensional), and $(\phi_i,\lambda_i)_{i \in I} \subset L_2(\nu) \times (0,\infty)$ are (countable) eigenfunctions and the associated eigenvalues of $T_k$ such that $\lambda_1 \geq \lambda_2 \geq \cdots > 0$:
\begin{equation*}
  T_k  \phi_i  = \lambda_i \phi_i, \quad i \in I.
\end{equation*}
The eigenfunctions $(\phi_i)_{i\in\mathbb{N}}$ form an orthonormal system in $L_2(\nu)$, i.e., $\left< \phi_i, \phi_j\right>_{L_2(\nu)} = \delta_{ij}$, where $\delta_{ij} = 1$ if $i = j$ and $\delta_{ij} = 0$ otherwise.

Mercer's theorem, which is named after \citet{Mercer1909}, states the kernel $k$ can be expressed in terms of the eigensystem $(\phi_i,\lambda_i)_{i \in I}$ in \eqref{eq:integral-eigen-decomp}.
This expression of the kernel provides useful ways of constructing GPs and RKHSs, as described shortly.
The following form of Mercer's theorem is due to \citet[Theorem 4.49]{Steinwart2008}, while we note that Mercer's theorem holds under weaker assumptions than those considered here \citep[Section 3]{SteSco12}. 

\begin{theorem}[Mercer's theorem] \label{theo:mercer}
Let $\cX$ be a compact metric space, $k:\cX \times \cX \to \Re$ be a continuous kernel, $\nu$ be a finite Borel measure whose support is $\cX$, and $(\phi_i,\lambda_i)_{i \in I}$ be as in \eqref{eq:integral-eigen-decomp}. 
Then we have 
\begin{equation}  \label{eq:mercer}
  k(x,x') =  \sum_{i \in I} \lambda_i \phi_i(x) \phi_i(x'), \quad x,x' \in \cX,
  \end{equation}
  where the convergence is absolute and uniform over $x, x' \in \cX$.
\end{theorem}

\begin{remark}\rm \label{rem:mercer-uniequness}
The expansion in \eqref{eq:mercer} depends on the measure $\nu$, since $(\phi_i,\lambda_i)_{i \in I}$ is an eigensystem of the integral operator \eqref{eq:integral_operator}, which is defined with $\nu$.
However, the kernel $k$ in the left side is unique, irrespective of the choice of $\nu$.
In other words, a different choice of $\nu$ results in a different eigensystem $(\phi_i,\lambda_i)_{i \in I}$, and thus results in a different basis expression of the same kernel $k$.
\end{remark}

\begin{remark}\rm \label{rem:mercer-measure}
In Theorem \ref{theo:mercer}, the assumption that $\nu$ has $\cX$ as its support is important, since otherwise the equality  \eqref{eq:mercer} may not hold for some $x \in \cX$.
For instance, assume that there is an open set $N \subset \cX$ such that $\nu(N) = 0$.
Then the integral operator \eqref{eq:integral_operator} does not take into account the values of a function $f$ on $N$, and therefore the eigenfunctions $\phi_i$ are only uniquely defined on $\cX \backslash N$, in which case, the equality in \eqref{eq:mercer} holds only on $\cX \backslash N$.
We refer to \citet[Corollaries 3.2 and 3.5]{SteSco12} for precise statements of Mercer's theorem in such a case.
\end{remark}

\subsubsection{Mercer Representation of RKHSs}

The eigensystem of the integral operator \eqref{eq:integral_operator} provides a series representation of the RKHS, which is called the Mercer representation \citep[Theorem 4.51]{Steinwart2008}.
\begin{theorem}[Mercer Representation] \label{theo:mercer-RKHS}
Let $\cX$ be a compact metric space, $k:\cX \times \cX \to \Re$ be a continuous kernel, $\nu$ be a finite Borel measure whose support is $\cX$, and $(\phi_i,\lambda_i)_{i \in I}$ be as in \eqref{eq:integral-eigen-decomp}. 
Then the RKHS $\cH_k$ of $k$ is given by 
\begin{equation}
  \label{eq:RKHS-by-EFs}
 \cH_k=\left\{f := \sum_{i \in I} \alpha_i \lambda_i^{1/2} \phi_i\,:\ \| f \|_{\cH_k}^2 := \sum_{i \in I} \alpha_i^2 < \infty \right\},
\end{equation}
and the inner-product is given by
\begin{equation*}
  \left< f,g \right>_{\cH_k} = \sum_{i \in I} \alpha_i\beta_i \q\text{for}
  \q f := \sum_{i \in I} \alpha_i \lambda_i^{1/2} \phi_i \in \cH_k,\q g := \sum_{i\in I} \beta_i \lambda_i^{1/2} \phi_i \in \cH_k. 
\end{equation*}
In other words, $(\lambda_i^{1/2} \phi)_{i \in I}$ forms an orthonormal basis of $\cH_k$.
\end{theorem}

\begin{remark}\rm
As mentioned in Remark \ref{rem:mercer-uniequness} for the expansion of a kernel \eqref{eq:mercer},
the Mercer representation in \eqref{eq:RKHS-by-EFs} depends on the measure $\nu$, since $(\phi_i,\lambda_i)_{i \in I}$ depends on $\nu$.
Under the assumptions in Theorem \ref{theo:mercer-RKHS}, however, a different choice of $\nu$, which results in a different eigensystem $(\phi_i,\lambda_i)_{i \in I}$, results in the same RKHS $\cH_k$.
Note also here that, as mentioned in Remark \ref{rem:mercer-measure}, the assumption that $\nu$ has its support on $\cX$ is crucial.
\end{remark}

\subsubsection{Karhunen-Loève Expansion of Gaussian Processes}
\label{sec:KL-expansion}
Corresponding to the Mercer representation of RKHSs, there exists a series representation of Gaussian processes known as the \emph{Karhunen-Loève (KL) expansion}.
The KL-expansion is based on the eigensystem of the integral operator in \eqref{eq:integral_operator}, as for the Mercer representation.
This is a consequence of the canonical isometric isomorphism between an RKHS and the corresponding {\em Gaussian Hilbert space} \citep{Janson1997}.
The following result, which is well known in the literature, follows from \citet[Lemmas 3.3 and 3.7]{Ste17}; see also e.g., \citet[Sections 3.2 and 3.3]{Ad90}
and \citet[Section 2.3]{Berlinet2004}.
\begin{theorem}[Karhunen-Loève Expansion] \label{theo:KL-expansion}
Let $\cX$ be a compact metric space, $k:\cX \times \cX \to \Re$ be a continuous kernel, $\nu$ be a finite Borel measure whose support is $\cX$, and $(\phi_i,\lambda_i)_{i \in I}$ be as in \eqref{eq:integral-eigen-decomp}. 
For a Gaussian process $\ssf\sim\GP(0,k)$, define 
\begin{equation} \label{eq:KL-normal-varibles}
\fz_i := \lambda_i^{-1/2} \int \ssf(x) \phi_i(x) d\nu(x), \quad i \in I.
\end{equation}
Then the following are true:
\begin{enumerate}
\item We have
\begin{equation} \label{eq:normal-variables-KL-expansion}
\fz_i \sim \mathcal{N}(0, 1)\quad \mathrm{and} \quad \bE [\fz_i \fz_j] = \delta_{ij},\quad  i,j \in I. 
\end{equation} 
\item For all $x \in \cX$ and for all finite $J \subset I$, we have
\begin{equation} \label{eq:KL-expansion}
  \bE\left[ \left( \ssf(x) - \sum_{i \in J} \fz_i \lambda_i^{1/2} \phi_i(x) \right)^2 \right] = k(x,x) - \sum_{j \in J} \lambda_j e_j^2(x).
\end{equation}
\item If $I = \mathbb{N}$, we have
\begin{equation} \label{eq:KL-uniform-mean-squre}
 \lim_{n \to \infty} \bE\left[ \left( \ssf(x) - \sum_{i = 1}^n \fz_i \lambda_i^{1/2} \phi_i(x) \right)^2 \right] = 0, \quad x \in \cX,
\end{equation}
where the convergence is uniform in $x \in \cX$.
\end{enumerate}
\end{theorem}
\begin{remark}\rm \label{rem:KL-expnasion-convergence}
Informally, Theorem \ref{theo:KL-expansion} shows that the Gaussian process $\ssf \sim \GP(0,k)$ can be expressed using ONB $(\lambda_i^{1/2} \phi_i)_{i \in I}$ of $\cH_k$ and standard Gaussian random variables $\fz_i \sim \mathcal{N}(0,1)$ as
\begin{equation} \label{eq:KL-exp-informal}
\ssf(x) = \sum_{i \in I} \fz_i \lambda_i^{1/2} \phi_i(x), \quad x \in \cX
\end{equation}
where the convergence is in the mean square sense and uniform over $x \in \cX$, as shown in \eqref{eq:KL-uniform-mean-squre}.
Note that \eqref{eq:KL-uniform-mean-squre} is an immediate consequence of \eqref{eq:KL-expansion} and Mercer's theorem (Theorem \ref{theo:mercer}).
\citet[Theorem 3.5]{Ste17} shows that, under the same conditions, the convergence in \eqref{eq:KL-exp-informal} also holds in $L_2(\nu)$.
The expression \eqref{eq:KL-exp-informal} is what is often called the KL expansion.
\end{remark}
\begin{remark}\rm
\eqref{eq:normal-variables-KL-expansion} shows that $(\fz_i)_{i \in I}$ as defined in \eqref{eq:KL-normal-varibles} are standard normal, and are independent to each other.
Note that $(\fz_i)_{i \in I}$ are dependent to the given $\ssf \sim \GP(0,k)$, as can be seen from \eqref{eq:KL-normal-varibles}; otherwise \eqref{eq:KL-expansion} and \eqref{eq:KL-uniform-mean-squre} do not hold.
On the other hand, {\em given} i.i.d.~standard normal random variables $\fz_1,\dots,\fz_n \iid \mathcal{N}(0,1)$ (i.e., independent to a {\em specific} realization $\ssf \sim \GP(0,k)$), one can construct a finite dimensional Gaussian process in the form $\sum_{i=1}^n \fz_i \lambda_i^{1/2} \phi_i$ that is approximately distributed as $\GP(0,k)$; this is often called a {\em truncated} KL expansion.
\end{remark}

\subsection{Sample Path Properties and the Zero-One Law} \label{sec:driscol-sample-path}
We review Driscoll's theorem, which provides a necessary and sufficient condition for a Gaussian process $\ssf \sim \GP(0,k)$ to belong to an RKHS $\cH_r$ with kernel $r$ with probability $1$ or $0$ \cite[Theorem 3]{driscoll1973reproducing}.
Here the kernels $r$ and $k$ can be different in general, but are defined on the same space $\cX$.
Since these probabilities (i.e., $1$ or $0$) are the only options, the theorem is called Driscoll's zero-one law. (In other words, a statement like ``$\ssf \sim \GP(0,k)$ belongs to $\cH_r$ with probability $0.3$'' is false.)
We review in particular a generalization of Driscoll's theorem by \citet[Theorem 7.4]{LukBed01}, which holds under weaker assumptions than the original theorem by \citet{driscoll1973reproducing}.
Our presentation below also uses some facts pointed out by \citet{Ste17}.
To state the result of \citet{LukBed01}, we need to introduce the notion of the {\em dominance operator}, whose existence is shown by \citet[Theorem 1.1]{LukBed01}.
\begin{theorem}[Dominance operator] \label{theo:dominance}
Let $k$ and $r$ be positive definite kernels on a set $\cX$, and let $\cH_k$ and $\cH_r$ be their respective RKHSs.
Assume $\cH_k \subset \cH_r$, and let $I_{kr}: \cH_k \to \cH_r$ be the natural inclusion operator, i.e., $I_{kr}g := g$ for $g \in \cH_k$.
Then $I_{kr}$ is continuous.
Moreover, there exists a unique linear operator $L: \cH_r \to \cH_k$ such that 
\begin{equation} \label{eq:dominance-op}
\left< f, g \right>_{\cH_r} = \left< Lf, g \right>_{\cH_k}, \quad \forall f \in \cH_r,\ \forall g \in \cH_k.
\end{equation}
In particular, we have
$$
L r(\cdot,x) = k(\cdot,x), \quad \forall x \in \cX.
$$
Furthermore, $I_{kr}L: \cH_r \to \cH_r$ is bounded, positive and symmetric.
\end{theorem}
The key concept in Driscoll's theorem is the {\em nuclear dominance}, which is defined in the following way.
As explained shortly, the nuclear dominance serves as a necessary and sufficient condition in Driscoll's zero-one law.
\begin{definition}[Nuclear dominance]
Under the same notation as in Theorem \ref{theo:dominance}, assume $\cH_k \subset \cH_r$.
Then $r$ is said to dominate $k$, and the operator $L$ in Theorem \ref{theo:dominance} is called the dominance operator of $\cH_r$ over $\cH_k$.
Moreover, the dominance is called nuclear, in which case it is written as $r \gg k$, if $I_{kr}L :\cH_r \to \cH_r$ is nuclear (or of trace class), i.e., 
$$
{\rm Tr}(I_{kr}L) = \sum_{i \in I} \left<I_{kr}L \psi_i, \psi_i \right>_{\cH_r} < \infty,
$$
where $(\psi_i)_{i \in I} \subset \cH_r$ is an ONB of $\cH_r$.
\end{definition}
Before stating Driscol's theorem, we mention that the dominance operator in Theorem \ref{theo:dominance} can be written in terms of the inclusion operator $I_{kr}$.
That is, \citet[Section 2]{Ste17} pointed out that the dominance operator $L$ is identical to $I_{kr}^*$, the adjoint operator of $I_{kr}$, as summarized in the following lemma.
\begin{lemma} \label{lemma:dominance-op-inclusion}
Under the same notation as in Theorem \ref{theo:dominance}, assume $\cH_k \subset \cH_r$.
Let $L$ be the dominance operator as given in Theorem \ref{theo:dominance}.
Then we have $L = I_{kr}^*$.
\end{lemma}
\begin{proof}
Let $I_{kr}^*$ be the adjoint of $I_{kr}$.
Then we have
$$
\left< g, f \right>_{\cH_r} = \left<I_{kr}g, f \right>_{\cH_r} = \left<g, I_{kr}^* f \right>_{\cH_r},  \quad \forall f \in \cH_r,\ \forall g \in \cH_k,
$$
which is the property \eqref{eq:dominance-op} of the dominance operator.
Since the dominance operator is unique by Theorem \ref{theo:dominance}, we have $L = I_{kr}^*$.
\end{proof}
The following result shows that the nuclear dominance is equivalent to the inclusion operator $I_{kr}$ being Hilbert-Schmidt.
The result is essentially given by the proof of \citet[Lemma 7.4, equivalence of (i) and (iii)]{Ste17}.
\begin{lemma} \label{lemma:equiv-nuclear-HS}
Under the same notation as in Theorem \ref{theo:dominance}, assume $\cH_k \subset \cH_r$.
Then the following statements are equivalent:
\begin{enumerate}
\item The nuclear dominance holds: $r \gg k$, i.e.,  $I_{kr}L:\cH_r \to \cH_r$ is nuclear.
\item The inclusion operator $I_{kr}: \cH_k \to \cH_r$ is Hilbert-Schmidt.
\end{enumerate}
\end{lemma}
\begin{proof}
From Lemma \ref{lemma:dominance-op-inclusion}, we have
$$
{\rm Tr}(I_{kr}L) = {\rm Tr}(I_{kr}I_{kr}^*) := \sum_{i \in I} \left<I_{kr} I_{kr}^* \psi_i, \psi_i \right>_{\cH_r} = \sum_{i \in I}  \left<I_{kr}^* \psi_i, I_{kr}^*\psi_i \right>_{\cH_r}  =:   \| I_{kr}^* \|_{\rm HS}^2,  
$$
where $\| \cdot \|_{\rm HS}$ denotes the Hilbert-Schmidt norm and ${\rm Tr}(\cdot)$ the trace, and $(\psi_i)_{i \in I} \subset \cH_r$ is an ONB of $\cH_r$.
Since we have $\| I_{kr}^* \|_{\rm HS} = \| I_{kr} \|_{\rm HS}$ (see e.g., \citealt[p.506]{Steinwart2008}), the assertion immediately follows.
\end{proof}
Using Lemma \ref{lemma:equiv-nuclear-HS}, Theorem 7.4 of \citet{LukBed01}, which is a generalization of the zero-one law of \citet[Theorem 3]{driscoll1973reproducing}, can be stated as Theorem \ref{theo:gen-Driscol-theorem} below.
To state it, we need to introduce a definition of a stochastic process being a {\em version} of a GP \citep[Definition 3.1.9]{Bre14}.
\begin{definition}[A version of a GP] \label{def:version-gp}
Let $\ssf \sim \GP(m,k)$ be a Gaussian process with mean function $m:\cX \to \Re$ and covariance kernel $k: \cX \times \cX \to \Re$, where $\cX$ is a nonempty set.
Then a stochastic process $\tilde{\ssf}$ on $\cX$ is called {\em a version of $\ssf$}, if 
$\ssf (x) = \tilde{\ssf}(x)$ holds with probability $1$ for all $x \in \cX$.
\end{definition}
\begin{theorem}[A generalized Driscol's theorem] \label{theo:gen-Driscol-theorem}
Let $k$ and $r$ be positive definite kernels on a set $\cX$, and let $\cH_k$ and $\cH_r$ be their respective RKHSs.
Assume $\cH_k \subset \cH_r$, and let $I_{kr}: \cH_k \to \cH_r$ be the natural inclusion operator.
Let $\ssf \sim \GP(m,k)$ be a Gaussian process such that $m \in \cH_r$.
Then the following statements are true.
\begin{enumerate}
\item If $I_{kr}$ is Hilbert-Schmidt, then there is a version $\tilde{\ssf}$ of $\ssf$ such that $\tilde{\ssf} \in \cH_r$ holds with probability $1$.
\item If $I_{kr}$ is not Hilbert-Schmidt, then $\ssf \in \cH_r$ holds with probability $0$.
\end{enumerate}
\end{theorem}
\begin{remark}\rm
In \citet[Theorem 3]{driscoll1973reproducing}, it is assumed that $\cX$ is a separable metric space, $k$ is a continuous kernel on $\cX$ and $\ssf \sim \GP(m,k)$ is almost surely continuous.
Under this assumption, \citet[Theorem 3]{driscoll1973reproducing} showed that a condition equivalent to the nuclear dominance condition \citep[Proposition 4.5]{LukBed01} implies that the {\em given} Gaussian process $\ssf$ belongs to $\cH_r$ with probability $1$.
That is, in this case one does not need to consider a version $\tilde{\ssf}$ of it.
\end{remark}
\begin{remark}\rm
In \citet[Theorem 5.1]{LukBed01}, it is shown that the nuclear dominance condition (which is equivalent to $I_{kr}$ being Hilbert-Schmidt) implies that {\em any} second-order process $\ssf$ with covariance kernel $k$ (i.e., $\ssf$ does not necessarily be Gaussian) belongs to $\cH_r$ with probability $1$.
\end{remark}
\begin{remark} \rm
One way to check whether $I_{kr}$ is Hilbert-Schmidt is given by \citet[Theorem A]{GonDud93}: They provide a necessary and sufficent  for $I_{kr}$ to be Hilbert-Schmidt in terms of an integral of the metric entropy of the embedding $I_{kr}: \cH_k \to \cH_r$.
See also \citet[Corollary 5.4]{Ste17} for a similar condition.
\end{remark}
From Theorem \ref{theo:gen-Driscol-theorem}, it is easy to show that a GP sample path $\ssf \sim \GP(0,k)$ does {\em not} belong to the corresponding RKHS $\cH_k$ with probability $1$ if $\cH_k$ is infinite dimensional, as summarized in Corollary \ref{coro:GP-path-not-in-RKHS} below.
This implies that GP samples are ``rougher'', or less regular, than RKHS functions (see also Figure~\ref{fig:RKHS_v_GP}).
Note that this fact has been well known in the literature; see e.g., \citet[p.~5]{wahba1990spline} and \citet[Corollary 7.1]{LukBed01}.
\begin{corollary} \label{coro:GP-path-not-in-RKHS}
Let $k$ be a positive definite kernel on a set $\cX$ and $\cH_k$ be its RKHS, and consider $\ssf \sim \GP(m,k)$ with $m:\cX \to \Re$ satisfying $m \in \cH_k$.
Then if $\cH_k$ is infinite dimensional, then $\ssf \in \cH_k$ with probability $0$.
If $\cH_k$ is finite dimensional, then there is a version $\tilde{\ssf}$ of $\ssf$ such that $\tilde{\ssf} \in \cH_k$ with probability $1$.
\end{corollary}
\begin{proof}
Consider Theorem \ref{theo:gen-Driscol-theorem} with $r := k$, and let $I_{kk}:\cH_k \to \cH_k$ be the inclusion operator, which is the identity map.
Let $(\psi_i)_{i \in I} \subset \cH_k$ be an orthonormal basis of $\cH_k$, where $|I| = \infty$ if $\cH_k$ is infinite dimensional, and $|I| < \infty$ if $\cH_k$ is finite dimensional.
Then 
$\| I_{kr} \|_{\rm HS}^2 = \sum_{i \in I}  \| I_{kk} \psi_i  \|_{\cH_r}^2 = \sum_{i \in I}  \| \psi_i  \|_{\cH_r}^2 = \sum_{i \in I} 1$.
Thus, $\| I_{kr} \|_{\rm HS} = \infty$ if $|I| = \infty$, and $\| I_{kr} \|_{\rm HS} < \infty$ if $|I| < \infty$. 
The assertion then follows from Theorem \ref{theo:gen-Driscol-theorem}.
\end{proof}
\begin{remark}\rm
Based on the KL expansion \eqref{eq:KL-exp-informal}, \citet[p.~5]{wahba1990spline} gave an intuitive, but rather heuristic argument to show that a GP sample path does not belong to the corresponding RKHS almost surely; see also \citet[p.~66]{Berlinet2004} and \citet[Section 6.1]{RasmussenWilliams}.
The argument is as follows.
For $\ssf \sim \GP(0,k)$, consider a KL-expansion $\ssf = \sum_{i=1}^\infty \fz_i \lambda_i^{1/2} \phi_i$ with $\fz_i \sim \mathcal{N}(0,1)$, where $(\lambda_i^{1/2} \phi_i)_{i=1}^\infty$ is an ONB of the RKHS $\cH_k$, which is assumed to be infinite dimensional.
Defining $\ssf_m := \sum_{i=1}^m \fz_i \lambda_i^{1/2} \phi_i$ for $m \in \mathbb{N}$, the KL-expansion may be written as $\ssf = \lim_{m \to \infty} \ssf_m$.
Then,
$$
  \bE[\| \ssf_m \|_{\cH_k}^2] = \bE \left[ \sum_{i=1}^m \fz_i^2  \right] =  \sum_{i=1}^m \bE[\fz_i^2] = \sum_{i=1}^m 1 = m.
$$
Therefore we have $\lim_{m \to \infty}  \bE[\| \ssf_m \|_{\cH_k}^2] = \infty$. 
This {\em may} imply that $\bE[\| \ssf \|_{\cH_k}^2] = \infty$, and further that $\ssf \notin \cH_k$ with  probability $1$.
Note that, while this argument is intuitive, it is {\em not} a proof.
This is because, as shown in Theorem \ref{theo:KL-expansion}, the standard result for the convergence of the KL-expansion $\ssf = \lim_{m \to \infty} \ssf_m$ is in the mean-square sense (or in $L_2(\nu)$, as mentioned in Remark \ref{rem:KL-expnasion-convergence}).
That is, the convergence of the KL-expansion is, of course, {\em weaker} than the convergence in the RKHS norm, and therefore  $\lim_{m \to \infty}  \bE[\| \ssf_m \|_{\cH_k}^2] = \infty$ does {\em not} imply $\bE[\| \ssf \|_{\cH_k}^2] = \infty$.
This shows that the importance of carefully considering the convergence type of the KL-expansion, which was investigated and used for establishing GP-sample path properties by \citet{Ste17}. 
\end{remark}

The following example, which follows from Corollary \ref{coro:GP-path-not-in-RKHS}, recovers the well-known fact that Brownian motion is ``non-smooth'' while it is continuous.
\begin{example}
Let $\ssf$ be the standard Brownian motion on $[0,1]$, which is a Gaussian process with kernel $k(x,y) = \min(x,y)$ for $x, y \in [0,1]$.
The corresponding RKHS $\cH_k$ is a Cameron-Martin space \citep[p.~68]{AdlJon07} given by 
$$
\cH_k = \left\{ f \in L_2([0,1]): D f\ {\rm exists}\ {\rm and} \int (D f (x))^2 dx < \infty \right\},
$$
where $D f$ denotes the weak derivative of $f$; this is the first-order Sobolev space on $[0,1]$.
Corollary \ref{coro:GP-path-not-in-RKHS} implies that $\ssf$ does not belong to $\cH_k$ almost surely. 
In other words, the Brownian motion does not admit a square-integrable weak derivative.
\end{example}

\subsection{Powers of RKHSs as GP Sample Spaces} \label{sec:power-RKHS-sample-path}
Driscoll's theorem (Theorem \ref{theo:gen-Driscol-theorem}) shows a necessary and sufficient condition for a version of $\ssf \sim \GP(m,k)$ to be a member of an RKHS $\cH_r$, but it does not directly provide a way of constructing the RKHS $\cH_r$ (nor its reproducing kernel $r$) based on the given covariance kernel $k$.
This is what is done in \citet{Ste17}: $\cH_r$ can be constructed as a {\em power} of the RKHS $\cH_k$, and $r$ as the corresponding power of the kernel $k$; these are concepts introduced by \citet[Definition 4.1]{SteSco12} based on Mercer's theorem.
We review this result, showing that it can be easily derived from Theorem \ref{theo:gen-Driscol-theorem}.

For simplicity, we assume here that a set $\cX$ is a compact metric space, a measure $\nu$ is a finite Borel measure with $\cX$ being its support, and a kernel $k$ is continuous on $\cX$.
However, we note that the results of \citet{Ste17} and \cite{SteSco12} hold under much weaker assumptions (while statements of the results should be modified accordingly).
We first introduce the definition of powers of RKHSs and kernels \citep[Definition 4.1]{SteSco12}.
\begin{definition}[Powers of RKHSs and kernels] \label{def:power-RKHS}
Let $\cX$ be a compact metric space, $k$ be a continuous kernel on $\cX$ with $\cH_k$ being its RKHS, and $\nu$ be a finite Borel measure whose support is $\cX$.
Let $0 < \theta \leq 1$ be a constant, and assume that $\sum_{i \in I} \lambda_i^\theta \phi_i^2(x) < \infty$ holds for all $x \in \cX$, where $(\lambda_i, \phi_i )_{i \in I}$ is the eigensystem of the integral operator in \eqref{eq:integral_operator}.
Then the $\theta$-th power of RKHS $\cH_k$ is defined as
\begin{equation}  \label{eq:power-RKHS}
\cH_k^\theta := \left\{ f = \sum_{i \in I} a_i \lambda_i^{\theta/2} \phi_i\ :\ \sum_{i \in I} a_i^2 < \infty \right\},
\end{equation} 
where the inner-product is given by
\begin{equation*}
  \left< f,g \right>_{\cH_k^\theta} = \sum_{i \in I} \alpha_i\beta_i \q\text{for}
  \q f := \sum_{i \in I} \alpha_i \lambda_i^{\theta/2} \phi_i \in \cH_k,\q g := \sum_{i\in I} \beta_i \lambda_i^{\theta/2} \phi_i \in \cH_k. 
\end{equation*}
The $\theta$-th power of kernel $k$ is a function $k^\theta: \cX \times \cX \to \Re$ defined by 
\begin{equation} \label{eq:power_kernel}
k^\theta(x,y) := \sum_{i\in I} \lambda_i^\theta \phi_i(x) \phi_i(y), \quad x,y \in \cX.
\end{equation}
\end{definition}
\begin{remark}\rm
The space $\cH_k^\theta$ defined as in \eqref{eq:power-RKHS} is in fact an RKHS, with its reproducing kernel being the $\theta$-th power of kernel \eqref{eq:power_kernel}, and $\cH_k^\theta$ and $k^\theta$ are uniquely determined independent of the chosen ONB $(\lambda_i^{1/2} \phi_i)_{i \in I}$ \citep[Proposition 4.2]{SteSco12}.
\end{remark}
\begin{remark}\rm
The power of the RKHS \eqref{eq:power-RKHS} is an intermediate space (or more precisely, an interpolation space) between $L_2(\nu)$ and $\H_k$, and the constant $0 < \theta \leq 1$ determines how close $\H_k^\theta$ is to $\H_k$ \citep[Theorem 4.6]{SteSco12}.
For instance, if $\theta = 1$ we have $\H_k^\theta = \H_k$, and  $\H_k^\theta$ approaches $L_2(\nu)$ as $\theta \to +0$. Indeed,  $\H_k^\theta$ is nesting with respect to $\theta$: 
\begin{equation*} \label{eq:nest_pow}
 \H_k = \H_k^1 \subset \H_k^\theta \subset \H_k^{\theta'}\subset L_2(\nu), \quad {\rm for\ all} \ \ 0 < \theta' < \theta < 1.
\end{equation*}
In other words, $\H_k^\theta$ gets larger as $\theta$ decreases. 
If $\H_k$ is an RKHS consisting of smooth functions (such as Sobolev spaces), then $\H_k^\theta$ contains less smooth functions than those in $\H_k$.
\end{remark}
The following result, which follows from Theorem \ref{theo:gen-Driscol-theorem}, provides a  characterization of GP-sample spaces in terms of powers of RKHSs $\cH_k^\theta$.
It is a special case of \citet[Theorem 5.2]{Ste17}, where assumptions required for $\cX$, $k$ and $\nu$ are much weaker.
\begin{theorem} \label{theo:gp-path-power-RKHS}
Let $\cX$ be a compact metric space, $k$ be a continuous kernel on $\cX$ with $\cH_k$ being its RKHS, and $\nu$ be a finite Borel measure whose support is $\cX$.
Let $0 < \theta < 1$ be a constant, and assume that $\sum_{i \in I} \lambda_i^\theta \phi_i^2(x) < \infty$ holds for all $x \in \cX$, where $(\lambda_i, \phi_i )_{i \in I}$ is the eigensystem of the integral operator in \eqref{eq:integral_operator}.
Consider $\ssf \sim \GP(0,k)$. 
Then the following statements are equivalent.
\begin{enumerate}
\item $\sum_{i \in I} \lambda_{i}^{1-\theta} < \infty$.
\item The inclusion operator $I_{k k^\theta}: \cH_k \to \cH_k^\theta$ is Hilbert-Schmidt.
\item There exists a version $\tilde{\ssf}$ of $\ssf$ such that $\tilde{\ssf} \in \cH_k^\theta$ with probability $1$. 
\end{enumerate}
\end{theorem}
\begin{proof}
The equivalence between 2.~and 3.~follows from Theorem \ref{theo:gen-Driscol-theorem} and the fact that $\cH_k^\theta$ is an RKHS with $k^\theta$ being its kernel.
The equivalence between 1.~and 2.~follows from
$$
\| I_{k k^\theta} \|_{\rm HS}^2 
= \sum_{i \in I} \| I_{k k^\theta} \lambda_i^{1/2} \phi_i \|_{\cH_k^\theta}^2 
= \sum_{i \in I} \| \lambda_i^{(1-\theta)/2} \lambda_i^{\theta/2} \phi_i \|_{\cH_k^\theta}^2 
=  \sum_{i \in I} \lambda_i^{1-\theta},
$$
where the first equality uses the definition of the Hilbert-Schmidt norm and the fact that $(\lambda_i^{1/2}\phi_i)_{i \in I}$ is an ONB of $\cH_k$, and the third follows from $(\lambda_i^{\theta/2}\phi_i)_{i \in I}$ being an ONB of $\cH_k^\theta$.
\end{proof}

%

\begin{remark}\rm
Theorem \ref{theo:gp-path-power-RKHS} shows that the power of the RKHS $\cH_k^\theta$ contains the support of $\ssf \sim \GP(0,k)$, if the eigenvalues $(\lambda_i)_{i \in I}$ satisfy $\sum_{i=1}^\infty \lambda_{i}^{1-\theta} < \infty$ for $0 < \theta < 1$. 
Therefore, if one knows the eigensystem $(\lambda_i,\phi_i)_{i \in I}$ of the integral operator \eqref{eq:integral_operator}, one may construct the GP-sample space as $\cH_k^\theta$ with largest $0<\theta<1$ satisfying $\sum_{i=1}^\infty \lambda_{i}^{1-\theta} < \infty$.
Note that the condition $\sum_{i=1}^\infty \lambda_{i}^{1-\theta} < \infty$ is stronger for larger $\theta$, requiring that the eigenvalues should decay more rapidly (when $|I| = \infty$). Also note that functions in $\cH_k^\theta$ get smoother as $\theta$ increases. 
\end{remark}

\subsection{Examples of Sample Path Properties} \label{sec:examples-sample-path}
We provide concrete examples of GP sample path properties, as corollaries of the above results.
We first show sample path properties for GPs with square-exponential kernels in Example \ref{ex:square-exp-kernel}.
Intuitively, the result follows from Theorem \ref{theo:gp-path-power-RKHS} and that the eigenvalues for a square-exponential kernel decay exponentially fast; see Section \ref{sec:proof-gp-sample-path-se} for a complete proof.
\begin{corollary}[Sample path properties for square-exponential kernels] \label{coro:gauss-sample-path}
Let $\cX \subset \Re^d$ be a compact set with Lipschitz boundary, $\nu$ be the Lebesgue measure on $\cX$, $k_\gamma$ be a square-exponential kernel on $\cX$ with bandwidth $\gamma > 0$ with $\cH_{k_\gamma}$ being its RKHS.
Then for all $0 < \theta < 1$, the $\theta$-th power  $\cH_{k_\gamma}^\theta$ of $\cH_{k_\gamma}$ in Definition \ref{def:power-RKHS} is well-defined.
Moreover, for a given $\ssf \sim \GP(0,k_\gamma)$, there exists a version $\tilde{\ssf}$ such that $\tilde{\ssf} \in \cH_{k_\gamma}^\theta$ with probability $1$ for all $0 < \theta < 1$.
\end{corollary}
\begin{remark}\rm
Since $\cH_{k_\gamma} \subset \cH_{k_\gamma}^\theta$ for $0 < \theta < 1$ and $\cH_{k_\gamma}^\theta$ approaches $\cH_{k_\gamma}$ as $\theta \to 1$,  Corollary \ref{coro:gauss-sample-path} shows that, informally, a GP sample path associated with a square-exponential kernel lies in a space that is infinitesimally larger than $\cH_{k_\gamma}$.
Therefore in practice one should not worry too much about the fact that a GP sample path  almost surely falls outside the RKHS $\cH_{k_\gamma}$, because it nevertheless lies almost surely on the ``infinitesimally small shell'' surrounding $\cH_{k_\gamma}$.
However, note that the situation is different for Mat\'ern kernels, of which the RKHSs have only a finite degree of smoothness; the above property for square-exponential kernels follows from, intuitively, that functions in the resulting RKHSs are infinitely smooth.
\end{remark}
Corollary \ref{coro:matern-sample-path} below provides sample path properties for GPs associated with Mat\'ern  kernels.
To state it, we need to introduce the {\em interior cone condition} \cite[Definition 3.6]{Wen05}, which requires that there is no `pinch point' (i.e.~a $\prec$-shape region) on the boundary of $\cX$. 
\begin{definition}[Interior cone condition] \label{def:interior_cone}
A set $\cX \subset \Re^d$ is said to satisfy an interior cone condition if 
there exist an angle $\theta \in (0,2\pi)$ and a radius $R > 0$ such that every $x \in \cX$ is associated with a unit vector $\xi(x)$ so that the cone $C(x, \xi(x), \psi, R)$ is contained in $\Omega$,
where
$$C(x, \xi(x), \psi, R) := \{ x + a y:\ y \in \Re^d,\ \|y\| = 1,\ \left<y, \xi(x) \right> \geq \cos \psi,\ a \in [0,R] \}.$$
\end{definition}
\begin{corollary}[Sample path properties for Mat\'ern  kernels] \label{coro:matern-sample-path}
Let $\cX \subset \Re^d$ be a bounded open set such that the boundary is Lipschitz and an interior cone condition is satisfied, and $k_{\alpha,h}$ be the Mat\'ern kernel on $\cX$ in Example \ref{ex:matern-kernel} with parameters $\alpha > 0$ and $h > 0$ such that $\alpha + d/2 \in \mathbb{N}$. 
Then for a given $\ssf \sim \GP(0,k_{\alpha,h})$, there exists a version $\tilde{\ssf}$ such that $\tilde{\ssf} \in \cH_{k_{\alpha', h'}}$ with probability $1$ for all $\alpha', h' > 0$ satisfying $\alpha > \alpha'+d/2 \in \mathbb{N}$, where $\cH_{k_{\alpha', h'}}$ is the RKHS of the Mat\'ern kernel $k_{\alpha', h'}$ with parameters $\alpha'$ and $h'$.
\end{corollary}
\begin{proof} 
Let $s := \alpha + d/2$ and $\beta := \alpha' + d/2$.
Denote by $W_2^s(\cX)$ and $W_2^\beta(\cX)$ the Sobolev spaces of order $s$ and $\beta$ respectively, as defined in Example \ref{ex:matern-rkhs}.
Since $\cX$ satisfies an interior cone condition and we have $s - \beta > d/2$,  Maurin's theorem \citep[Theorem 6.61]{AdaFou03} implies that the embedding $W_2^s(\cX) \to W_2^\beta(\cX)$ is Hilbert-Schmidt.
Since the boundary of $\cX$ is Lipschitz, by \citet[Corollary 10.48]{Wen05} the RKHS $\cH_{k_{\alpha,h}}$ of $k_{\alpha,h}$ is norm-equivalent to $W_2^s(\cX)$, and $\cH_{k_{\alpha',h'}}$ is norm-equivalent to $W_2^\beta(\cX)$. (See also Example \ref{ex:matern-rkhs}.)
Therefore the embedding $\cH_{k_{\alpha,h}} \to \cH_{k_{\alpha',h'}} $ is also Hilbert-Schmidt.
The assertion then follows from Theorem \ref{theo:gen-Driscol-theorem}.
\end{proof}
\begin{remark}\rm \label{rem:matern-smoothness-path}
Recall that, as shown in Example \ref{ex:matern-rkhs}, the RKHS $\cH_{k_{\alpha,h}}$ of the Mat\'ern kernel $k_{\alpha,h}$ is norm-equivalent to the Sobolev space $W_2^s(\cX)$ of order $s := \alpha + d/2$, and $\cH_{k_{\alpha',h'}}$ is norm-equivalent to $W_2^\beta(\cX)$ with $\beta := \alpha' + d/2$.
From the assumption $\alpha > \alpha' + d/2$, we have $s  > \beta + d/2$.
Therefore, Corollary \ref{coro:matern-sample-path} shows that, roughly, the smoothness of a GP sample path with $k_{\alpha,h}$, which is $\beta$, is $d/2$-smaller than the smoothness $s$ of the RKHS $\cH_{k_{\alpha,h}}$.
\end{remark}
\begin{remark}\rm
The condition $\alpha > \alpha' + d/2$ can be shown to be sharp: If $\alpha = \alpha'+ d/2$, a sample path  $\ssf \sim \GP(0,k_{\alpha,h})$ does {\em not} belong to $\cH_{k_{\alpha', h'}}$ almost surely \citep[Corollary 5.6, ii]{Ste17}.
Also note that the condition $\alpha' + d/2 \in \mathbb{N}$ may be removed; $\alpha$ can be any positive real satisfying $\alpha > \alpha' + d/2$  \citep[Corollary 5.6, i]{Ste17}.
\end{remark}

\section{Convergence and Posterior Contraction Rates in Regression} \label{sec:rates-and-posterior-contraction}
In this section, we review asymptotic convergence results for Gaussian process and kernel ridge regression.
For both approaches, there have been extensive theoretical investigations, but it seems that the connections between the obtained results for the two approaches are rarely discussed.
We therefore discuss the connections between the convergence results for the two approaches in Section \ref{sec:convergence-GP-kernel}, and show that there is indeed a certain equivalence that highlights the role of regularization and the output noise assumption.
The key role in showing this equivalence is played by sample path properties discussed in Section \ref{sec:theory}.
We also review theoretical results from the kernel interpolation literature in Section \ref{sec:upper-bound-variance}.
Thanks to the worst case error viewpoint explained in Section \ref{sec:posterior_variance}, these results provide upper-bounds for marginal posterior variances in GP-regression.
Such bounds are useful in understanding what factors affect the speed of contraction of marginal posterior variances.

\subsection{Convergence Rates for Gaussian Process and Kernel Ridge Regression} \label{sec:convergence-GP-kernel}

We here review existing convergence results for Gaussian process regression and kernel ridge regression.
Specifically, we compare the posterior contraction rates for GP-regression provided by \citet{VarZan11} and the convergence rates for kernel ridge regression by \citet{SteHusSco09}.
The rates obtained in these papers are minimax optimal for regression in Sobolev spaces.
Thus, it would be natural to ask how these two results are related.
We show that, by focusing on regression in Sobolev spaces, the rates of \citet{VarZan11} can be recovered from those of \citet{SteHusSco09}.
This highlights the equivalence between regularization in kernel ridge regression and the additive noise assumption in Gaussian process regression. 
The arguments are based on sample path properties reviewed in Section \ref{sec:theory}.

\paragraph{Gaussian process regression.}
The following model is considered in \citet{VarZan11}.
Let $\cX = [0,1]^d$ be the input space, and $f_0:[0,1]^d \to \Re$ be the unknown regression function to be estimated.
Let $(x,y) \in [0,1]^d \times \Re$ be a joint random variable such that $x \sim P_\cX$ for a distribution $P_\cX$ on $[0,1]^d$ and
\begin{equation} \label{eq:GP-model-theory}
y = f_0(x) + \varepsilon,  \quad \varepsilon \sim \N(0,\sigma^2),
\end{equation}
where $\sigma^2 > 0$ is the variance of independent noise $\varepsilon$.
$P_\cX$ is assumed to have a density function that is bounded away from zero and infinity. 
Denote by $P$ the joint distribution of $(x,y)$, and assume that one is given an i.i.d.~sample $\mathcal{D}_n := (x_i,y_i)_{i=1}^n$ of size $n$ from $P$ as training data.
A prior for $f_0$ is a zero-mean Gaussian process $\GP(0, k_s)$ with covariance kernel $k_s$, where $k_s$ denotes the Mat\'ern kernel (Example \ref{eq:matern-kernel}) whose RKHS is norm-equivalent to the Sobolev space $W_2^s[0,1]^d$ of order $s > d/2$.
(In the notation of Example \ref{eq:matern-kernel}, this corresponds to $\alpha := s - d/2$; see Example \ref{ex:matern-rkhs} for the RKHSs of Mat\'ern kernels.)

GP-regression is performed based on this GP prior, training data $\mathcal{D}_n$ and the likelihood given by the additive Gaussian noise model \eqref{eq:GP-model-theory}.
\citet[Theorem 5]{VarZan11} provided the following posterior contraction rate for this setting, assuming the unknown $f_0$ belongs to a Sobolev space $W_2^\beta[0,1]^d$ of order $\beta > d/2$.
Below $C^\beta([0,1]^d)$ denotes the H\"older space of order $\beta$, and $L_2(P_\cX)$ the Hilbert space of square-integrable functions with respect to $P_\cX$.

\begin{theorem} \label{theo:GP-rate-van}
Let $k_s$ be a kernel on $[0,1]^d$ whose RKHS is norm-equivalent to the Sobolev space $W_2^s([0,1]^d)$ of order $s := \alpha + d/2$ with $\alpha > 0$. 
If $f_0 \in C^\beta([0,1]^d) \cap W_2^\beta([0,1]^d)$ and $\min(\alpha,\beta) > d/2$, then we have
\begin{equation} \label{eq:post-contraction}
\bE_{\mathcal{D}_n|f_0} \left[ \int \| f - f_0 \|_{L_2(P_\cX)}^2  d\Pi_n(f|\mathcal{D}_n) \right] = O(n^{- 2 \min(\alpha,\beta) / (2
\alpha + d) }) \quad (n \to \infty),
\end{equation}
where $\bE_{X,Y|f_0}$ denotes the expectation with respect to $\mathcal{D}_n = (x_i,y_i)_{i=1}^n$ with the model \eqref{eq:GP-model-theory}, and $\Pi_n(f|\mathcal{D}_n)$ the posterior distribution given by GP-regression with kernel $k_s$. 
\end{theorem}

\begin{remark}\rm
By definition, the posterior mean function \eqref{eq:posteior_mean} is given as $\bar{m}_n := \int f d\Pi_n(f|D_n)$ (here we made the dependence of $m_n$ on sample size $n$ explicit).
It is easy to show that
$$
 \| \bar{m}_n - f_0 \|_{L_2(P_\cX)}^2 \leq  \int \| f - f_0 \|_{L_2(P_\cX)}^2  d\Pi_n(f|\mathcal{D}_n), 
$$
and therefore \eqref{eq:post-contraction} implies the convergence of $\bar{m}_n$ to $f_0$,
\begin{equation} \label{eq:rate-post-mean}
\bE_{\mathcal{D}_n|f_0} \left[ \| \bar{m}_n - f_0 \|_{L_2(P_\cX)}^2 \right] = O(n^{- 2 \min(\alpha,\beta) / (2
\alpha + d) }) \quad (n \to \infty).
\end{equation}
\end{remark}

\begin{remark}\rm
The best rate with \eqref{eq:rate-post-mean} is attained when $\alpha = \beta$, which results in the rate $n^{-2\beta / (2 \beta + d)}$; this is the minimax-optimal rate for regression of a function in $W_2^\beta([0,1]^d)$ \citep{Sto80,Tsy08}.
Note that $\alpha = s - d/2$ is essentially the smoothness of $\ssf \sim \GP(0,k_s)$, a sample path of the GP prior (see Corollary \ref{coro:matern-sample-path} and Remark \ref{rem:matern-smoothness-path}).
Therefore the requirement $\alpha = \beta$ means that the smoothness $\alpha$ of sample paths from the GP prior should match the smoothness $\beta$ of the regression function $f_0$.
\end{remark}

For later comparison with kernel ridge regression, we point out here that the smoothness $s$ of the corresponding Sobolev RKHS $H^s([0,1]^d)$ should be specified as $s = \beta + d/2$ to attain the optimal rate.
In other words, the smoothness of the RKHS $H^s([0,1]^d)$ should be {\em greater} than the Sobolev space $H^\beta([0,1]^d)$ to which $f_0$ belongs.
This may be counterintuitive, given the equivalence between GP-regression and kernel ridge regression.
As we show below, however, the above consequence can be explained from the modeling assumption \eqref{eq:GP-model-theory} that noise variance $\sigma^2$ remains constant even if $n$ increases.

\paragraph{Kernel ridge regression.}
For kernel ridge regression, we discuss the convergence results of \cite{SteHusSco09}, which do not require the true regression function $f_0$ be in the RKHS.
Let $\cX$ be an arbitrary measurable space and $\cY := [-M,M] \subset \Re$ be the output space, where $M > 0$ is some constant, and $P$ be a joint distribution on $\cX \times \cY$.
The unknown regression function $f_0: \cX \to [-M,M]$ is defined as the conditional expectation $f_0(x) := \bE[y|x]$ as usual, where the expectation is with respect to the conditional distribution of $y$ given $x$ with $(x,y) \sim P$.
Let $P_\cX$ be the marginal distribution of $P$ on $\cX$, and assume that it has a density bounded away from zero and infinity; this is the same assumption as for \citet{VarZan11}.

Given a training sample $(x_1,y_1),\dots,(x_n,y_n) \iid P$, let $\hat{f}_\lambda$ be the estimator of $f_0$ by kernel ridge regression defined as the solution of \eqref{eq:square-loss}, with $\lambda > 0$ being a regularization constant (here we made the dependence on $\lambda$ explicit).
In \cite{SteHusSco09}, the following clipped version $\check{f}_\lambda$ of $\hat{f}_\lambda$ is considered for theoretical analysis:
$$
\check{f}_\lambda(x) := 
\begin{cases}
M \quad (\hat{f}_\lambda(x) > M) \\
\hat{f}_\lambda(x) \quad (-M \leq \hat{f}(x) \leq M) \\
-M \quad (\hat{f}_\lambda(x) < -M)
\end{cases}
$$
The following result follows from Corollary 6 in \citet{SteHusSco09}, the arguments in the paragraphs following Theorem 9 in \citet{SteHusSco09}, and the fact $H^\beta([0,1]^d) \subset B_{2,\infty}^\beta([0,1]^d)$, where $\beta > d/2$ and $H^\beta([0,1]^d)$ is the Sobolev space of order $\beta > d/2$ and $B_{2,\infty}^\beta([0,1]^d)$ is a certain Besov space of the same order $\beta$; see.~e.g.~\citet[Eq.7 in p.59]{EdmTri96}.

\begin{theorem} \label{theo:rate-krr-steinwart}
Let $k_s$ be a kernel on $\cX := [0,1]^d$ whose RKHS is norm-equivalent to the Sobolev space $W_2^s([0,1]^d)$ of order $s > d/2$.
Assume that $f_0 \in W_2^\beta([0,1]^d)$ for some $d/2 \leq \beta \leq s$.
If $\lambda_n > 0$ is set as
\begin{equation} \label{eq:lambda_sob_spe}
\lambda_n = c n^{-2s/(2\beta+d)},
\end{equation}
for a constant $c > 0$ independent of $n$, then we have 
\begin{equation} \label{eq:rate-KRR-sobolev}
\| \check{f}_{\lambda_n} - f_0 \|_{L_2(P_\cX)}^2 = O_p(n^{-2\beta / (2\beta + d)}) \quad (n \to \infty).
\end{equation}
\end{theorem}

\begin{remark}\rm
As mentioned earlier, the rate \eqref{eq:rate-KRR-sobolev} is minimax optimal for regression in the Sobolev space $W_2^\beta([0,1]^d)$ of order $\beta$ \citep{Sto80,Tsy08}.
\end{remark}

\begin{remark}\rm
Theorem \ref{theo:rate-krr-steinwart} does not require that $f_0$ be in the RKHS of kernel $k_s$: the smoothness $\beta$ of $f_0$ can be smaller than the smoothness $s$ of the RKHS, in which case $f_0$ does not belong to $H^s([0,1]^d)$.
Intuitively, this is possible because a function outside the RKHS but ``not very far away'' from the RKHS can be approximated well by functions in the RKHS.
This intuition is in fact formally characterized and exploited in \citet{SteHusSco09} by using approximation theory based on interpolation spaces.
Note also that the degree of approximation is controlled by the regularization schedule \eqref{eq:lambda_sob_spe}: $\lambda_n$ should decrease more quickly, as the smoothness $\beta$ of $f_0$ becomes smaller.
\end{remark}

The following corollary is a special case of Theorem \ref{theo:rate-krr-steinwart}, which is essentially equivalent to Theorem \ref{theo:GP-rate-van} for GP-regression.
\begin{corollary} \label{coro:KRR-rate}
Assume that $f_0 \in W_2^\beta([0,1]^d)$ for $\beta > 0$.
Let $k_s$ be a kernel on $\cX := [0,1]^d$ whose RKHS is norm-equivalent to the Sobolev space $W_2^s([0,1]^d)$ of order $s := \beta + d/2$, and define $\lambda_n := c n^{-1}$ with $c > 0$ being any constant.
Then \eqref{eq:rate-KRR-sobolev} holds for $\check{f}_{\lambda_n} $.
\end{corollary}

\begin{remark}\rm
Recall that the variance $\sigma^2$ of output noise in GP-regression is related to the regularization constant $\lambda$ in kernel ridge regression as $\sigma^2 = n \lambda_n$.
Therefore, the modeling assumption that $\sigma^2$ does not vary with $n$ in GP-regression is equivalent to the regularization schedule $\lambda_n = c n^{-1}$ in kernel ridge regression.
With this regularization schedule, the optimal rate \eqref{eq:rate-KRR-sobolev} is attained by kernel $k_s$ with $s = \beta + d/2$.
This optimal order $s$ of the kernel is the same as for the optimal order in Theorem \ref{theo:GP-rate-van} for GP-regression (i.e.,~$s = \alpha + d/2$ with $\alpha = \beta$).
Thus, Corollary \ref{coro:KRR-rate} is essentially equivalent to Theorem \ref{theo:GP-rate-van}, revealing a theoretical equivalence between GP and kernel ridge regression.
\end{remark}

\subsection{Upper-bounds and Contraction Rates for Posterior Variance} \label{sec:upper-bound-variance}

We focus here on the noise-free case, and review the results that provide contraction rates for posterior variance $\bar{k}(x,x)$ in GP-interpolation \eqref{eq:posterior-variance_interp}.
It seems that these results have not been well known in the machine learning literature.
However, we believe that they are important in particular in understanding the mechanism of active learning methods based on GPs and Bayesian optimization, as these methods make use of the posterior variance function in exploring new points to evaluate.
In fact, these results have essentially been used in \citet{Bul11} for theoretical analysis of Bayesian optimization; see also e.g.,~\citet{BriOatGirOsbSej15,TuoWu16,StuTec18} for similar applications of these results.

The results we review are from the literature on kernel interpolation \citep{Wen05,SchWen06,SchSchSch13}. 
In the kernel interpolation literature, the posterior standard deviation function $(\bar{k}(x,x))^{1/2}$ is called the \emph{power function}, and has been studied because it provides an upper-bound for the error of kernel interpolation, as can be seen from Corollary \ref{coro:upper-bound-mean-func}.
The key role is played by the quantity called the {\em fill distance}, which quantifies the denseness of points $X = \{ x_1,\dots,x_n \} \subset \Re^d$ in the region of interest $\cX \subset \Re^d$.
For a constant $\rho > 0$, the fill distance at $x \in \cX$ is defined by 
\begin{equation} \label{eq:local-fill-dist}
h_{\rho,X}(x) := \sup_{y \in \cX: \| x - y \| \leq \rho} \min_{x_i \in X} \| y - x_i \|. 
\end{equation}
In other words, by thinking of the ball $B(x, \rho)$ around $x$ of radius $\rho$, the fill distance $h_{\rho,X}(x)$ is the radius of the maximum ball in $B(x, \rho)$ in which no points are contained.
Thus, $h_{\rho,X}(x)$ being smaller implies that more points are located around $x$.

The following result, which is from \citet[Theorem 5.14]{WuSch93}, provides an upper-bound for the posterior variance in terms of the fill distance \eqref{eq:local-fill-dist}. In particular it applies to cases where the kernel induces an RKHS that is norm-equivalent to Sobolev spaces (e.g., Mat\'ern kernels).
\begin{theorem} \label{theo:upper-bound-post-var}
Let $k$ be a kernel on $\Re^d$ whose RKHS is norm-equivalent to the Sobolev space of order $s$.
Then for any $\rho > 0$, there exist constants $h_0 > 0$ and $C >0$ satisfying the following: For any $x \in \Re^d$ and any set of points $X = \{x_1,\dots,x_n\} \subset \Re^d$ satisfying $h_{\rho,X}(x) \leq h_0$, we have
\begin{equation} \label{eq:bound-post-var}
\bar{k}(x,x) \leq C h_{\rho,X}^{2s - d}(x).
\end{equation}
\end{theorem}

\begin{remark}\rm
For a kernel with infinite smoothness (such as square-exponential kernels), the exponent in the upper-bound \eqref{eq:bound-post-var} can be arbitrarily large. 
That is, for $\rho > 0$ and any $\alpha > 0$, there exist constants $h_\alpha > 0$ and $C_\alpha >0$ satisfying the following: For any $x \in \Re^d$ and any set of points $X = \{x_1,\dots,x_n\} \subset \Re^d$ satisfying $h_{\rho,X}(x) \leq h_\alpha$, we have
\begin{equation*}
\bar{k}(x,x) \leq C_\alpha h_{\rho,X}^{\alpha}(x).
\end{equation*}
Note that constants $C_\alpha$ and $h_\alpha$ depend on $\alpha$, so the upper-bound may not monotonically decrease as $\alpha$ increases, for fixed points $X$.
\end{remark}

\begin{remark}\rm
An upper-bound of exponential order can be derived for a kernel with infinite smoothness, but this technically requires that the fill distance be defined globally on the region of interest $\cX \subset \Re^d$: For $X = \{ x_1,\dots,x_n \} \subset \cX$, the (global) fill distance $h_X > 0$ is defined as 
$$
h_X := \sup_{x \in \cX} \min_{x_i \in X} \| x - x_i \|.
$$
For instance, if $\cX$ is a cube in $\Re^d$ and $k$ is a square-exponential kernel on $\cX$, \citet[Theorem 11.22]{Wen05} shows that there exists a constant $c > 0$ that does not depend on $h_X$ such that
$$
\bar{k}(x,x) \leq  \exp( c \log(h_{X,\Omega}) / \sqrt{h_X} ).
$$
whenever $h_X$ is sufficiently small.
\end{remark}

Theorem \ref{theo:upper-bound-post-var} shows that the amount of posterior variance $\bar{k}(x,x)$ is determined by i) the local fill distance $h_{\rho,X}$, ii) the dimensionality $d$ of the space, and iii) the smoothness $s$ of the kernel.
This implies that  $\bar{k}(x,x)$ contracts to $0$ as the denseness of the points $x_1,\dots,x_n$ around $x$ increases, and that the rate of contraction is determined by $d$ and $s$.
This is formally characterized by the following corollary, which directly follows from Theorem \ref{theo:upper-bound-post-var}.

\begin{corollary}
Let $k$ be a kernel on $\Re^d$ whose RKHS is norm-equivalent to the Sobolev space of order $s > d/2$, and fix $\rho > 0$.
For $x \in \Re^d$, assume that $X = (x_1,\dots,x_n) \subset \Re^d$ satisfy $h_{\rho,X}(x) = O(n^{-b})$ as $n \to \infty$ for some $b > 0$.
Then we have
\begin{equation} \label{eq:marginal-var-rate}
\bar{k}(x,x) = O(n^{-b(2s - d)}) \quad (n \to \infty).
\end{equation}
\end{corollary}
\begin{remark}\rm
If $x_1,\dots,x_n$ are given as equally-spaced grid points in the ball of radius $\rho$ around $x$, then it can be easily shown that $h_{\rho,X}(x) = O(n^{-1/d})$ as $n \to \infty$.
Thus in this case, the rate in \eqref{eq:marginal-var-rate} becomes $\bar{k}(x,x) = O(n^{-(2s/d - 1) })$, which reveals the existence of the curse of dimensionality: the required number of points increases exponentially in the dimension $d$ to achieve a certain level of posterior contraction.
\end{remark}

\section{Integral Transforms} 
\label{sec:integral_transforms}

This section reviews some examples of the connections between RKHS and GP approaches that involve integrals of the kernel. 
Such computations arise both in the context of estimating a latent (probability) measure and when estimating the integral of a latent function against a known measure.
In Section \ref{sub:kernel_mean_embeddings_integrated_but_not_uncertain_observations}, we first introduce kernel mean embeddings of distributions and resulting metrics on probability measures (Maximum Mean Discrepancy), and then provide their probabilistic interpretations based on Gaussian processes.
We next describe their connections to kernel and Bayesian quadrature, which are approaches to numerical integration based on positive definite kernels in Section \ref{sec:kernel_and_bayesian_quadrature}.
Coming back to the statistical setting, a kernel mean shrinkage estimator and its Bayesian interpretation are discussed in Section \ref{sec:shrinkage}.
Finally, we review a nonparametric dependency measure called Hilbert-Schmidt Independence Criterion, which is defined via kernel mean embeddings, and present its probabilistic interpretation based on Gaussian processes.

In this section, we will use the following notation to denote integrals:
$$
Pf := \int f(x) dP(x), \quad P_n f := \sum_{i=1}^n w_i f(x_i),
$$
where $P$ is a measure on a measurable space $\cX$, $P_n := \sum_{i=1}^n w_i \delta_{x_i}$ is an empirical measure on $\cX$ with $(w_i,x_i)_{i=1}^n \subset \Re \times \cX$ and $\delta_{x_i}$ being a Dirac distribution at $x_i$, and $f:\cX \to \Re$ is a function.

\subsection{Maximum Mean Discrepancy: Worst Case and Average Case Errors} 
\label{sub:kernel_mean_embeddings_integrated_but_not_uncertain_observations}%

Let $(\cX, \mathfrak{B})$ be a measurable space, $k$ be a bounded measurable kernel on $\cX$ with $\cH_k$ being its RKHS, and $\mathcal{P}$ be the set of all probability measures on $\cX$.
For any $P \in \mathcal{P}$, its {\em kernel mean} (or mean embedding) is defined as the integral of the canonical feature map $k(\cdot,x)$ with respect to $P$: 
\begin{equation} \label{eq:kmean-def}
\mu_P := \int k(\cdot,x) dP(x)  \in \cH_k,
\end{equation}
which is well-defined as a Bochner integral, as long as $\int \sqrt{k(x,x)} dP(x) < \infty$ \citep[Theorem 1]{SriGreFukSchetal10}.
The kernel mean \eqref{eq:kmean-def} is an element in RKHS $\cH_k$ that represents $P$.

\begin{remark}\rm
The notion of embeddings of probability measures can be readily extended to embeddings of finite signed measures \citep{SriFukLan11}.
This is important because in the case where $P$ is an empirical approximation or estimator of the form $P_n := \sum_{i=1}^n w_i \delta_{x_i}$, the weights $w_1,\dots,w_n$ can be negative.
For instance, this is the case of kernel or Bayesian quadrature discussed in Section \ref{sec:kernel_and_bayesian_quadrature}.
\end{remark}

\paragraph{Characteristic kernels and MMD.}
If the mapping $P \in \mathcal{P} \mapsto \mu_P \in \cH_k$ is injective, that is, if $\mu_P = \mu_Q$ implies $P=Q$ for any probability measures $P$ and $Q$ on $\cX$, then the kernel $k$ is called {\em characteristic} \citep{FukBacJor2004,FukGreSunSch08,SriGreFukSchetal10}.
For instance, characteristic kernels on $\cX = \Re^d$ include the Gaussian and Mat\'ern kernels \citep{SriGreFukSchetal10}.
If kernel $k$ is characteristic, each kernel mean $\mu_P$ is uniquely associated with the embedded measure $P$, and therefore one can define a distance between probability measures $P$ and $Q$ as the distance between the kernel means $\mu_P$ and $\mu_Q$ in the RKHS:
\begin{equation} \label{eq:MMD}
{\rm MMD}(P,Q; \cH_k) := \left\| \mu_P - \mu_Q \right\|_{\cH_k} = \sup_{\substack{f \in \cH_k:\\ \| f \|_{\cH_k} \leq 1}} \left( Pf - Qf \right) = \sup_{\substack{f \in \cH_k:\\ \| f \|_{\cH_k} \leq 1}} \left| Pf - Qf \right|,
\end{equation}
where the second equality follows from $\cH_k$ being a vector space and the Cauchy-Schwartz inequality and the third follows from $\cH_k$ being a vector space; see the proof of \citet[Lemma 4]{GreBorRasSchetal12}.
Because of this expression, this distance between kernel means is called {\em maximum mean discrepancy} (MMD) \citep{GreBorRasSchetal12}, as it is the maximum difference between integrals (means) $Pf$ and $Qf$, when $f$ is taken from the unit ball in RKHS $\cH_k$.
If $k$ is characteristic, MMD becomes a proper metric on probability measures, and thus its estimator can be used as a test statistic in hypothesis testing for the two sample problem or goodness of fit \citep{GreBorRasSchetal12,pmlr-v48-chwialkowski16,pmlr-v48-liub16}.

\begin{remark}\rm
Let $k$ be a bounded, continuous shift-invariant kernel on $\cX = \Re^d$ such that $k(x,y) = \phi(x-y)$ for $x,y\in \Re^d$, where $\phi:\Re^d \to \Re$ is a positive definite function.
For such a kernel, \citet[Corollary 4]{SriGreFukSchetal10} provided a spectral characterization of MMD in terms of the Fourier transform $\mathcal{F}[\Phi]$ of $\Phi$: For any Borel probability measures $P$ and $Q$ on $\Re^d$, it holds that
$$
\MMD^2(P,Q,\cH_k) = \int \left| \psi_P(\omega) - \psi_Q (\omega) \right|^2 \cF [\Phi](\omega) d\omega
$$
where $\psi_P$ and $\psi_Q$ are the characteristic functions of $P$ and $Q$, respectively.
That is, MMD between $P$ and $Q$ is the weighted $L_2$ distance between the characteristic functions $\psi_P$ and $\psi_Q$, where the weight function is the Fourier transform $\mathcal{F}[\phi]$.
From this expression and the fact that characteristic functions and distributions are one-to-one, \citet{SriGreFukSchetal10} provided a necessary and sufficient condition for a shift-invariant kernel to be characteristic:
A bounded continuous shift invariant kernel $k$ is characteristic if and only if the support of the Fourier transform $\cF [\phi]$ is $\Re^d$ \citep[Theorem 9]{SriGreFukSchetal10}. 
\end{remark}
\begin{remark}\rm
The use of RKHS in defining \eqref{eq:MMD} is practically convenient because, thanks to the reproducing property, the squared MMD can be written as
$$
\left\| \mu_P - \mu_Q \right\|_{\cH_k}^2 
= \bE_{x,x'} [k(x,x')] - 2 \bE_{x,y} [k(x,y)] + \bE_{y,y'} [k(y,y')], 
$$
where $x, x' \sim P$ and $y, y' \sim Q$ are all independent \citep[Lemma 6]{GreBorRasSchetal12}.
Therefore, by replacing the expectations in the right side by empirical ones, one can straightforwardly estimate the MMD from samples.
\end{remark}

\paragraph{Worst case error.}
In the literature on numerical integration or quasi Monte Carlo, the right side of \eqref{eq:MMD} is known as the {\em worst case error} \citep{Hic98,DicKuoSlo13}. 
To be more precise, in the problem of numerical integration or sampling, $P$ is a {\em known} probability measure and $Q := P_n := \sum_{i=1}^n w_i \delta_{x_i}$ is its approximation, where $(w_i,x_i)_{i=1}^n \subset \Re \times \cX$ are generated by a user so that $P_n$ becomes an accurate approximation of $P$.
As such, one is interested in the quality of approximation of $P_n$ to $P$.
For this purpose, \eqref{eq:MMD} is used as a quantitative measure of approximation, and is interpreted as the worst case error of numerical integration $|Pf - P_n f|$ when $f$ is taken from the unit ball in RKHS $\cH_k$.
We will discuss this problem of numerical integration in detail in Section \ref{sec:kernel_and_bayesian_quadrature}.

\paragraph{Probabilistic interpretation as the average case error.}
Proposition \ref{prop:equiv-worst-ave} below provides a probabilistic interpretation of MMD in terms of the GP of the kernel $k$.
More specifically, MMD between $P$ and $Q$ can be understood as the {\em expected squared difference} between integrals $P\ssf$ and $Q\ssf$, where the expectation is with respect to a draw $\ssf$ from $\GP(0,k)$. 
In the terminology of numerical integration, this shows the equivalence between the RKHS worst case error and the Gaussian process {\em average case error}.
While this equivalence has been known in the literature \citep[Corollary 7 in p.40]{Rit00}, we provide a proof, as it is instructive.
\begin{proposition} \label{prop:equiv-worst-ave}
Let $k$ be a bounded kernel on a measurable space $\cX$, and $P$ and $Q$ be finite measures on $\cX$.
Then we have
\begin{equation} \label{eq:worst-ave}
{\rm MMD}^2(P,Q; \cH_k) = \left( \sup_{ \| f \|_{\cH_k} \leq 1 } \left( P f  - Q f \right) \right)^2 = \bE_{\ssf \sim \GP(0,k)} \left[ \left( P \ssf - Q \ssf \right)^2  \right]. 
\end{equation}
\end{proposition}


\begin{proof}
Let $x$ and $x'$ be independent random variables following $P$, $y$ and $y'$ be those following $Q$, and $\ssf$ be an independent draw from $\GP(0,k)$.
By using the reproducing property and the expression \eqref{eq:kerenl-gp-expect} of the kernel $k(x,y) = \bE_{\ssf}[\ssf(x)\ssf(y)]$
\begin{eqnarray*}
\left\| \mu_P - \mu_Q \right\|_{\cH_k}^2 
&=& \bE_{x,x'} [k(x,x')] - 2 \bE_{x,y} [k(x,y)] + \bE_{y,y'} [k(y,y')] \\
&=& \bE_{x,x'} [\bE_{\ssf}\ssf(x)\ssf(x')] - 2 \bE_{x,y} [\bE_{\ssf}\ssf(x)\ssf(y)] + \bE_{y,y'} [\bE_{\ssf}\ssf(y)\ssf(y')] \\
&\stackrel{*}{=}& \bE_{\ssf} \left[ \bE_{x,x'} [\ssf(x)\ssf(x')] - 2 \bE_{x,y} [\ssf(x)\ssf(y)] + \bE_{y,y'} [\ssf(y)\ssf(y')] \right] \\
&=& \bE_{\ssf} \left[ (\bE_{x} [\ssf(x)])^2 - 2 \bE_{x} [\ssf(x)] \bE_y [\ssf(y)] + (\bE_{y} [\ssf(y)])^2 \right] \\
&=&  \bE_{\ssf} \left[ \left( \bE_{x} [\ssf(x)] - \bE_{y}[\ssf(y)] \right)^2 \right],
\end{eqnarray*}
where $\stackrel{*}{=}$ follows from Fubini's theorem, 
which is applicable because we have
$$
\bE_{x,x'}\bE_{\ssf} |\ssf(x)\ssf(x')|\le \bE_{x,x'}\sqrt{\bE_{\ssf}\ssf^2(x)}\sqrt{\bE_{\ssf}\ssf^2(x')}=\bE_{x}\sqrt{k(x,x)})\bE_{x'}\sqrt{k(x',x')})<\infty,
$$ 
where the last inequality follows from $k$ being bounded.
\end{proof}

\begin{remark}\rm 
Since $\ssf \sim \GP(0,k)$ is a mean-zero Gaussian process, the expectation of the real-valued random variable $P\ssf-Q\ssf$ is $0$.
Therefore the right side of \eqref{eq:worst-ave} can be seen as the variance of this random variable $P\ssf-Q\ssf$:
$$
\bE_{\ssf \sim \GP(0,k)} \left[ \left( P \ssf - Q \ssf \right)^2  \right]  = \Var[P\ssf-Q\ssf].  
$$
Since $P\ssf-Q\ssf$ is a linear transform of Gaussian process $\ssf$, it is a real-valued Gaussian random variable.
This implies that, when dealing with MMD between $P$ and $Q$, one implicitly deals with the distribution of $P\ssf - Q \ssf$, which is $\mathcal{N}(0, \sigma^2)$ where $\sigma^2 := \MMD(P,Q; \cH_k)$.
\end{remark}

The following corollary immediately follows from Proposition \ref{prop:equiv-worst-ave} and the definition of a kernel being characteristic.
It provides a probabilistic interpretation of a characteristic kernel in terms of the corresponding Gaussian process.
\begin{corollary} \label{coro:GP-characteristic}
Let $k$ be a bounded characteristic kernel on a measurable space $\cX$.
Then for any probability measures $P$ and $Q$ on $\cX$, we have $P=Q$ if and only if $P\ssf=Q\ssf$ holds almost surely for a Gaussian process $\ssf \sim \GP(0,k)$.
\end{corollary}

\subsection{Sampling and Numerical Integration} 
\label{sec:kernel_and_bayesian_quadrature}
Here we consider the problem of numerical integration or (deterministic) sampling.
Let $f: \cX \to \Re$ be an integrand and $P$ be a {\em known} probability measure on $\cX$, and assume that the integral $\int f(x)dP(x)$ cannot be computed analytically.
The task is to numerically compute the integral as a weighted sum of function values
$$
\sum_{i=1}^n w_i f(x_i) \approx \int f(x) dP(x).
$$
Therefore the problem is how to select the weighted points $(w_i, x_i)_{i=1}^n \subset \Re \times \cX$ so that this approximation becomes as accurate as possible.
If one has prior knowledge about certain properties of $f$ such as its smoothness, then one can use this information for the construction of a quadrature rule.
This can be done by making use of positive definite kernels.

\paragraph{Kernel quadrature.}
For simplicity, assume that design points $x_1,\dots,x_n \in \cX$ are already given and fixed.
In kernel quadrature, weights $w_1,\dots,w_n$ are obtained by the minimization of MMD between $Q := P_n := \sum_i w_i \delta_{x_i}$ and $P$:
\begin{equation*}
\min_{w_1,\dots,w_n \in \Re} \MMD(P_n,P;\cH_k) = \min_{w_1,\dots,w_n \in \Re} \sup_{ \| f \|_{\cH_k} \leq 1} \left| P_n f - Pf \right|, 
\end{equation*}
which is, as shown in the right side, equivalent to the minimization of the worst case error in the unit ball of RKHS $\cH_k$.
Assuming that kernel matrix $k_{XX}$ is invertible, this optimization problem can be solved in closed form, and the resulting weights are given by
\begin{equation} \label{eq:KQ_weights}
(w_1,\dots,w_n)^T = k_{XX}^{-1} \mu_X \in \Re^n,
\end{equation}
where $\mu_{X} := (\mu_P(x_i) )_{i=1}^n \in \Re^n$.
Using these weights, integral $Pf$ is approximated as 
\begin{equation} \label{eq:KQ-func-integral}
P_nf = \sum_{i=1}^n w_i f(x_i).
\end{equation}
We assumed here that design points $x_1,\dots,x_n$ are given at the beginning, but there are also approaches that obtain design points by aiming at the minimization of the worst case error; examples include quasi Monte Carlo methods \citep{DicKuoSlo13} and kernel herding \citep{CheWelSmo10}.

\begin{remark}\rm
To calculate the weights in \eqref{eq:KQ_weights}, one needs to be able to evaluate function values of the kernel mean $\mu_P = \int k(\cdot,x)dP(x)$, thus requiring a certain compatibility between $k$ and $P$.
For instance, this is possible when $k$ is square-exponential and $P$ is Gaussian on $\cX \subset \Re^d$.
For other examples, see \citet[Table 1]{BriOatGirOsbSej15}.
Note that one is also able to perform kernel quadrature using kernel Stein discrepancy; in this case the weights \eqref{eq:KQ_weights} can be calculated if one knows the gradient of log density of $P$ \citep{OatGirCho17,LiuLee17}.
This remark also applies to Bayesian quadrature explained blow.
\end{remark}

\paragraph{Bayesian quadrature.}
Bayesian quadrature \citep{diaconis1988bayesian,Oha91,BriOatGirOsbSej15,KarOatSar18} is a probabilistic approach to numerical integration based on Gaussian processes.
As before, assume for simplicity that design points $x_1,\dots,x_n$ are fixed.
In this approach, a prior distribution is put on the integrand $f$ as a Gaussian process $\ssf \sim \GP(0,k)$.
Assume that functions values $f(x_1),\dots,f(x_n)$ at the design points are provided.
In Bayesian quadrature, these input-output pairs $(x_i,f(x_i))_{i=1}^n$ are regarded as ``observed data.''
Then the posterior distribution of the integral is given by 
\begin{equation} \label{eq:BQ-posterior}
P\ssf \mid (x_i,f(x_i))_{i=1}^n \sim \mathcal{N}(\mu_n, \sigma_n^2),
\end{equation}
where $\mu_n \in \Re$ and $\sigma_n^2 > 0$ are respectively the posterior mean and variance given by
\begin{eqnarray} 
 \mu_n &:=& \mu_X\Trans k_{XX}^{-1} f_X = P_n f,  \label{eq:BQ-mean-quad}  \\
 \sigma_n^2 &:=& \int \int  k(x,x') dP(x)dP(x') - \mu_X\Trans k_{XX} ^{-1} \mu_X, \label{eq:BQ-post-quad}
\end{eqnarray}
where $\mu_X := (\mu_P(x_i))_{i=1}^n \in \Re^n$ with $\mu_P$ being the kernel mean and $P_n := \sum_{i=1}^n w_i \delta_{x_i}$.

\begin{remark}\rm
As discussed in Section \ref{sec:GP-regression}, since we deal with noise-free observations $f(x_1), \dots,$ $f(x_n)$, there is no likelihood model in the above derivation (or, the likelihood function is degenerate).
Therefore \eqref{eq:BQ-posterior} is not a ``posterior distribution'' in the usual sense of Bayesian inference.
However, \eqref{eq:BQ-posterior} is still well-defined as a conditional distribution of $P\ssf$ given $(x_i,f(x_i))_{i=1}^n$ \citep[Section 2.5]{CocOatSulGir17}.
For discussion regarding what it means by ``Bayesian'' in the noise-free setting or in numerical analysis, we refer to
\citet{CocOatSulGir17}.
\end{remark}

\begin{remark}\rm
Note that the posterior variance \eqref{eq:BQ-mean-quad} does not depend on the given integrand $f$, and determined only by the kernel $k$, the design points $x_1,\dots,x_n$ and the measure $P$.
\end{remark}

First note that the posterior mean \eqref{eq:BQ-mean-quad} is identical to the integral estimate \eqref{eq:KQ-func-integral} of kernel quadrature.
The following result shows that the posterior variance \eqref{eq:BQ-post-quad} of Bayesian quadrature is also equal to the squared MMD between  $P_n$ and $P$, where the weights $w_1,\dots,w_n$ are given in \eqref{eq:KQ_weights}.
This identity has been known in the literature \citep{huszar2012optimally,BriOatGirOsbSej15}, but we provide a proof, as it is simple and instructive. 
\begin{proposition} \label{prop:BQ-KQ-identity}
Let $\sigma_n^2$ be the posterior variance \eqref{eq:BQ-post-quad} of Bayesian quadrature, and $P_n := \sum_{i=1}^n w_i \delta_{x_i}$ be the empirical measure with the weights $w_1,\dots,w_n$ given in \eqref{eq:KQ_weights}.
Then
\begin{equation} \label{eq:quadrature-equivalence}
\sigma_n^2 = \left\| \mu_{P_n} - \mu_P \right\|_{\cH_k}^2 = \MMD^2(P_n,P;\cH_k).
\end{equation}
\end{proposition}


\begin{proof}
By the reproducing property and the definition of the weights $w = (w_1,\dots,w_n)\Trans \in \Re^n$, we have
\begin{eqnarray*}
\left\| \mu_{P_n} - \mu_P \right\|_{\cH_k}^2
&=& \| \mu_{P_n} \|_{\cH_k}^2 - 2 \left< \mu_{P_n}, \mu_P \right>_{\cH_k} + \| \mu_{P} \|_{\cH_k}^2 \\
&=& w\Trans k_{XX} w - 2 w\Trans \mu_X + \int \int k(x,x') dP(x)dP(x') \\
&=& - \mu_X\Trans k_{XX}^{-1} \mu_X + \int \int k(x,x') dP(x)dP(x'),
\end{eqnarray*}
and the result follows.
\end{proof}
\begin{remark}\rm
Proposition \ref{prop:BQ-KQ-identity} shows the equivalence between the average case error w.r.t.~GPs and the (squared) worst case error in the RKHS, in the setting of numerical integration.
This is because \eqref{eq:quadrature-equivalence} can be written as
$$
\bE_{\ssf} \left[(P\ssf - \mu_n)^2 \mid (x_i,f(x_i))_{i=1}^n \right] = \left( \sup_{\| f \|_{\cH_k} \leq 1} \left| P_n f - P f \right| \right)^2.
$$
This equivalence has been used by \citet{BriOatGirOsb15,BriOatGirOsbSej15} to establish posterior contraction rates of Bayesian quadrature, by transferring results on the worst case error \citep{BacJulObo12,DicKuoSlo13} to the probabilistic or Bayesian setting.
\end{remark}


\paragraph{Noisy observations and robustness to misspecification.}
If one's knowledge about integrand $f$ of interest is limited, it could happen that it does not belong to RKHS $\cH_k$, that is, misspecification of the hypothesis class may occur.
\citet{KanSriFuk16,KanSriFuk17} showed that kernel quadrature can be made robust to such misspecification, by introducing a quadratic regularizer for the weights, i.e., 
\begin{equation} \label{eq:kernel-quadrature-quad-reg}
\min_{w_1,\dots,w_n \in \Re^n} \MMD^2(P_n,P; \cH_k) + \lambda \sum_{i=1}^n w_i^2,
\end{equation}
where $\lambda > 0$ is a regularization constant.
The resulting weights are then given by
\begin{equation} \label{eq:KQ_weights_reg}
(w_1,\dots,w_n)\Trans = (k_{XX} + n \lambda I_n)^{-1} \mu_X \in \Re^n.
\end{equation}
Kernel quadrature with quadratic weight regularization has also been studied by \citet{Bac17}, who pointed out such regularization is equivalent to assuming the existence of additive noises in the function values. That is, assume that, instead of observing the exact function values $f(x_1),\dots,f(x_n)$, one is given noisy observations $y_i = f(x_i) + \varepsilon_i$, where $\varepsilon_i \iid \mathcal{N}(0,\sigma^2)$ with $\sigma^2 := n \lambda$.
Bayesian quadrature under this assumption yields the posterior distribution of the integral as $P\ssf \mid (x_i,y_i)_{i=1}^n \sim \mathcal{N}(\mu_n, \sigma_n^2)$, where 
\begin{eqnarray*}
\mu_n &:=& \mu_X\Trans (k_{XX} + \sigma^2 I_n)^{-1} Y = \sum_{i=1}^n w_i y_i, \\ \sigma_n^2 &:=& \int \int  k(x,x') dP(x)dP(x') - \mu_X\Trans (k_{XX} + \sigma^2 I_n)^{-1} \mu_X,
\end{eqnarray*}
where $w_1,\dots,w_n$ are given by \eqref{eq:KQ_weights_reg}.
These are regularized versions of \eqref{eq:BQ-mean-quad} and \eqref{eq:BQ-post-quad}, with $f_X$ replaced by $Y = (y_1,\dots,y_n)\Trans \in \Re^n$.

\paragraph{Discussion on the difference between the kernel and Bayesian approaches.}
The optimization viewpoint of kernel quadrature allows one to directly incorporate a constraint on quadrature weights $w_1,\dots,w_n$.
For instance, \citet{LiuLee17} proposed to optimize the weights under a constraint that the weights be non-negative. 
Such a constraint is not straightforward to be realized only with a probabilistic perspective.
On the other hand, with Bayesian quadrature one can express prior knowledge about the integrand that is not easy to be incorporated with the kernel approach.
For instance, \citet{gunter_sampling_2014} proposed to model an integrand that is non-negative as a squared GP (i.e., as a chi-squared process).
Such modeling can be realized because one expresses the prior knowledge as a generative model, but this is not straightforward to achieve through the optimization of weights.

\subsection{Kernel Mean Shrinkage Estimator and Its Bayesian Interpretation}
\label{sec:shrinkage}
Given an i.i.d.~sample $x_1,\dots,x_n \sim P$, an empirical estimator of the kernel mean \eqref{eq:kmean-def} is defined as
\begin{equation} \label{eq:emp_kmean}
\hat{\mu}_P := \frac{1}{n} \sum_{i=1}^n k(\cdot,x_i),
\end{equation}
which satisfies $\bE\left[\| \hmu_P - \mu_P \|_{\cH_k}\right] = O(n^{-1/2})$ as $n \to \infty$, if $k$ is bounded \citep{SmoGreSonSch07,TohSriMua17}.
One way to compute MMD empirically is to substitute this estimate (and that for $\mu_Q$) in (\ref{eq:MMD}): this results in a V-statistic estimate of MMD \citep[Eq.~5]{GreBorRasSchetal12}.
\citet{TohSriMua17} showed that the rate $n^{-1/2}$ is minimax-optimal and thus cannot be improved, meaning that \eqref{eq:emp_kmean} is an {\em asymptotically} optimal estimator.
However, when the sample size $n$ is {\em fixed}, it is known that there exists a ``better'' estimator that improves upon \eqref{eq:emp_kmean}, which was shown by \citet{JMLR:v17:14-195}.
More precisely, consider an estimator $ \hat{\mu}_{P,\alpha}$ defined by 
\begin{equation} \label{eq:kmean-shrinkage-first}
\hat{\mu}_{P,\alpha} := f^* + (1-\alpha) \hat{\mu}_P,
\end{equation} 
where $f^* \in \cH_k$ is fixed and arbitrary and $\alpha$ is a constant. 
\citet[Theorem 1]{JMLR:v17:14-195} proved that, if (and only if) the constant satisfies $0 < \alpha < 2 \bE[\| \hat{\mu}_P - \mu_P \|_{\cH_k}^2] / (  \bE[\| \hat{\mu}_P - \mu_P \|_{\cH_k}^2] + \| f^* - \mu_P \|_{\cH_k}^2 )$, then $\hat{\mu}_{P,\alpha}$ produces a smaller mean-squared error than $\hat{\mu}_P$:
$$
\bE\left[\| \hat{\mu}_{P,\alpha} - \mu_P \|_{\cH_k}^2 \right] < \bE\left[\| \hmu_P - \mu_P \|_{\cH_k}^2 \right].
$$
A motivation for \cite{JMLR:v17:14-195} was the so called {\em Stein phenomenon}, which states that the standard empirical estimator for the mean of a $d$-dimensional Gaussian distribution with $d \geq 3$ is {\em inadmissible}, meaning that there exists an estimator that yields smaller mean squared error for a fixed sample size \citep{Ste56}.
Motivated by this old result, \cite{JMLR:v17:14-195} proposed a number of shrinkage estimators including the one in \eqref{eq:kmean-shrinkage-first}, some of which have been theoretically proven to be ``better'' than the standard empirical estimator in \eqref{eq:emp_kmean} in terms of the mean squared RKHS error.



We review here a certain shrinkage estimator proposed by \citet[Section 4]{JMLR:v17:14-195} called the {\em spectral kernel mean estimator} (SKME), and its Bayesian interpretation given by \citet{FlaSejCunFil2016}; this provides another instance of the connection between the kernel and Bayesian approaches.
Different from \eqref{eq:kmean-shrinkage-first}, however, the SKME has {\em not} been shown to be theoretically better than the empirical estimator \eqref{eq:emp_kmean}, but has only been shown to yield better empirical performance.
Given an i.i.d.~sample $X = (x_1,\dots,x_n) \sim P$, the SKME is defined as
\begin{equation} \label{eq:kmean-shrink}
\cmu_{P,\lambda} := \sum_{i=1}^n w_i k(\cdot, x_i)
\end{equation}
where the weights $w_1,\dots,w_n \in \Re$ are given by
\begin{equation} \label{eq:weights-BQ-shrin}
(w_1,\dots,w_n)^T := (k_{XX} + n \lambda I_n)^{-1} \hmu_X \in \Re^n,
\end{equation}
with $\hmu_X := (\hmu_P(x_i))_{i=1}^n \in \Re^n$, $k_{XX} \in \Re^{n \times n}$ being the kernel matrix, and $\lambda > 0$ being a regularization constant.
The estimator in \eqref{eq:kmean-shrink} was originally derived from a certain RKHS-valued ridge regression problem.
Alternatively, the estimator can be derived by solving the following minimization problem:
\begin{equation} \label{eq:weights_skme}
 \min_{w_1,\dots,w_n \in \Re} \left\| \sum_{i=1}^n w_i k(\cdot,X_i) - \hat{\mu}_P \right\|_\cH^2 + \lambda \| w \|^2,
\end{equation}
where $\hmu_P$ is the empirical estimator in \eqref{eq:emp_kmean}.
The weights in \eqref{eq:weights-BQ-shrin} are given as the solution of this optimization problem.
This interpretation shows that as the regularization constant $\lambda$ increases, the estimator \eqref{eq:kmean-shrink} shrinks towards zero in the RKHS, thus reducing the variance of the estimator while introducing bias.
Therefore $\lambda$ controls the bias-variance trade off; this is beneficial in practice when the sample size is relatively small, in which case the variance of the empirical estimator \eqref{eq:emp_kmean} may be large.
\begin{remark}\rm
The weights \eqref{eq:weights-BQ-shrin} of the SKME are essentially the same as those for kernel quadrature with quadratic regularization \eqref{eq:KQ_weights_reg}. 
The only difference is that, while the information of the true kernel mean $\mu_P$ is used in \eqref{eq:KQ_weights_reg}, the empirical mean $\hat{\mu}_P$ is used in \eqref{eq:weights-BQ-shrin}, since in the statistical setting $\mu_P$ is an unknown quantity to be estimated.
Note also the essential equivalence of the two optimization problems \eqref{eq:kernel-quadrature-quad-reg} and \eqref{eq:weights_skme}, based on which the weights \eqref{eq:KQ_weights_reg} and \eqref{eq:weights-BQ-shrin} are respectively derived.
\end{remark}

\paragraph{Bayesian interpretation of the shrinkage estimator.}
\citet{FlaSejCunFil2016} showed that there exists a Bayesian interpretation of the shrinkage estimator in \eqref{eq:kmean-shrink}.
We formulate their approach by using the powered kernel (\ref{eq:power_kernel}) in Section \ref{sec:power-RKHS-sample-path}.
The prior for the kernel mean $\mu_P$ is defined as a Gaussian process:
\begin{equation} \label{eq:flexman-prior}
\mu_P \sim \GP(0, k^\theta),
\end{equation}
where $k^\theta$ is the powered kernel (\ref{eq:power_kernel}) with the power $\theta \geq 1$ appropriately chosen so that $\mu_P$ can be a sample path of the GP.
On the other hand, by regarding the evaluations $\hmu_P$ at sample points $x_1,\dots,x_n$ as ``observations'', \citet{FlaSejCunFil2016} proposed to define a likelihood function as an additive Gaussian-noise model
\begin{equation} 
 \hmu_P(x_i) = \mu_P(x_i) + \varepsilon_i, \quad \varepsilon_i \iid \N(0, \sigma^2), \quad i=1,\dots,n. \label{eq:flexman-likelihood}
\end{equation}
Note that these assumptions are essentially those of GP-regression.
(We will discuss the validity of this assumption shortly.)
Therefore a similar argument as in Section \ref{sec:GP-regression} implies that the posterior distribution of $\mu_P$ is also a GP and given by
\begin{equation*}
 \mu_P \mid (x_i,\hmu_P(x_i))_{i=1}^n\   \sim\ \GP( \bar{\mu}_P, \bar{k}^\theta), 
\end{equation*}
where $\bar{\mu}_P$ and $\bar{k}^\theta$ are the posterior mean function and the posterior covariance function, respectively, and are given by
\begin{eqnarray}
\bar{\mu}_P (x) &:=& k_{xX}^\theta (k_{XX}^\theta + \sigma^2 I_n )^{-1} \hmu_{X}, \quad x \in \cX, \label{eq:post-kmean} \\
\bar{k}^\theta(x,x') &:=& k^\theta(x,x') - k^\theta_{xX}(k_{XX}^\theta + \sigma^2 I_n)^{-1} k^\theta_{Xx'}, \quad x,x' \in \cX,  \label{eq:post-var-kmean}
\end{eqnarray}
where $k_{Xx}^\theta = {k_{xX}^\theta}\Trans := (k^\theta(x,x_i))_{i=1}^n \in \Re^n$, $k_{XX}^\theta := (k^\theta(x_i,x_j)) \in \Re^{n \times n}$ and $\hmu_X := (\hmu_P(x_i))_{i=1}^n \in \Re^n$.

\begin{remark}\rm
If $\theta = 1$ and $\lambda = \sigma^2/n$, the posterior mean function \eqref{eq:post-kmean} is equal to the shrinkage estimator \eqref{eq:kmean-shrink}.
Therefore in this case, the modeling assumptions \eqref{eq:flexman-prior} and \eqref{eq:flexman-likelihood} provide a probabilistic interpretation for the shrinkage estimator \eqref{eq:kmean-shrink}.
In other words, the shrinkage estimator \eqref{eq:kmean-shrink} implicitly performs Bayesian inference under the probabilistic model \eqref{eq:flexman-prior} and \eqref{eq:flexman-likelihood}.
This probabilistic viewpoint turns out to be practically useful.
For instance, \cite{FlaSejCunFil2016} made use of the probabilistic formulation for selecting the kernel hyper-parameter (and the noise variance $\sigma^2$) in an unsupervised fashion, by using the empirical Bayes method; \cite{LawSutSejFla2017} used the posterior covariance function \eqref{eq:post-var-kmean} to enable uncertainty quantification in application to distribution regression.
\end{remark}


\paragraph{Discussion on the modeling assumption.}
We make some remarks on the modeling assumptions \eqref{eq:flexman-prior} and \eqref{eq:flexman-likelihood}.
First, as mentioned above, the GP prior \eqref{eq:flexman-prior} should be defined so that the kernel mean $\mu_P$ can be a sample path of the GP.
If $\theta = 1$, this is not the case: As reviewed earlier, GP sample paths do not belong to the RKHS of the covariance kernel with probability one.
This fact motivated \citet{FlaSejCunFil2016} to use a certain kernel that is smoother than the kernel defining the kernel mean, in order to guarantee that the kernel mean can be a GP sample path.
We instead defined the GP prior \eqref{eq:flexman-prior} using the powered kernel $k^\theta$: If $\theta > 1$, the kernel $k^\theta$ is ``smoother'' than the original kernel $k$, and there is ``sufficiently large'' $\theta$ to guarantee that a GP sample path lies in the RKHS.
For instance, if $k$ is a square-exponential kernel it can be shown from Corollary \ref{coro:gauss-sample-path} that, the choice $\theta = 1 + \varepsilon$ with $\varepsilon > 0$ being arbitrarily small guarantees that 
sample paths of $\GP(0,k^\theta)$ belong to the RKHS of $k$ with probability one.
In other words, $k^\theta$ can be chosen so that the resulting power of RKHS $\cH_k^\theta$ is ``infinitesimally smaller'' than the original RKHS $\cH_k$. 
This may explain why the use of $\theta = 1$ with a square-exponential kernel resulted in good empirical performance in \citet{LawSutSejFla2017}.

We note that \citet{FlaSejCunFil2016} introduced the likelihood model \eqref{eq:flexman-likelihood} for computational feasibility, i.e., to obtain the posterior distribution of $\mu_P$ as a GP.
Therefore, the assumption \eqref{eq:flexman-likelihood} may not be conceptually well-motivated.
For instance, if the kernel $k$ is bounded, so is $\hmu_P = \frac{1}{n}\sum_{i=1}^n k(\cdot,x_i)$; this fact is not captured by the assumption that noise $\varepsilon_i$ is additive Gaussian, which is unbounded.
Moreover, noise $\varepsilon_i$ is neither independent nor of constant variance in general, since the differences $\hat{\mu}_P(x_i) - \mu_P(x_i)$ at different points $x_i$ are dependent. 

\paragraph{Open questions.}
If $\theta > 1$, which is required to ensure that $\mu_P$ is a sample path of $\GP(0,k^\theta)$, then the resulting posterior mean function \eqref{eq:post-kmean} does not coincide with the shrinkage estimator \eqref{eq:kmean-shrink}.
This raises the following question: Can the use of the smoother kernel $k^\theta$ lead to ``better'' performance for estimation of $\mu_P = \int k(\cdot,x)dP(x)$ in terms of the mean-square error with a fixed sample size, when compared to the standard empirical estimator \eqref{eq:emp_kmean}?
\citet{JMLR:v17:14-195} was not able to show such superiority of the shrinkage estimator; this may be because they used the kernel $k^\theta$ with $\theta = 1$ in  \eqref{eq:kmean-shrink}, which is not supported from the Bayesian interpretation.

\paragraph{Relation to Bayesian quadrature.}
For any RKHS function $f \in \cH$, the integral $\int f(x)dP(x)$ can be estimated as a weighted sum $\sum_{i=1}^n w_i f(x_i)$, where the weighted points $(w_i,x_i)_{i=1}^n$ are those expressing $\cmu_{P,\lambda}$ as in \eqref{eq:kmean-shrink}.
This follows from the inequality
$$
\left|\sum_{i=1}^n w_i f(x_i) - \int f(x)dP(x) \right| = \left| \left<\cmu_{P,\lambda} - \mu_P, f\right>_{\cH_k} \right| \leq \left\| \cmu_{P,\lambda} - \mu_P \right\|_{\cH_k} \|f\|_{\cH_k}
$$
and that $\cmu_P$ should be close to $\mu_P$.
It is easy to show that the weighted sum can be written as
$$
\sum_{i=1}^n w_i f(x_i) = \frac{1}{n} \sum_{i=1}^n \bar{m}(x_i),
$$
where $\bar{m}:\cX \to \Re$ is the posterior mean function \eqref{eq:posteior_mean} in GP regression.
In other words, the weighted sum is equal to the empirical mean of the the fitted function $\bar{m}$.
As shown in Section \ref{sec:kernel_and_bayesian_quadrature}, this is essentially Bayesian quadrature using the empirical measure $P_n = \frac{1}{n} \sum_{i=1}^n \delta_{x_i}$.

\subsection{Gaussian Process Interpretation of Hilbert Schmidt Independence Criterion} \label{sec:dependence}

Here we deal with a popular kernel-based dependency measure known as the {\em Hilbert-Schmidt Independence Criterion} (HSIC) \citep{GreBouSmoSch05}, which can be defined as the MMD between the joint distribution of two random variables and the product of their marginals.
HSIC is a nonparametric dependency measure, and as such does not require a specific parametric assumption about the form of dependencies between random variable variables.
Since it also can be calculated in a simple way using kernels, it has found a wide range of applications including independence testing \citep{GreFukTeoSonSchSmo08,ZhaFilGreSej18}, variable selection \citep{SonSmoGreBedetal12,YamJitSigXinSug14}, post selection inference \citep{pmlr-v84-yamada18a}, and causal discovery \citep{NikBuhSchPet18}, to name a few.
We provide here a probabilistic interpretation of HSIC in terms on GPs; to the best of our knowledge, this probabilistic interpretation of general HSIC measures is novel and it recovers Brownian distance covariance of \citet{SzeRiz09}, known to be a special case of HSIC \citep{SejSriGreFuk13}.

Let $X$ and $Y$ be random variables taking values in measurable spaces $\cX$ and $\cY$ respectively.
Denote by $P_{\cX \times \cY}$ the joint distribution of $X$ and $Y$, and let $P_\cX$ and $P_\cY$ be its marginal distributions on $\cX$ and $\cY$, respectively.
Let $k$ and $\ell$ be positive definite kernels on $\cX$ and $\cY$ respectively, and let $\cH_\cX$ and $\cH_\cY$ be their respective RKHSs.
For the product kernel $k \otimes \ell: (\cX \times \cY) \times (\cX \times \cY) \to \Re$ defined by $(k \otimes \ell) ((x,y),(x',y')) := k(x,x')\ell(y,y')$, the corresponding RKHS is denoted by $\cH_\cX \otimes \cH_\cY$, which is the tensor product of $\cH_\cX$ and $\cH_\cY$.

While HSIC was originally defined by \citet{GreBouSmoSch05} as the Hilbert-Schmidt norm of a certain cross-covariance operator \citep{FukBacJor2004}, 
we follow here an equivalent definition given by \citet{SmoGreSonSch07}: HSIC is defined as the (squared) MMD between $P_{\cX \times \cY}$ and $P_\cX P_\cY$ in the RKHS $\cH_\cX \otimes \cH_\cY$:
\begin{equation} \label{eq:HSIC-def}
\HSIC(X,Y) := \MMD^2(P_{\cX \times \cY}, P_\cX P_\cY) = \left\| \mu_{P_{\cX \times \cY}} - \mu_{P_\cX} \otimes \mu_{P_\cY} \right\|_{\cH_\cX \otimes \cH_\cY}^2,
\end{equation}
where $\mu_{P_{\cX \times \cY}}$ and $\mu_{P_\cX} \otimes \mu_{P_\cY}$ are the kernel means of $P_{\cX \times \cY}$ and $P_\cX P_\cY$, respectively.
If the kernel $k \otimes \ell$ is characteristic, then the HSIC is zero if and only if $X$ and $Y$ are independent; see \citet{szabo17characteristic_TR} for thorough analysis of conditions for $k \otimes \ell$ being characteristic.
Thus in this case, HSIC is qualified as a nonparametric measure of dependence. Thanks to the reproducing property, HSIC can be expressed in terms of expectations of the kernels \cite[Lemma 1]{GreBouSmoSch05}:
\begin{eqnarray}
&& \HSIC(X,Y, k, \ell) = \bE_{X,Y,X',Y'} \left[ k(X,X') \ell(Y,Y') \right] \label{eq:HSIC-kernel} \\
&& - 2 \bE_{X,Y}\left[ \bE_{X'}[k(X,X')] \bE_{Y'} \ell(Y,Y')\right] + \bE_{X,X'} \left[ k(X,X') \right] \bE_{Y,Y'}\left[ \ell(Y,Y') \right], \nonumber
\end{eqnarray}
where $X'$ and $Y'$ are respectively independent copies of $X$ and $Y$.
Given an i.i.d.~sample $((X_i,Y_i))^n_{i=1}$ from the joint distribution $P_{\cX \times \cY}$, an empirical estimator of HSIC is straightforwardly given by replacing the expectations in (\ref{eq:HSIC-kernel}) by the corresponding empirical averages; for details see \cite{GreBouSmoSch05}.

\paragraph{Gaussian Process Interpretation.}
Consider independent draws from the zero-mean GPs of the covariance kernels $k$ and $\ell$:
\begin{equation*} \label{eq:GPs-for-HSIC}
\ssf \sim \GP(0,k), \quad \ssg \sim \GP(0,\ell).
\end{equation*}
The following result provides a probabilistic interpretation of HSIC in terms of these GPs. 
\begin{proposition} \label{prop:HSIC-GP}
Let $k$ and $\ell$ be positive definite kernels on measurable spaces $\cX$ and $\cY$ respectively, and let $\ssf \sim \GP(0,k)$ and $\ssg \sim \GP(0,\ell)$ be independent Gaussian processes. 
For random variables $X$ and $Y$ taking values respectively in $\cX$ and $\cY$, we have
\begin{equation} \label{eq:HSIC-GP-Interp}
\HSIC(X,Y) = \bE_{\ssf,\ssg} \cov^2(\ssf(X), \ssg(Y)),
\end{equation}
where $\HSIC(X,Y)$ is defined by \eqref{eq:HSIC-def}, and
\begin{equation*}
\cov(\ssf(X), \ssg(Y)) := \bE_{X,Y}\left[ \left(\ssf(X) - \bE[\ssf(X)|\ssf]\right) \left( \ssg(Y) - \bE[ \ssg(Y) | \ssg ] \right) | \ssf,\ssg \right] .
\end{equation*}
\end{proposition}

Note that $\cov(\ssf(X), \ssg(Y))$ is the covariance between the {\em real-valued} random variables  $\ssf(X)$ and $\ssg(Y)$, with $\ssf$ and $\ssg$ being fixed.
Proposition \ref{prop:HSIC-GP} thus shows that HSIC is the expectation of the square of this covariance with respect to the draws $\ssf \sim \GP(0,k)$ and $\ssg \sim \GP(0,\ell)$.
In other words, the computation of $\HSIC(X,Y)$ amounts to simultaneously considering various nonlinear transformations $\ssf(X)$ and $\ssg(Y)$ of random variables $X$ and $Y$ (as defined by the GPs), and then computing the average of the (squared) covariance between these transformed variables.
We believe that this interpretation provides a simple way to understand HSIC as a measure of dependence, in particular for people who are familiar with GPs but not with RKHSs.

\paragraph{Connection to Brownian Covariance.}
We mention that the expression in the right side of \eqref{eq:HSIC-GP-Interp} is related to the {\em Brownian (distance) covariance} introduced by \citet{SzeRiz09}.
To describe this, let $X_\ssf$ and $Y_\ssg$ be {\em real-valued} random variables such that
\begin{eqnarray}
X_\ssf &:=& \ssf(X) - \bE[\ssf(X)|\ssf] = \ssf(X) - \int \ssf(x)dP_\cX(x), \label{def:centered_X} \\
Y_\ssg &:=& \ssg(Y) - \bE[\ssg(Y)|\ssg] = \ssg(Y) - \int \ssg(y)dP_\cY(y) \label{def:centered_Y} .
\end{eqnarray}
Then the covariance between $\ssf(X)$ and $\ssg(Y)$ (with $\ssf$ and $\ssg$ being fixed) can be written as
\begin{eqnarray*}
\mathrm{cov}(\ssf(X), \ssg(Y)) 
&=& \bE_{X,Y}\left[ \left(\ssf(X) - \bE[\ssf(X)|\ssf] \right) \left( \ssg(Y) - \bE[ \ssg(Y) | \ssg ] \right) | \ssf, \ssg \right] \\
&=& \bE_{X,Y}\left[ X_\ssf Y_\ssg | \ssf, \ssg \right], 
\end{eqnarray*}
and therefore it follows that
\begin{equation} \label{eq:brownian-cov}
\bE_{\ssf,\ssg} \cov^2(\ssf(X),\ssg(Y)) = \bE_{\ssf,\ssg,(X,Y),(X',Y')} \left[ X_\ssf X'_\ssf Y_\ssg Y'_\ssg \right],
\end{equation}
where $(X',Y')$ is an independent copy of the joint random variable $(X,Y)$, and $X'_\ssf$ and $Y'_\ssg$ are defined similarly to (\ref{def:centered_X}) and (\ref{def:centered_Y}). The right side in \eqref{eq:brownian-cov} coincides with the definition of a dependence measure given by \citet[Definition 5]{SzeRiz09}, where they consider as $\ssf$ and $\ssg$ arbitrary stochastic processes on Euclidean spaces.
Specifically, the right side in \eqref{eq:brownian-cov} is the definition of the Brownian covariance \citep[Definition 4]{SzeRiz09}, if $\cX = \Re^p$ and $\cY = \Re^q$ for $p, q \in \mathbb{N}$ and if $\ssf$ and $\ssg$ are respectively the Brownian motions with the covariance kernels $k$ and $\ell$ given by
\begin{eqnarray}
k(x,x') &:=& \| x \| + \| x' \| - 2 \| x - x' \|, \quad x,x' \in \Re^p, \label{eq:cov-dist1} \\
\ell(y,y') &:=& \| y \| + \| y' \| - 2 \| y - y' \|, \quad y,y' \in \Re^q. \label{eq:cov-dist2}  
\end{eqnarray}
In this case, the Brownian covariance is further identical to the {\em distance covariance} \citep[Theorem 8]{SzeRiz09}, a nonparametric measure of dependence for random variables taking values in Euclidean spaces \citep[Definition 1]{SzeRiz09}.
Therefore our result implies that HSIC is identical to the distance covariance, when the kernels are given by \eqref{eq:cov-dist1} and \eqref{eq:cov-dist2}; we thus have recovered the result of \citet[Theorem 24]{SejSriGreFuk13} based on the probabilistic interpretation of HSIC.

\section{Conclusions}

In machine learning, statistics and numerical analysis, both the notion of a kernel and that of a Gaussian process play central roles in theoretical analysis. 
In fact,  they are so central in machine learning that they may be seen as placeholders for statistical learning theory and Bayesian analysis, the two mathematical frameworks that have historically provided the theoretical foundation of the field.
Kernel methods are founded on notions like regularization and optimization, while Gaussian processes are generative models operating in terms of marginal and conditional distributions. 
The present text provided a review of the intersection of these two areas, covering both fundamental equivalences and differences. It is important to clarify and understand these relationships to facilitate the transfer of knowledge and methods from one side to the other. At a time when machine learning is arguably expecting the emergence of a third, still only vaguely discernible new theoretical foundation in particular for deep models, this paper is also an opportunity to note that ``frequentist'' and ``Bayesian'' statistics are not always as different from each other as they may appear at first sight. 
We hope that this contribution is an important step towards developing a common language between the two fields, which will lead to further advances in each field, which otherwise would have been much more difficult to achieve.

\subsection*{Acknowledgements}
We would like to thank Mark van der Wilk for fruitful discussions.
The original idea for this manuscript arose during Workshop 16481 of the Leibniz-Centre for Computer Science at Schlo{\ss} Dagstuhl. The authors would like to express the Centre for their hospitality and support. MK and PH acknowledge support by the European Research Council (StG Project PANAMA). BKS is supported by NSF-DMS-1713011. DS is supported in part by The Alan Turing Institute (EP/N510129/1) and by the ERC (FP7/617071). All authors except the first are arranged in an alphabetical order.


\appendix

\section{Proofs}

\subsection{Proof of Lemma \ref{lemma:RKHS-norm-worst}}
\label{sec:proof-norm-worst}

\begin{proof}
By the reproducing property, the right side of \eqref{eq:RKHS-norm-worst} can be written as
\begin{equation} \label{eq:norm-worst-proof}
 \sup_{\| f \|_{\cH_k} \leq 1} \sum_{i=1}^m c_i f(x_i) = \sup_{\| f \|_{\cH_k} \leq 1} \left< \sum_{i=1}^m c_i k(\cdot,x_i), f \right>_{\cH_k}.
\end{equation}
By the Cauchy-Schwartz inequality, the right side of this equality is upper-bounded as 
$$
\sup_{\| f \|_{\cH_k} \leq 1} \left< \sum_{i=1}^m c_i k(\cdot,x_i), f \right>_{\cH_k} \leq \sup_{\| f \|_{\cH_k} \leq 1} \left\| \sum_{i=1}^m c_i k(\cdot,x_i) \right\|_{\cH_k} \left\| f \right\|_{\cH_k} = \left\| \sum_{i=1}^m c_i k(\cdot,x) \right\|_{\cH_k}.
$$
On the other hand, defining $g := \sum_{i=1}^m c_i k(\cdot,x_i) /  \left\| \sum_{i=1}^m c_i k(\cdot,x) \right\|_{\cH_k}$, we have $\| g \|_{\cH_k} = 1$, and thus the right side of \eqref{eq:norm-worst-proof} can be lower-bounded as
$$
\sup_{\| f \|_{\cH_k} \leq 1} \left< \sum_{i=1}^m c_i k(\cdot,x_i), f \right>_{\cH_k} \geq \left< \sum_{i=1}^m c_i k(\cdot,x_i), g \right>_{\cH_k} =  \left\| \sum_{i=1}^m c_i k(\cdot,x) \right\|_{\cH_k}.
$$
The assertion follows from \eqref{eq:norm-worst-proof} and these lower and upper bounds.
\end{proof}

\subsection{Proof of Corollary \ref{coro:gauss-sample-path}} \label{sec:proof-gp-sample-path-se}
For Banach spaces $A$ and $B$, we denote by $A \hrarrow B$ that $A \subset B$ and that the inclusion is continuous.
We first need to the notion of {\em interpolation spaces}; for details, see e.g.,~\citet[Section 7.6]{AdaFou03}, \citet[Section 5.6]{Steinwart2008}, \citet[Section 4.5]{CucZho07} and references therein.
\begin{definition}
Let $E$ and $F$ be Banach spaces such that $E \hrarrow F$.
Let $K: E \times \Re_+ : \to \Re_+$ be the $K$-functional defined by
$$
K(x,t) := K(x,t,E,F) :=  \inf_{y \in F} \left( \| x - y \|_E + t \| y \|_F \right), \quad x \in E,\ t > 0.
$$
Then for $0 < \theta \leq 1$, the interpolation space $[E, F]_{\theta,2}$ is a Banach space defined by
$$
[E, F]_{\theta,2} := \left\{ x \in E:\ \| x \|_{[E, F]_{\theta,2}} < \infty \right\},
$$
where the norm is defined by  
$$
\| x \|_{[E, F]_{\theta,2}}^2 := \int_0^\infty \left( t^{-\theta} K(x,t) \right)^2 \frac{dt}{t}.
$$
\end{definition}
We will need the following lemma.
\begin{lemma} \label{lemma:interp-space-embed}
Let $E$, $F$ and $G$ be Banach spaces such that $G \hrarrow F \hrarrow E$.
Then we have $[E,G]_{\theta,2} \hrarrow [E,F]_{\theta,2}$ for all $0 < \theta \leq 1$.
\end{lemma}
\begin{proof}
First note that, since $G \hrarrow F$, there exists a constant $c > 0$ such that $ \| z \|_{F} \leq c \| z \|_G$ holds for all $z \in G$.
Therefore, for all $x \in E$ and $t > 0$, we have
\begin{eqnarray*}
K(x, t,E,F) 
&=&  \inf_{y \in F} \left( \| x - y \|_E +  t \| y \|_F \right) \\
&\leq& \inf_{z \in G} \left( \| x - z \|_E +  t \| z \|_F \right) \\
&\leq& \inf_{z \in G} \left( \| x - z \|_E +  c t \| z \|_G \right) = K(x,ct,E,G). 
\end{eqnarray*}
Thus, we have
\begin{eqnarray*}
\| x \|_{[E, F]_{\theta,2}}^2 &=&  \int_0^\infty \left( t^{-\theta} K(x,t,E,F) \right)^2 \frac{dt}{t} \\
&\leq&  \int_0^\infty \left( t^{-\theta} K(x,ct,E,G) \right)^2 \frac{dt}{t} \\
&=& c^{-2\theta} \int_0^\infty \left( s^{-\theta} K(x,s,E,G) \right)^2 \frac{ds}{s} \quad (s := c t) \\
&=& c^{-2\theta} \| x \|_{[E, G]_{\theta,2}}^2, \quad \quad x \in [E, G]_{\theta,2}.
\end{eqnarray*} 
which implies the assertion.
\end{proof}
We are now ready to prove Corollary \ref{coro:gauss-sample-path}.
\begin{proof}
Fix $\theta \in (0,1)$.
First we show that $\sum_{i = 1}^\infty \lambda_i^\theta \phi_i^2(x) < \infty$ holds for all $x \in \cX$, which implies that the power of RKHS $\cH_{k_\gamma}^\theta$ is well defined.
To this end, by \citet[Theorem 2.5]{Ste17}, it is sufficient to show that 
\begin{equation} \label{eq:interp-sup}
[L_2(\nu), \cH_{k_\gamma} ]_{\theta,2} \hrarrow L_\infty(\nu),
\end{equation}
where $L_2(\nu)$ and $L_\infty(\nu)$ are to be understood as quotient spaces with respect to $\nu$, and $\cH_{k_\gamma}$ as the embedding in $L_2(\nu)$; see \citet[Section 2]{Ste17} for precise definition.

Let $m \in \mathbb{N}$ be such that $\theta m > d /2$, and $W_2^m(\cX)$ be the Sobolev space of order $m$ on $\cX$. 
By \citet[Theorem 4.48]{Steinwart2008}, we have $\cH_{k_\gamma} \hrarrow W_2^m(\cX)$. 
Therefore Lemma \ref{lemma:interp-space-embed} implies that  $[ L_2(\nu), \cH_\gamma]_{\theta, 2} \hrarrow [ L_2(\nu), W_2^m(\cX)]_{\theta, 2}$.
Note that, since $\nu$ is the Lebesgue measure, $[ L_2(\nu), W_2^m(\cX)]_{\theta, 2}$ is the Besov space $B_{22}^{\theta m}(\cX)$ of order $\theta m$ \citep[Section 7.32]{AdaFou03}.
Since $\cX$ is a bounded Lipschitz domain and $\theta m > d/2$, we have $B_{22}^{\theta m}(\cX) \hrarrow L_\infty(\nu)$ \citep[Proposition 4.6]{Tri06}; see also \citet[Theorem 7.34]{AdaFou03}.
Combining these embeddings implies \eqref{eq:interp-sup}.

We next show that $\sum_{i = 1}^\infty \lambda_i^{1-\theta} < \infty$, which implies the assertion by Theorem \ref{theo:gp-path-power-RKHS}.
To this end, for $i \in \mathbb{N}$,  define the $i$-th (dyadic) entropy number of the embedding ${\rm id}: \cH_{k_\gamma} \to L_2(\nu)$ by
$$
\varepsilon_i({\rm id}: \cH_{k_\gamma} \to L_2(\nu)) := \inf \left\{ \varepsilon > 0: \exists\ (h_j)^{2^{i-1}}_{j=1} \subset L_2(\nu)\ {\rm s.t.}\ B_{\H_{k_\gamma}} \subset \bigcup_{j=1}^{2^{i-1}} (h_j + \varepsilon B_{L_2(\nu)})  \right\},
$$
where $B_{\H_{k_\gamma}}$ and $B_{L_2(\nu)}$ denote the centered unit balls in $\H_{k_\gamma}$ and $B_{L_2(\nu)}$, respectively.
Note that since $\cX$ is bounded, there exists a ball of radius $r \geq \gamma$ that contains $\cX$.
From \citet[Theorem 12]{MeiSte17}, for all $p \in (0,1)$, we have
$$
\varepsilon_i({\rm id}: \cH_{k_\gamma}(\cX) \to L_2(\nu)) \leq c_{p,d,r,\gamma}\ i^{-1/2p}, \quad i \geq 1.
$$
$c_{p,d,r,\gamma} > 0$ is a constant depending only on $p$, $d$, $r$ and $\gamma$.
Using this inequality, the $i$-th largest eigenvalue $\lambda_i$ is upper-bounded as
$$
\lambda_i \leq 4 \varepsilon_i^2 ({\rm id}: \cH_{k_\gamma} \to L_2(\nu)) \leq 4 c_{p,d,r,\gamma}^2 i^{-1/p}, \quad i \geq 1,
$$
where the first inequality follows from \citet[Lemma 2.6 Eq.~23]{Ste17}; this lemma is applicable since we have $\int_\cX k_\gamma(x,x)d\nu(x) < \infty$ and thus the embedding ${\rm id}: \cH_{k_\gamma} \to L_2(\nu)$ is compact \citep[Lemma 2.3]{SteSco12}.
Therefore we have
$$
\sum_{i=1}^\infty \lambda_{i}^{1-\theta} < (4 c_{p,d,r,\gamma}^2)^{1-\theta} \sum_{i=1}^\infty  i^{-(1-\theta)/p}. 
$$
The right side is bounded, if we take $p \in (0,1)$ such that $1-\theta > p$. 
This implies $\sum_{i=1}^\infty \lambda_{i}^{1-\theta} < \infty$.
\end{proof}

\subsection{Proof of Proposition \ref{prop:HSIC-GP}}

\begin{proof}
Since we have the identity \eqref{eq:brownian-cov}, it is sufficient to prove that the right side of \eqref{eq:brownian-cov} is equal to the HSIC \eqref{eq:HSIC-def}.
First note that
\begin{eqnarray*}
&& X_\ssf X'_\ssf = 
\left( \ssf(X) - \int \ssf(x)dP_\cX(x) \right) \left( \ssf(X') - \int \ssf(x)dP_\cX(x) \right) \\
&& = 
 \ssf(X) \ssf(X') - \int \ssf(x) \ssf(X') dP_\cX(x)   - \int \ssf(X) \ssf(x)dP_\cX(x)  + \int \int \ssf(x) \ssf(x') dP_\cX(x) dP_\cX(x').
\end{eqnarray*}
By using the identity $k(x,x') = \bE_\ssf [\ssf(x)\ssf(x')]$ for $x,x' \in \cX$, we then obtain
\begin{eqnarray*}
&& \bE_\ssf[  X_\ssf X'_\ssf | X,X',Y,Y' ] \\
&=& k(X,X') - \int k(x,X')dP_\cX(x) - \int k(X,x)dP_\cX(x) + \int \int k(x,x') dP_\cX(x) dP_\cX(x') \\
&=& k(X,X') - \mu_{P_\cX}(X') - \mu_{P_\cX}(X) + \left\| \mu_{P_\cX} \right\|_{\cH_\cX}^2 \\
&=& \left< k(\cdot,X) - \mu_{P_\cX}, k(\cdot,X') - \mu_{P_\cX} \right>_{\cH_\cX}.
\end{eqnarray*}
Similarly, one can show that
$$
\bE_\ssf[  Y_\ssg Y'_\ssg | X,X',Y,Y' ] = \left< \ell(\cdot,Y) - \mu_{P_\cY}, \ell(\cdot,Y') - \mu_{P_\cY} \right>_{\cH_\cY}.
$$
Define $\overline{k}(\cdot,X):=k(\cdot,X) - \mu_{P_\cX}$ and $\overline{l}(\cdot,Y):=l(\cdot,Y) - \mu_{P_\cY}$. Therefore, the right side of \eqref{eq:brownian-cov} can be written as
\begin{eqnarray*}
&& \bE[ X_\ssf X'_\ssf  Y_\ssg Y'_\ssg ] \\
&=& \bE_{X, X', Y, Y'} \left[ \bE_{\ssf} \left[ X_\ssf X'_\ssf | X, X', Y, Y'  \right]\ \bE_{\ssg} \left[  Y_\ssg Y'_\ssg | X, X', Y, Y' \right] \right] \\
&=& \bE_{X, X', Y, Y'}  \left[ \left< k(\cdot,X) - \mu_{P_\cX}, k(\cdot,X') - \mu_{P_\cX} \right>_{\cH_\cX} \left< \ell(\cdot,Y) - \mu_{P_\cY}, \ell(\cdot,Y') - \mu_{P_\cY} \right>_{\cH_\cY} \right] \\
&=& \bE_{X, X', Y, Y'} \left[ \left< \overline{k}(\cdot,X) \otimes \overline{\ell}(\cdot,Y), \overline{k}(\cdot,X') \otimes \overline{\ell}(\cdot,Y') \right>_{\cH_\cX \otimes \cH_\cY} \right] \\
&=&  \left< \bE_{X,Y}  \left[ \overline{k}(\cdot,X) \otimes \overline{\ell}(\cdot,Y) \right], \bE_{X',Y'} \left[ \overline{k}(\cdot,X') \otimes \overline{\ell}(\cdot,Y') \right] \right>_{\cH_\cX \otimes \cH_\cY} \\
&=&  \left\| \bE_{X,Y}  \left[ \overline{k}(\cdot,X) \otimes \overline{\ell}(\cdot,Y) \right]  \right\|^2_{\cH_\cX \otimes \cH_\cY} \\
&=& \left\| \bE_{X,Y} \left[ k(\cdot,X) \otimes \ell(\cdot,Y) - k(\cdot,X) \otimes \mu_{P_\cY} - \mu_{P_\cX} \otimes \ell(\cdot,Y) + \mu_{P_\cX} \otimes \mu_{P_\cY} \right]  \right\|^2_{\cH_\cX \otimes \cH_\cY} \\
&=& \left\| \mu_{P_{\cX \times \cY}} - \mu_{P_\cX} \otimes \mu_{P_\cY} \right\|^2_{\cH_\cX \otimes \cH_\cY} = \HSIC(X,Y).
\end{eqnarray*}
\end{proof}

\bibliographystyle{abbrvnat}
\bibliography{bibfile}

\end{document}